\documentclass{article} % For LaTeX2e
\usepackage{iclr2023_conference,times}
\iclrfinalcopy
\usepackage[utf8]{inputenc} % allow utf-8 input
\usepackage[T1]{fontenc}    % use 8-bit T1 fonts
\usepackage{url}            % simple URL typesetting
\usepackage{booktabs}       % professional-quality tables
\usepackage{enumitem}
\usepackage{natbib}
\usepackage{amsfonts}       % blackboard math symbols
\usepackage{nicefrac}       % compact symbols for 1/2, etc.
\usepackage{microtype}      % microtypography
\usepackage{xcolor}         % colors
\usepackage{enumitem}
\usepackage{amsthm}
\usepackage{amsmath,amsfonts,bm}
\makeatletter
\newtheorem*{rep@theorem}{\rep@title}
\newcommand{\newreptheorem}[2]{%
\newenvironment{rep#1}[1]{%
 \def\rep@title{#2 \ref{##1}}%
 \begin{rep@theorem}}%
 {\end{rep@theorem}}}
\makeatother

\newtheorem{definition}{Definition}[section]
\newtheorem{theorem}{Theorem}

\newtheorem{remark}{Remark}
\newreptheorem{theorem}{Theorem}
\newtheorem{lemma}{Lemma}

\usepackage[framemethod=TikZ]{mdframed}
\mdfdefinestyle{MyFrame}{%
    linecolor=black,
    outerlinewidth=.3pt,
    roundcorner=5pt,
    innertopmargin=1pt, %\baselineskip,
    innerbottommargin=1pt, %\baselineskip,
    innerrightmargin=1pt,
    innerleftmargin=1pt,
    backgroundcolor=black!0!white}

\mdfdefinestyle{MyFrame2}{%
    linecolor=white,
    outerlinewidth=1pt,
    roundcorner=2pt,
    innertopmargin=\baselineskip,
    innerbottommargin=\baselineskip,
    innerrightmargin=10pt,
    innerleftmargin=10pt,
    backgroundcolor=black!3!white}

\newcommand{\RNum}[1]{\uppercase\expandafter{\romannumeral #1\relax}}

\newcommand{\tX}{$\textnormal{X}$}
\newcommand{\tK}{$\textnormal{\Kappa}$}

\newcommand{\vertiii}[1]{{\left\vert\kern-0.25ex\left\vert\kern-0.25ex\left\vert #1 
    \right\vert\kern-0.25ex\right\vert\kern-0.25ex\right\vert}}
\newcommand{\vertiiii}[1]{{\vert\kern-0.25ex\vert\kern-0.25ex\vert #1 
    \vert\kern-0.25ex\vert\kern-0.25ex\vert}}

\usepackage{mathtools}

\newcommand{\xhdr}[1]{{\noindent\bfseries #1}.}
\newcommand{\cut}[1]{}
\newcommand{\CITE}{\textcolor{red}{CITE}}

\newcommand{\removelatexerror}{\let\@latex@error\@gobble}

\def\eqref#1{Eq.~\ref{#1}}
\def\1{\bm{1}}

\def\rvx{{\mathbf{x}}}

\def\rmK{{\mathbf{K}}}

\def\rmP{{\mathbf{P}}}

\def\rmS{{\mathbf{S}}}

\def\rmW{{\mathbf{W}}}
\def\rmX{{\mathbf{X}}}

\DeclareMathAlphabet{\mathsfit}{\encodingdefault}{\sfdefault}{m}{sl}
\SetMathAlphabet{\mathsfit}{bold}{\encodingdefault}{\sfdefault}{bx}{n}
\newcommand{\tens}[1]{\bm{\mathsfit{#1}}}

\def\tK{{\tens{K}}}

\def\tS{{\tens{S}}}

\def\tX{{\tens{X}}}
\def\tY{{\tens{Y}}}

\def\gA{{\mathcal{A}}}
\def\gB{{\mathcal{B}}}
\def\gC{{\mathcal{C}}}
\def\gD{{\mathcal{D}}}

\def\gF{{\mathcal{F}}}

\def\gH{{\mathcal{H}}}
\def\gI{{\mathcal{I}}}

\def\gL{{\mathcal{L}}}

\def\gO{{\mathcal{O}}}
\def\gP{{\mathcal{P}}}
\def\gQ{{\mathcal{Q}}}
\def\gR{{\mathcal{R}}}
\def\gS{{\mathcal{S}}}

\def\gU{{\mathcal{U}}}

\def\sF{{\mathbb{F}}}

\def\sI{{\mathbb{I}}}

\def\sN{{\mathbb{N}}}

\def\sR{{\mathbb{R}}}

\def\sZ{{\mathbb{Z}}}

\DeclareMathOperator{\Ind}{Ind}

\DeclareMathOperator{\Span}{Span}
\DeclareMathOperator{\Vect}{Vec}

\usepackage{todonotes}
\usepackage{amssymb,fge}
\usepackage{thm-restate}
\usepackage{wrapfig}
\usepackage{bbm}
\usepackage{mathrsfs}
\usepackage{soul}
\usepackage{array}
\usepackage{multirow}

\def\setstretch#1{\renewcommand{\baselinestretch}{#1}}
\usepackage[final,pagebackref=true]{hyperref} % adds hyper links inside the generated pdf file
\renewcommand*{\backrefalt}[4]{%
    \ifcase #1 \footnotesize{(Not cited.)}%
    \or        \footnotesize{(Cited on page~#2)}%
    \else      \footnotesize{(Cited on pages~#2)}%
    \fi}

\hypersetup{
	colorlinks=true,       % false: boxed links; true: colored links
	linkcolor=blue,        % color of internal links
	citecolor=blue,        % color of links to bibliography
	filecolor=magenta,     % color of file links
	urlcolor=blue         
}

\newcommand{\ggi}[1]{\todo[inline]{{\textbf{Gauthier:} \emph{#1}}}}
\newcommand{\dfe}[1]{\todo[inline]{{\textbf{Damien:} \emph{#1}}}}
\newcommand{\jb}[1]{\todo[inline]{{\textbf{Joey:} \emph{#1}}}}
\title{A General Framework For Proving The Equivariant Strong Lottery Ticket Hypothesis}

\author{%
  Damien Ferbach\thanks{denotes equal contribution.}\\
  École Normale Supérieure, PSL, and Mila \thanks{Work done during an internship at Mila}\\
  \texttt{damien.ferbach@ens.psl.eu} \\
  \And
  Christos Tsirigotis $^*$\\
  Université de Montréal and Mila\\
  \And 
  Gauthier Gidel \thanks{Canada Cifar AI Chair}\\
  Université de Montréal and Mila\\
  \And
  Avishek (Joey) Bose \\
  McGill University and Mila\\
}

\begin{document}

\maketitle

\begin{abstract}
The Strong Lottery Ticket Hypothesis (SLTH) stipulates the existence of a subnetwork within a sufficiently overparameterized (dense) neural network that---when initialized randomly and without any training---achieves the accuracy of a fully trained target network. Recent works by \citet{da2022proving,burkholz2022convolutional} demonstrate that the SLTH can be extended to translation equivariant networks---i.e. CNNs---with the same level of overparametrization as needed for the SLTs in dense networks. However, modern neural networks are capable of incorporating more than just translation symmetry, and developing general equivariant architectures such as rotation and permutation has been a powerful design principle. In this paper, we generalize the SLTH to functions that preserve the action of the group $G$---i.e. $G$-equivariant network---and prove, with high probability, that one can approximate any $G$-equivariant network of fixed width and depth by pruning a randomly initialized overparametrized $G$-equivariant network to a $G$-equivariant subnetwork. We further prove that our prescribed overparametrization scheme is optimal and provides a lower bound on the number of effective parameters as a function of the error tolerance. We develop our theory for a large range of groups, including subgroups of the Euclidean $\text{E}(2)$ and Symmetric group $G \leq \gS_n$---allowing us to find SLTs for MLPs, CNNs, $\text{E}(2)$-steerable CNNs, and permutation equivariant networks as specific instantiations of our unified framework. Empirically, we verify our theory by pruning overparametrized $\text{E}(2)$-steerable CNNs, $k$-order GNNs, and message passing GNNs to match the performance of trained target networks. %within a given error tolerance.
\end{abstract}
% \everypar{\looseness=-1}
\section{Introduction}
\vspace{-5pt}
\looseness=-1
Many problems in deep learning benefit from massive amounts of annotated data and compute that enables the training of models with an excess of a billion parameters. 
Despite this appeal of overparametrization many real-world applications are resource-constrained (e.g., on device) and demand a reduced computational footprint for both training and deployment \citep{deng2020model}. A natural question that arises in these settings is then: is it possible to marry the benefits of large models---empirically beneficial for effective training---to the computational efficiencies of smaller sparse models? 

\looseness=-1
A standard line of work for building \emph{compressed} models from larger fully trained networks with minimal loss in accuracy is via weight pruning \citep{blalock2020state}. There is, however, increasing empirical evidence to suggest weight pruning can occur significantly prior to full model convergence. \citet{frankle2018lottery} postulate the extreme scenario termed \emph{lottery ticket hypothesis} (LTH) where a subnetwork extracted at initialization can be trained to the accuracy of the parent network---in effect ``winning" the weight initialization lottery. In an even more striking phenomenon~\citet{ramanujan2020s} find that not only do such sparse subnetworks exist at initialization but they already achieve impressive performance without any training. This remarkable occurrence termed the \emph{strong lottery ticket hypothesis} (SLTH) was proven for overparametrized dense networks with no biases~\citep{malach2020proving,pensia2020optimal,orseau2020logarithmic}, non-zero biases~\citep{fischer2021towards}, and vanilla CNNs \citep{da2022proving}.
Recently, \citet{burkholz2022most} extended the work of \citet{pensia2020optimal} to most activation functions that behave like ReLU around the origin, and adopted another overparametrization framework as in \citet{pensia2020optimal} such that the overparametrized network has depth $L+1$ (no longer $2L$). However, the optimality with respect to the number of parameters (Theorem 2 in \citet{pensia2020optimal}) is lost with this method. Moreover, \citet{burkholz2022convolutional} extended the results of \citet{da2022proving} on CNNs to non-positive inputs.

Modern architectures, however, are more than just MLPs and CNNs and many encode data-dependent inductive biases in the form of equivariances and invariances that are pivotal to learning smaller and more efficient networks~\citep{he2021efficient}. This raises the important question: can we simultaneously get the benefits of equivariance and pruning? In other words, does there exist winning tickets for the equivariant strong lottery for general equivariant networks given sufficient overparametrization?

\xhdr{Present Work}
\looseness=-1 
In this paper, we develop a unifying framework to study and prove the existence of strong lottery tickets (SLTs) for general equivariant networks. Specifically, in our main result (Thm.~\ref{thm:main_approx_network}) we prove that any fixed width and depth target $G$-equivariant network that uses a point-wise ReLU can be approximated with high probability to a pre-specified tolerance by a subnetwork within a random $G$-equivariant network that is overparametrized by doubling the depth and increasing the width by a logarithmic factor. 
Such a theorem allows us to immediately recover the results of \citet{pensia2020optimal,orseau2020logarithmic} for MLPs and of \citet{burkholz2021existence,da2022proving} for CNNs as specific instantiations under our unified equivariant framework. Furthermore, we prove that a logarithmic overparametrization is necessarily optimal---by providing a lower bound in Thm.~\ref{th:lowerbound}---as a function of the tolerance. Crucially, this is \emph{irrespective} of which overparametrization strategy is employed which demonstrates the optimality of Theorem~\ref{thm:main_approx_network}. Notably, the extracted subnetwork is also  $G$-equivariant, preserving the desirable inductive biases of the target model; such a fact is importantly not achievable via a simple application of previous results found in \citep{pensia2020optimal, da2022proving}.

\looseness=-1
Our theory is broadly applicable to any equivariant network 
that uses a pointwise ReLU nonlinearity. This includes the popular $\text{E}(2)$-steerable CNNs with regular representations~\citep{weiler2019general} (Corollary \ref{cor:e2CNN}) that model symmetries of the $2d$-plane as well as subgroups of the symmetric group of $n$ elements $\gS_n$, allowing us to find SLTs for permutation equivariant networks (Corollary \ref{cor:perm}) as a specific instantiation.
We substantiate our theory by conducting experiments by explicitly computing the pruning masks for randomly initialized overparametrized $\text{E}(2)$-steerable networks, $k$-order GNNs, and MPGNNs to approximate another fully trained target equivariant network.

\section{Background and Related Work}
\vspace{-5pt}
\looseness=-1
\xhdr{Notation and Convention}
% \xhdr{General Notations} 
For $p\in \sN$, $[p]$ denotes $\{0, \cdots, p-1\}$. We assume that the starting index of tensors (vectors, matrices,...) is $0$, e.g., $\rmW_{p,q}\,,\; p,q \in [d]$. $G$ is a group, and $\rho$ is its representation. We use $|\cdot|$ for the cardinality of a set, while $\bigoplus$ represents the direct sum of vector spaces or group representations and $\otimes$ indicates the Kroenecker product. We use $*$ to denote a convolution. We define $x^+,x^-$ as $ x^{+}=\max(0,x) \text{ and } x^{-}=\min(0,x)$. $\|\cdot\|$ is a $\ell_p$ norm while $\vertiii{\cdot}$ is its operator norm. 
For a basis $\gB=\{b_1, \dots, b_p\}$, we write $\vertiii{\gB}= \max_{\|\alpha\|_\infty \leq 1}\vertiii{\sum_{k=1}^p\alpha_kb_k}$. $\sigma(x)=x^+$ is the pointwise ReLU.
Finally, we take $(\epsilon, \delta) \in [0,\frac{1}{2}]^{2}$, and $\gU([a,b])$ is the uniform distribution on $[a,b]$. %$\tilde{n} \in \sN$

\xhdr{Equivariance} We are interested in building equivariant networks that encode the symmetries induced by a given group $G$ as inductive biases.  To act using a group we require a group representation $\rho: G \to GL(\mathbb{R}^D)$, which itself is a group homomorphism and satisfies $\rho(g_1g_2) = \rho(g_1)\rho(g_2)$ as $GL(\mathbb{R}^D)$ is the group of $D \times D$ invertible matrices with group operation being ordinary matrix multiplication. Let us now recall the main definition for equivariance:
\begin{definition}
Let $\mathcal{X}\subset\sR^{D_x}$ and $\mathcal{Y}\subset\sR^{D_y}$ be two sets with an action of a group $G$. A map $f: \mathcal{X} \to \mathcal{Y}$ is called $G$-equivariant, if it respects the action, i.e., $\rho_{\mathcal{Y}}(g) f(x) = f(\rho_{\mathcal{X}}(g)  x), \forall g\in G$ and $x \in \mathcal{X}$.  A map $h: \mathcal{X} \to \mathcal{Y}$ is called $G$-invariant, if $h(x) = h(\rho_{\mathcal{X}}(g)  x), \forall g\in G$ and $x \in \mathcal{X}$.
\end{definition}
As a composition of equivariant functions is equivariant, to build an equivariant network it is sufficient to take each layer $f_i$ to be $G$-equivariant and utilize a $G$-equivariant non-linearity (e.g. pointwise ReLU).
Given a vector space and a corresponding group representation we can define a feature space  $\sF_i \coloneqq (\sR^{D_i}, \rho_{i})$. Note that we can stack multiple such feature spaces in a layer, for example, the input feature space to an equivariant layer $i$ can be written as $n_i$ blocks $\sF_i^{n_i} := \bigoplus_{m=1}^{n_i} \sF_i$.

% \cut{
% \xhdr{On groups} $G$ will denote a group. $\gO(2)$ denotes the orthogonal group of the plane, $SO(2)$ denotes the special orthogonal group of the plane. $G\leq H$ for $G$ and $H$ groups means that G is a subgroyup of $H$. $E(2)$ denotes $(\sZ^{2},+)\rtimes \gO(2)$ the euclidean group of the plane.
% By analogy, $E(n)$ denotes $(\sZ^{n},+)\rtimes \gO(n)$ the euclidean group in dimension $n$.
% $d_{in/out} \in \sN$, $\rho_{in/out}:G \rightarrow GL(\sR^{d_{in/out}})$ homomrophisms are two representations of $G$. $\sF_{in/out}$ denotes $(\sR^{d_{in/out}}, \rho_{in/out})$ a feature space. $\tilde{\sF}_{in}$ denotes $\bigoplus_{i=1}^{\tilde{n}}\sF_{in}$ an intermediate feature space.}

\looseness=-1
A $G$-equivariant basis is a basis of the space of equivariant linear maps between two vector spaces. We can decompose a $G$-equivariant linear map $f_i: \sF_{i}\to \sF_{i+1}$ in a corresponding equivariant basis $\gB_{i \to i+1} = \{b_{i \to i+1,k} \in \sR^{D_i\times D_{i+1}}, \forall k\in [|\gB_{i \to i+1}|]\}$. When working with stacks of $n_i$ (resp. $n_{i+1}$) input (resp. output) feature spaces we may express the full equivariant basis by considering $\kappa_{n_i\to n_{i+1}}= \{\kappa^{p,q}_{n_i\to n_{i+1}}\in \mathbb{R}^{n_i \times n_{i+1}}\,, \; (p,q) \in [n_i] \times [n_{i+1}]\}$, where each element $\kappa^{p,q}_{n_i \to n_{i+1}}$ is a matrix with a single non-zero entry at position $(p,q)$. Then the basis for $G$-equivariant maps between $\sF_{i}^{n_i}\to\sF_{i+1}^{n_{i+1}}$ can be written succinctly as the Kronecker product between two basis elements $\kappa_{n_i\to n_{i+1}} \otimes \gB_{i \to i+1} $. Some instances of $G$ and $\sF_i$ are presented in Tab.~\ref{tab:classification_using_thm_one}. For example, in the case of CNNs with kernel size $d^2$, the linear map $f$ is a convolution where $n_i$ (resp. $n_{i+1}$) are the number of input (resp. output) channels and $\kappa_{n_{i}\to n_{i+1}} \otimes \gB_{i \to i+1} $ is the basis of convolutions of size $d^2\times n_i\times n_{i+1}$.

\looseness=-1 
\xhdr{Related Work on Strong Lottery Tickets}
Winning SLTs approximate a target ReLU network $f(x)$ by pruning an overparametrized ReLU network $g(x)$ with weights in any given layer drawn i.i.d. from $w_i \sim \gU([-1,1])$.\footnote{It will work with any distribution which contains a uniform distribution, e.g. Gaussian, see \S\ref{app:subset_sum}}
Our error metric of choice is the uniform approximation over a unit ball: $\max_{x \in \mathbb{R}^D: ||x|| \leq 1} ||f(x) - \hat{g}(x)|| \leq \epsilon,$
where $\hat{g}(x)$ is the subnetwork constructed from pruning $g(x)$. Let us first consider the case of approximating a single neuron $w_i \in [-1, 1]$ in some layer of $f(x)$ with $n$ i.i.d. samples $X_1, \dots, X_n \sim \gU([-1,1])$. If $n=O(1/\epsilon)$ then there exists a $X_i$ that is $\epsilon$-close to $w_i$ \citep{malach2020proving}. A similar approximation fidelity can be achieved with an exponentially smaller number of samples by not relying on just a single $X_i$ but instead a subset whose sum approximates the target weight. \citet{lueker1998exponentially,da2022revisiting} proved that $n = O(\log (1/\epsilon))$ random variables were sufficient for the existence of a solution to the random \textsc{Subset-Sum} problem (a subset $S \subseteq \{1, \dots, n\}$ such that $|w_i - \sum_{i \in S} X_i | \leq \epsilon$). \citet{pensia2020optimal}  utilize the \textsc{Subset-Sum} approach for weights on dense networks resulting in a logarithmic overparametrization of the width of a layer in $g(x)$. To bypass the non-linearity (ReLU) \citet{pensia2020optimal} decompose the output activation $\sigma(wx) = w^+x^+ + w^-x^-$ and approximate each term separately. With no additional assumption on the inputs (\citet{da2022proving} assume positive entries), this approach fails for equivariant networks as each entry of the output of an equivariant linear map is affected by multiple input entries. 
% In order to use a similar proof technique~\citet{da2022proving} assume the input $x$ only has positive entries.
% \footnote{A recent preprint~\citet{burkholz2022convolutional} also fixed this issue for CNNs by using a similar proof technique as the one we use for Corollary~\ref{theorem_CNN}. We consider that our work has been developed independently, under the more general umbrella of $G$-equivariant networks.} 
% However, this assumption can be considered as restrictive since in practice input may be negative due to data standardization or when using batch normalization inside the neural network architecture.

\section{SLT for General Equivariant Nets}
\vspace{-5pt}
\label{sec:general_slt}
%\everypar{\looseness=-1}
\begin{wrapfigure}{o}{0.47\textwidth}
 \vspace{-15pt}
    \includegraphics[width=.95\linewidth]{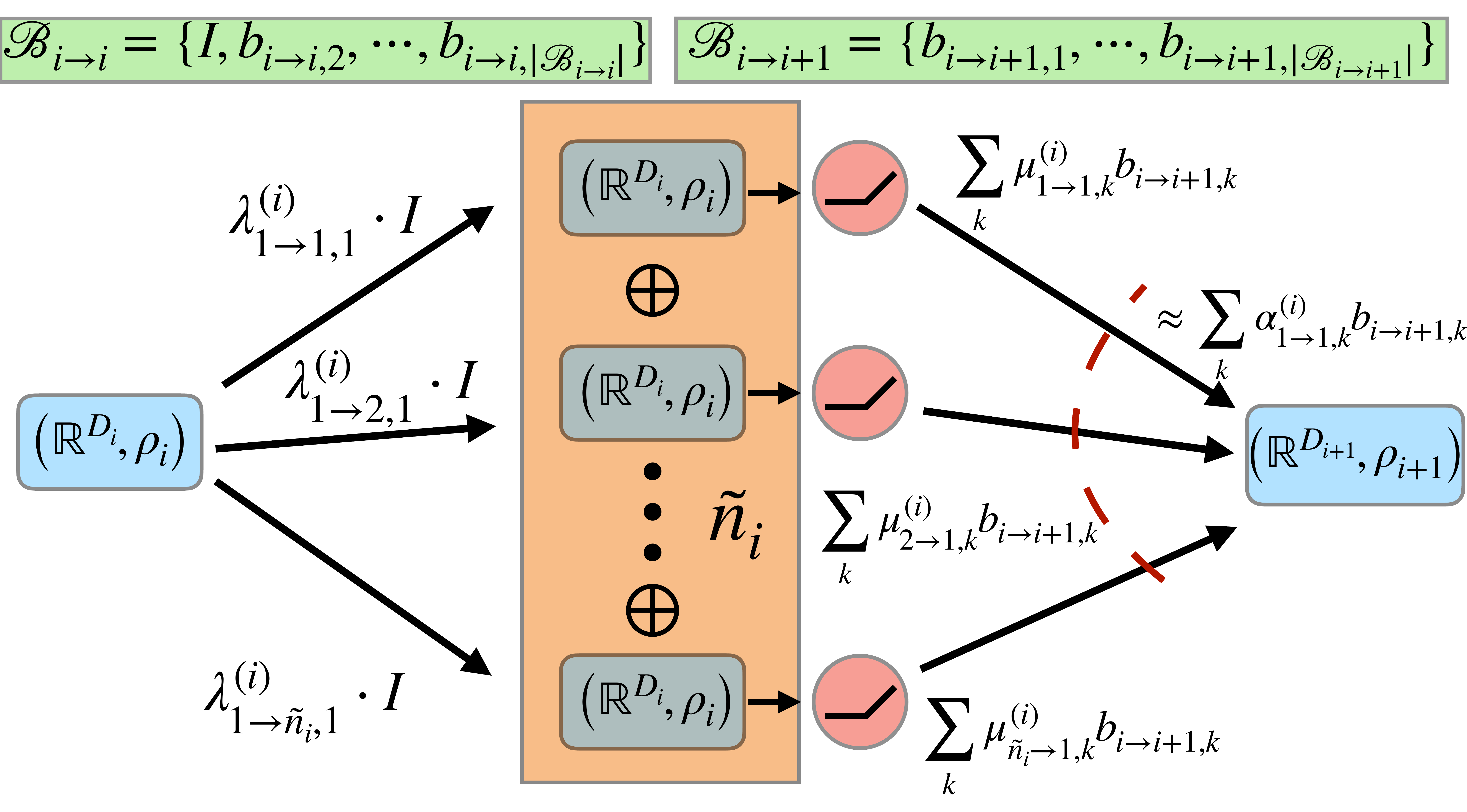}
 %\vspace{-5pt}
\caption{General Equivariant Pruning Method}
%  \vspace{-10pt}
\label{fig:general_equivar_pruning}
\end{wrapfigure}
Our results and proof techniques build upon the line of work by~\citet{pensia2020optimal}, \citet{da2022proving}, and ~\citet{burkholz2022convolutional}. 
Specifically, we rely on the \textsc{Subset-Sum} algorithm \citep{lueker1998exponentially} to aid in approximating any given parameter of the target network. Departing from prior work, the main idea used in our technical analysis is to prune an overparametrized equivariant network in a way that \emph{preserves equivariance}, as applying \textsc{Subset-Sum} using~\citet{da2022proving} construction may destroy the prunned network's equivariance.

\xhdr{Challenges in Adapting Proof Techniques} There are two major difficulties in adapting the tools first introduced in~\citet{pensia2020optimal} to $G$-steerable networks.
In proving the SLTH for dense networks the relevant parameters that can be pruned are all the parameters of weight matrices, which can be intuitively understood as pruning in a canonical basis. 
However, such a strategy immediately fails for $G$-equivariant maps as the canonical basis is not generally $G$-equivariant, thus pruning in this basis breaks the structure of the network and its equivariance. In fact, as described in \citet{weiler2019general} a $G$-equivariant linear map consists of linearly combining the elements of the equivariant basis with learned combination coefficients which are the effective parameters of the $G$-equivariant model. To preserve equivariance we may only prune these parameters and \emph{not} any weight in $f_i$. However, this introduces a new complication as the interaction with the ReLU becomes more challenging. \citet{da2022proving}~circumvent this in the special case of regular CNNs by assuming only positive inputs.
In contrast, our main technical lemma (Lem.~\ref{Approximation_of_a_layer}), introduces a construction that does not require such a restrictive assumption and generalizes the techniques of~\citet{burkholz2022convolutional} to $G$-equivariant networks.

\xhdr{Overparameterized Network Shape}
We seek to approximate a single $G$-equivariant layer with two random overparameterized $G$-equivariant layers.
We take the input $\|x\| \leq 1$ to be in a bounded domain to control the error which could diverge on unbounded domains.
Let $\gF_{i}$ be the set of $G$-equivariant linear maps $\sF_{i}^{n_i}\to \sF_{i+1}^{n_{i+1}}$ of the $i$-th layer in the target network. Then, $f_{i} \in \gF_{i} \; \text{s.t.}, \; \vertiiii{f_{i}}\leq 1$, is a specific realization of a target equivariant map that we will approximate---i.e. $f_{i}(x) = \rmW^f_i(x)$. Without any loss of generality, let the coefficients of $\rmW^f_i$ be such that $|\alpha_k| \leq 1$ when decomposed in the basis $\kappa_{n_i \to n_{i+1}}\otimes \gB_{i \to i+1}$. Concretely, $f_i \in \gF_{i} \coloneqq \big\{ \rmW_{i}^f =  \sum_{k}\alpha_{k}b_{k} : b_k \in \kappa_{n_i \to n_{i+1}}\otimes \gB_{i \to i+1},|\alpha_k| \leq 1, \vertiiii{\rmW_i^f}\leq 1\big\}$. We can now recursively apply the previous constructions to construct a desired $G$-equivariant target network $f \in \gF$ of depth $l \in \mathbb{N}$. Analagously, we can define an atomic unit of our random overparameterized source model $\gH_{i}$ as the set of $G$-equivariant maps with one intermediate feature space (layer) ${\sF}_{i}^{\tilde n_i}$ followed by a ReLU. That is, any $h_{i} \in \gH_{i}$ applied to an input $x$ can be written as $h_i(x)=\rmW_{2i+1}^h\sigma(\rmW_{2i}^h x)$. In our construction, we choose $\rmW^h_{2i}$ whose equivariant basis is $\kappa_{n_i\to \tilde{n}_i} \otimes \gB_{i \to i}$ where $\tilde{n}_i$ is the overparametrization factor of the $i$-th layer. We assume $\gB_{i \to i}$ contains the identity element, which is trivially equivariant. The basis coefficients of $\rmW_{2i}^h$ are written as $\lambda^{(i)}_{p\to q,k}$, which refers to the coefficient of the $k$-th basis element in $\gB_{i \to i}$ for the map between the $p$-th block of $\sF_{i}^{n_i}$ to the $q$-th block of ${\sF}_{i}^{\tilde n_i}$. Similarly, $\rmW^h_{2i+1}$ can be decomposed in the basis $\kappa_{\tilde{n}_i\to n_{i+1}} \otimes \gB_{i \rightarrow i+1} $ with coefficients $\mu^{(i)}_{p\to q,k}$. \cut{Such an overparametrization scheme leads to a ``diamond'' shape as illustrated in Fig \ref{fig:general_equivar_pruning}.
In practice, the $i$-th layer in the source networks can be a stack of $n_i$ feature spaces $\hat \sF^{\oplus n_i}_{i}$ which simply amounts to using $\kappa_{n_i \to \tilde{n}_i n_i} \otimes \gB_{i \to i}$ and $\kappa_{\tilde{n}_i n_i \to n_{i+1}} \otimes \gB_{i\to i+1}$ as the basis decomposition for $\rmW_{2i}^h$ and $\rmW_{2i+1}^h$ respectively.} Fig. \ref{fig:general_equivar_pruning} illustrates this construction after pruning the first layer for $n_i=n_{i+1}=1$ which leads to a ``diamond" shape. We can finally apply the previous construction to build an overparametrized network $h \in \gH$ of depth $2l$. We summarize all the notation used in the rest of the paper in Tab.~\ref{tab:notation_table}.

\begin{table}[ht!]
    \small
    \centering
    \vspace{-7pt}
    \begin{tabular}{ccc}
    \toprule
      $G$-Equivariant map & Basis  & Basis Coefficients\\
    \midrule
     $\rmW_{2i}^h: \sF_i^{n_i} \to \sF_i^{\tilde n_i}$ & $\kappa_{n_i \to \tilde{n}_i} \otimes \gB_{i \to i}$ & $\lambda_{p\to q, k}^{(i)}\,, \quad$ $p \in [n_i], q\in [\tilde{n}_i], k \in [|\gB_{i \to i}|]$ \\
    \midrule 
     $\rmW_{2i+1}^h: {\sF}_i^{\tilde n_i} \to \sF_{i+1}^{n_{i+1}}$ & $\kappa_{\tilde{n}_i \to n_{i+1}} \otimes \gB_{i \to i + 1}$ & $\mu_{p\to q, k}^{(i)}\,, \quad$ $p \in [\tilde{n}_i], q\in [n_{i+1}], k \in [|\gB_{i \to i+1}|]$ \\
    \midrule
     $\rmW_i^f: \sF_i^{n_i} \to \sF_{i+1}^{n_{i+1}}$ & $\kappa_{n_i \to n_{i+1}} \otimes \gB_{i \to i + 1}$ & $\alpha_{p\to q, k}^{(i)}\,, \quad$ $p \in [n_i], q\in [n_{i+1}], k \in [|\gB_{i \to i+1}|]$\\
    \bottomrule
    \end{tabular}
    \caption{\small
   Summary of notation used to decompose each $G$-equivariant map in the source and target networks.
    }
    \label{tab:notation_table}
    \vspace{-10pt}
\end{table}

\subsection{Theoretical Results}
\vspace{-5pt}
We first prove Lemma~\ref{Approximation_of_a_layer} which states that with high probability a random overparametrized $G$-equivariant network of depth $l=2$ (Fig.~\ref{fig:general_equivar_pruning}) can $\epsilon$-approximate any target map in $\gF_{i}$ via pruning.

\begin{mdframed}[style=MyFrame2]
\begin{restatable}{lemma}{approxonelayer}
\label{Approximation_of_a_layer}
Let $h_i \in \gH_{i}$ be a random overparametrized $G$-equivariant network as defined above, with coefficients $\lambda^{(i)}_{p\to q,k}$ and $\mu^{(i)}_{p \to q,k}$ drawn from $\gU([-1,1])$. Further suppose that each $\tilde{n}_i= C_1n_{i}\log(\frac{n_{i}n_{i+1}\max\left(|\gB_{i\to i+1}|, \vertiii{\gB_{i\to i+1}}\right)}{\min\left(\epsilon, \delta\right)})$ where $C_1$ is a constant.
Then, with probability $1-\delta$, for every target $G$-equivariant layer $f_{i}\in \gF_{i}$, one can find two pruning masks $\mathbf{S}_{2i},\mathbf{S}_{2i+1}$ on the coefficients $\lambda^{(i)}_{p \to q,k}$ and $\mu^{(i)}_{p \to q,k}$ respectively such that:
\begin{equation}
    \max_{\mathbf{x}\in \sR^{D_i\times n_i},\,\|\mathbf{x}\|\leq 1} \|(\mathbf{S}_{2i+1} \odot\rmW^h_{2i+1})\sigma((\mathbf{S}_{2i} \odot \rmW^h_{2i})\rvx) - f_{i}(\rvx)\| \leq \epsilon \,.
\end{equation}
\end{restatable}
\vspace{-4mm}
\end{mdframed}
\vspace{-2mm}
\begin{proof}[Proof sketch]
\looseness=-1
We prune all non-identity coefficients of the basis decomposition of the first layer obtaining
``diamond'' shape (see Fig.~\ref{fig:general_equivar_pruning} for ($n_i=n_{i+1}=1$)) allowing us to bypass the pointwise ReLU. The two layers can now be used to approximate every weight of the target by solving independent \textsc{Subset-Sum} problems on the coefficients of the second layer.
The full proof is provided in~\S\ref{app:proof_approximation_of_a_layer}.
\end{proof}
To approximate any $f$ in $ \gF_{i}$ which is a $G$-equivariant target network of depth $l$ and fixed width, we can now apply Lemma \ref{Approximation_of_a_layer} $l$-times to obtain our main theorem, whose proof is provided in~\S\ref{app:proof_thm_one}.

% to prove our main result where we approximate an entire $G$-equivariant target network of depth $l$. That is we by applying Lemma $l$-times for each layer in $f_t$ we obtain the following theorem:

\begin{mdframed}[style=MyFrame2]
\begin{restatable}{theorem}{approxnetwork}
\label{thm:main_approx_network}
Let $h \in \gH$ be a random overparametrized $G$-equivariant network with coefficients $\lambda_{p\to q,k}^{(i)}$ and $\mu_{p\to q,k}^{(i)}$, for $i \in [l]$ and indices $p,q,k$ as defined in Table \ref{tab:notation_table}, all drawn from $\gU([-1,1])$. Suppose that $\tilde{n}_i= C_2n_{i}\log(\frac{n_{i}n_{i+1}\max\left(|\gB_{i\to i+1}|, \vertiii{\gB_{i\to i+1}}\right)l}{\min\left(\epsilon, \delta\right)})$, where $C_2$ is a constant.
Then with probability $1-\delta$, for every $f \in \gF$, one can find a collection of pruning masks $\mathbf{S}_{2l-1}, \dots \mathbf{S}_0$ on the coefficients $\lambda_{p \to q,k}^{(i)}$ and $\mu_{p \to q,k}^{(i)}$ for every layer $i \in [l]$ such that:
\begin{equation}
    \max_{\mathbf{x}\in \sR^{D_0 \times n_0},\,\|\mathbf{x}\|\leq 1} \|(\mathbf{S}_{2l-1} \odot\rmW^h_{2l-1})\sigma\left( \dots \sigma((\mathbf{S}_0 \odot \rmW^h_{0})\rvx)\right) - f(\mathbf{x})\| \leq \epsilon \,.
\end{equation}
\end{restatable}
\vspace{-4mm}
\end{mdframed}
\looseness=-1
We recover a similar overparametrization as~\citet{pensia2020optimal} with respect to the width of $h$. However, the significant improvement provided by this result is that, since we do not prune dense nets but $G$-equivariant ones, the number of effective parameters in the overparametrized network is $\nicefrac{|\gB_{i \to i+1}|}{D_i D_{i+1}}$ smaller than a dense net of the same width. In section~\ref{sub:lower} we make this difference explicit and show Theorem~\ref{thm:main_approx_network} is optimal up to log factors not only with respect to the tolerance $\epsilon$ but also with respect to $\nicefrac{|\gB_{i \to i+1}|}{D_i D_{i+1}}$ quantifying the expressiveness of $G$-equivariant networks. %is in comparison to a dense network.
% network in which we do our search of pruning masks is sparse. 
% \begin{remark}

% \end{remark}
%The detailed proof for theorem \ref{thm:main_approx_network} is outlined in~\S\ref{app:proof_thm_one}.

\subsection{Lower Bound On The Overparametrization}
\vspace{-5pt}
\label{sub:lower}
\looseness=-1
When searching for equivariant winning tickets a natural question that arises is the optimality of the overparametrization factor $\tilde{n}_i$ with respect to the tolerance~$\epsilon$. In the same vein as \citet{pensia2020optimal} for MLPs, we now prove under mild assumptions that, in the equivariant setting, $\tilde{n}_i$ is indeed optimal (Theorem \ref{th:lowerbound}). \cut{Without loss of generality, we assume $n_i=n_{i+1}=1$ and} We will assume that our equivariant basis $\gB_{i\to i+1}$ has the following property: $\forall f_i \in \Span(\gB_{i \to i+1})$ where $f_i=\sum_{k}\alpha_kb_{i\to i+1,k}$ we have:
$\vertiii{f_i} \leq 1 \implies |\alpha_k|\leq 1,\, k \geq 0$. Note that this can be obtained by a rescaling of the basis elements. Lastly, we also assume the existence of positive constants $M_1$ and $M_2$ such that $|\gB_{i \to i}|\leq M_1|\gB_{i\to i+1}|$ and $n_i \leq M_2n_{i+1}$. These assumptions are relatively mild and hold in the practical situations described in Tab.~\ref{tab:classification_using_thm_one} (cf \S\ref{app:lowerbound} for details).
Under these assumptions we achieve the following (tight) lower bound.
% on the overparametrization factor $\tilde{n}_i$.

% Specifically, we prove a tight lower bound on $\tilde{n}_i$

% Probably the most interesting point of this work is that the bound $\tilde{n_i}$ that we computed is optimal with respect to the dependency on $\epsilon$. \cite{pensia2020optimal} showed that their result was optimal in the particulare case of MLPs. A very interesting fact is that it remains true in the general case. Our proof in the following builds upon the same tools as originally developed by \cite{pensia2020optimal} but in a far more general setting.

% We will make some assumptions to simplify the shape of Theorem \ref{th:lowerbound}. But the reader should have in mind that the result is very general and does not restrict itself to the following case.

% We will first take $n_i=n_{i+1}=1$ in the following without loss of generality. This can be very easily generalized to all integers by looking at what the norms become in multiple dimensions. Moreover let's assume that the basis $\gB_{i\to i+1}$ that we choose has the following property: $\forall f_t^{(i)} \in \Span(\gB_{i \to i+1})$ where $f_i=\sum_{k}\alpha_kb_{i\to i+1,k}$ we have:
% $\vertiii{f_i} \leq 1 \implies \|\alpha\|_{\infty}\leq 1$

% Note that this assumption can be achieved by rescaling the elements of the basis.

% Finally we go to the case where $\exists M\in \sR^+$ such that $|\gB_{i \to i}|\leq M|\gB_{i\to i+1}|$

\begin{mdframed}[style=MyFrame2]
\begin{restatable}{theorem}{lowerbound}
\label{th:lowerbound}
Let $\hat h_i$ be a network with $\Theta$ parameters such that:
\begin{align}
    \forall f_{i} \in \gF_i, \ \exists S_i \in \{0,1\}^{\Theta} \text{ such that }  \max_{\mathbf{x}\in \sR^{D_i \times n_i},\,\|\mathbf{x}\|\leq 1} \|(S_i \odot \hat h_i)(\mathbf{x}) - f_{i}(\mathbf{x})\| \leq \epsilon\,.
\end{align}
Then $\Theta$ is at least $\Omega\left(n_in_{i+1}|\gB_{i \to i+1}|\log(\frac{1}{\epsilon})\right)$ and  $\tilde{n}_i$ is at least $\Omega(n_i\log\left(\frac{1}{\epsilon}\right))$ in Theorem  \ref{thm:main_approx_network}.
\end{restatable}
\vspace{-20pt}
\end{mdframed}

\begin{proof}[Proof Idea]
The full proof is provided in~\S\ref{app:lowerbound} and relies on a counting argument to compare the number of pruning masks and functions in $\gF_i$ within a distance of at least $2\epsilon$ of each other.
\end{proof}
Thm.~\ref{th:lowerbound} dictates that if we wish to approximate a $G$-equivariant network target network to $\epsilon$-tolerance by pruning an overparametrized arbitrary network, the latter must have at least $\Omega(n_in_{i+1}|\gB_{i \to i+1}|\log(\frac{1}{\epsilon}))$ parameters.
Applying the above result to our prescribed overparametrization scheme in Thm.~\ref{thm:main_approx_network} we find our proposed strategy is optimal with respect to $\epsilon$ and almost optimal with respect to $|\gB_{i \to i+1}|$. We incur a small extra $\log$ factor whose origin is discussed in~\S\ref{app:lowerbound}. In the equivariant setting, the result in \citet{pensia2020optimal} is far from optimal as their result gives guarantees on the pruning of dense nets with a similar width as the $G$-equivariant targets which incurs an increase by a factor $\nicefrac{D_i D_{i+1}}{|\gB_{i \to i+1}|}$ in the number of parameters. As a specific example, for overparametrized $G$-steerable networks (Tab.~\ref{tab:classification_using_thm_one}), we have $\nicefrac{D_i D_{i+1}}{|\gB_{i \to i+1}|} = d^2|G|$.
On images of shape $\mathbb{R}^{224 \times 224 \times 3}$ with $G= C_8$,  it corresponds to $\approx 4.10^5$ fewer ``effective" parameters than a dense network. Finally, we note that Thm.~\ref{th:lowerbound} makes no statement on which overparametrization strategy \emph{achieves} such a lower bound. Remarkably, the pruning strategy prescribed by Thm.~\ref{thm:main_approx_network} recovers this optimal lower bound on $\tilde{n}_i$, meaning that, unsurprisingly, $G$-equivariant nets are the most suitable structure to prune. %approximate a $G$-equivariant net by pruning.

% In sharp contrast, our results restrict the overparametrized network to be $G$-equivariant, resulting in an essentially optimal overparametrization scheme with respect to the ``effective parameters'' of the model. For example when dealing with image net dataset to train $G$-steerable CNNs, our overparametrized models could have 1000000 less parameters that dens nets with the same width.\cut{However, such an overparametrization is not optimal with respect to the number of effective parameters by a factor $|\gB_{i \to i+1}|$ for each layer $i$ in the network.} 

% We might summarize these ideas by the following: theorem \ref{th:lowerbound} dictates that if we wish to approximate a $G$-equivariant network target network to $\epsilon$-tolerance by pruning an overparametrized arbitrary network, the latter must have at least $\Omega(n_in_{i+1}|\gB_{i+1}\log(\frac{1}{\epsilon}))$ parameters more than the target network.

 %It is almost optimal because, (like Pensia has a $\log(d_id_{i+1} \times 1)$, with $|\gB_{i \to i+1}|=1$), we still have a $\log(n_in_{i+1}|\gB_{i \to i+1}|)$ in $\tilde{n}_i$. which is not necessary. However, this non-optimal factor comes from only of the use of norms that do not feat very well in the setting of the proof, and does not result of a bad type of overparametrization.

\cut{\dfe{we are optimal with respect to $\epsilon$ and almost optimal with respect to $|\gB_{i\to i+1}|$ the log comes from the choice of norms nothing we can really do in the general case}}
% The striking fact is that it does not, however say anything about the structure of the overparametrized network. It therefore raises the question:''which structure should one choose to ensure that the overparametrization factor needed to match a target network by pruning, is optimal?" In the dual case of matching a target network by modifying weights (usual learning), the structure that one will choose is straightforwardly a $G$-equivariant network without overparametrization: when modifying the weights one will be able to match a target $G$-equivariant network. In the case of pruning it may seem less trivial. Theorem \ref{thm:main_approx_network} however answers this question by proving that if the target is a $G$-equivariant network, the structure to overparametrize is a $G$-equivariant network.
% This can seem trivial but the reader should understand that was not at all trivial and comes from the interaction between the pruning method and the structure of networks.

\section{SLT for Specific Choices of $G$}
\vspace{-5pt}
\label{sec:specific_cases}
\everypar{\looseness=-1}
In this section, we turn our focus to specific instantiations of our main theoretical results for different choices of groups. To apply Theorem~\ref{thm:main_approx_network}, one simply needs to specify the group $G$, the group representation $\rho(g)$, and finally the feature space $\sF$. For instance, we can immediately recover the results for dense networks \citep{pensia2020optimal} by noticing $G=\{ e \}$ is the trivial group with a trivial action on $\sR^D$ (see the proof in \S\ref{app:pensia}). In Table~\ref{tab:classification_using_thm_one} below we highlight different $G$-equivariant architectures through the framework provided in~\S\ref{sec:general_slt} before proving each setting in the remainder of the section.

\begin{table}[ht!]
    \small
    \centering
    \vspace{-10pt}
    \begin{tabular}{cccccccc}
    \toprule
     & $G$ & $\rho_{i}$ & $\sF_{i}$ & $|\gB_{i\rightarrow i+1}|\vee\vertiii{\gB_{i\rightarrow i+1}}$\\
    \midrule
    MLP & $\{e\}$ & trivial & $\sR$&$1$\\
    \midrule
    CNN & $(\sZ^{2},+)$ & $ f_{i}(x-t)$ & $(\sR^{d^2},\rho_{i})$ & $d^2$\\
    \midrule
    \text{E}(2)-CNN & $(\sZ^{2},+)\!\rtimes\! \text{O}(2)$ & $ \rho_{\text{reg}}(g)f_{i}(g^{-1}(x -t))$ & $(\sR^{d^{2}\times |G|^{2}}, \rho_{i})$&$d^{2}|G|^3$\\
    \midrule
    Permutation & $\gS(n)$ & $\tX_{i_{\sigma(1)}, \dots, i_{\sigma(k)}, j}$ & $(\sR^{n^{k_i}}, \rho_{i})$& $\Tilde{b}(k_i+k_{i+1}) \vee (n^{k_i}+1)$ \\
    \bottomrule
    \end{tabular}
    \caption{
    Instantiations of Theorem \ref{thm:main_approx_network} for different choices of $G$. MLP was proven in \citet{pensia2020optimal}, CNN was proven in \citet{da2022proving,burkholz2022convolutional}. We note $a\vee b := \max(a,b)$.
    }
    \vspace{-15pt}
    \label{tab:classification_using_thm_one}
\end{table}

\subsection{A Case Study with CNNs}
\vspace{-5pt}
\looseness=-1
As a warmup, let us consider the case of vanilla CNNs that possess translation symmetry. In this case, $G=(\sZ^{2},+)$ the group of translations of the plane and $D_i=d^{2}$ where $d^{2}$ is the size of a feature map at layer $i$. Finally, $\rho_{i}$ acts on the feature space $\sR^{d^{2}}$ by translating the coordinates of a point in the plane. The equivariant basis of $f_i$ in this setting (what we denoted $\gB_{i\to i+1}$ in the general case) are convolutions with kernels $\tK_{i}^f \in \sR^{d^{2}\times n_{i} \times n_{i+1}}$
that are built using the canonical basis and $n_i$ and $n_{i+1}$ are the input/output channels. We can apply Thm.~\ref{thm:main_approx_network} to achieve Cor.~\ref{theorem_CNN} (see~\S\ref{app:proof_CNN} for details) which recovers~\citet[Thm. 3.1]{burkholz2022convolutional} and is a strict generalization of the result by \citet{da2022proving}.

\subsection{SLT for $\text{E}(2)$ Steerable Nets}
\vspace{-5pt}
\label{sec:e_two_steerable}
The Euclidean group $\text{E}(2)$ is the group of isometries of the plane $\sR^2$ and is defined as the semi-direct product between the translation and orthogonal groups of two dimensions $(\sR^2, + ) \rtimes \text{O}(2)$ with elements $(t,g) \in \text{E}(2)$ being shifts and planar rotations or flips. The most general method to build equivariant networks for $\text{E}(2)$ is in the framework of 
steerable $G$-CNN's where filters are designed to be \textit{steerable} with respect to the action of $G$ \citep{cohen2016steerable,weiler2018learning}. Concretely, steerable feature fields associate a $D$-dimensional feature vector to each point in a base space $f: \sR^2 \to \sR^D$ which transform according to their \emph{induced representation} $\left[\Ind^{(\sR^{2}\rtimes G)}_{G}\rho\right]$,
\begin{equation}
    f(x) \rightarrow \left(\left [\Ind^{(\sR^{2}\rtimes G)}_{G}\rho \right ](tg) \cdot f \right)(x) := \rho(g) \cdot f(g^{-1}(x-t)).
\end{equation}
Clearly, a RGB image---a scalar field---transforms according to the \emph{trivial representation} $\rho(g) = 1\,,\; \forall g \in G$, but intermediate layers may transform according to other representation types such as regular. As proven in \citet{cohen2019general}, any equivariant linear map between steerable feature spaces transforming under $\rho_{i}$ and $\rho_{i+1}$ must be a group convolution with $G$-steerable kernels satisfying the following constraint: $\pi_i(gx) = \rho_{i+1}(g)\pi_i(x)\rho_{i}(g^{-1}) \ \forall g \in G, x \in \sR^2$. An equivariant basis is then composed of convolutions with a basis of equivariant kernels that we compute next.

%\dfe{probabilitry one to have a basis by replacing one vector by the identity/ argument to construct our basis would be to show that theorem 1 is valid for zany basis (as opposed to the proof of da cunha which cannot be extended for example with their proof sketch to other basis}

One of the key ingredients needed to apply Theorem~\ref{thm:main_approx_network} is the availability of an equivariant basis with an identity element. One could in principle always take an existing equivariant basis, such as the one provided by \citet{weiler2019general}, and include an identity element by replacing the first basis element resulting in another equivariant basis with probability $1$. In what follows, we show the generality of Theorem \ref{thm:main_approx_network} by constructing a different equivariant basis from first principles via the canonical basis and then symmetrizing using the action of $G\leq \text{O}(2)$. As we show in our experiments, we can find winning tickets for both basis with negligible difference in performance.

\xhdr{Classification of Equivariant Maps for $\text{E}(2)$}
We now seek to precisely characterize which kernels satisfy the equivariance constraint.
Let $\gR$ be the equivalence relation on $\sR^{2}$,
\begin{equation}
    \gR:= \forall (x,y) \in \sR^{2} \times \sR^{2},\; x\sim y \Longleftrightarrow \exists g \in G \quad \text{such that} \ y=g\cdot x.
\end{equation}
The equivalence class of $x\in \sR^{2}$ denoted $\gO(x)$, is the orbit of $x$ under the action of $G$ on $\sR^{2}$.
Designate $\gA_{\gR}= \sR^{2}/\gR \subset \sR^{2}$ a set of representatives. 
% For instance, if $G=\gC_{4}$ (rotations by $\pi/2$), then the orbit  $\gO((1,1))= \{(1,1), (-1,1), (-1,-1), (1,-1) \}$ and $\gA_{\gR}$ is the upper right quadrant.
Due to the equivariance constraint on the kernels $\pi(\cdot)$, once the value of $\pi(x)$ is chosen, it automatically fixes $\pi(g \cdot x)$ for $g \cdot x \in \gO(x)$.
% which transform equivariantly as order-$2$ tensors, i.e. matrices.
Note that because $|\gO(x)|=|G|$, all possible initial matrices $\sR^{|G| \times |G|}$ can be chosen at a point $x \neq 0$.\footnote{Care must be taken at the origin, since $\forall g \in G, g \cdot 0 =0$, and the set of permissible matrices depends on $G$ as well as our choice of representations. We provide a thorough treatment of this case in~\S\ref{app:origin}.}

\xhdr{Remark} In practice, $G$-steerable equivariant networks do not operate on signals in $\sR^2$ but on a fixed size pixelized grid $\{1,2,\dots,d\}^{2}$ denoted as $[d]^{2} \subset \mathbb{Z}^2$. Henceforth, we consider all our target networks as well as the overparameterized $G$-steerable network to be defined on input signals sampled on $[d]^2$ and in appendix~\S\ref{app:discretization}
we highlight two practical challenges that result from such a discretization, but crucially these do not disrupt our subsequent theory nor pruning techniques.

\xhdr{Computing $\gB$}\footnote{$\gB$ is the basis of equivariant kernels. $\gB_{i \to i+1}$ is obtained by taking the 2D convolution with these elements.}
To explicitly build a basis of the $G$-equivariant layers, it is illustrative to first consider the case for a single input-output pair of representations for a layer---i.e. $n_i=n_{i+1}=1$. 
\begin{wrapfigure}{o}{0.5\textwidth}
  \vspace{-10pt}
    \includegraphics[width=1\linewidth]{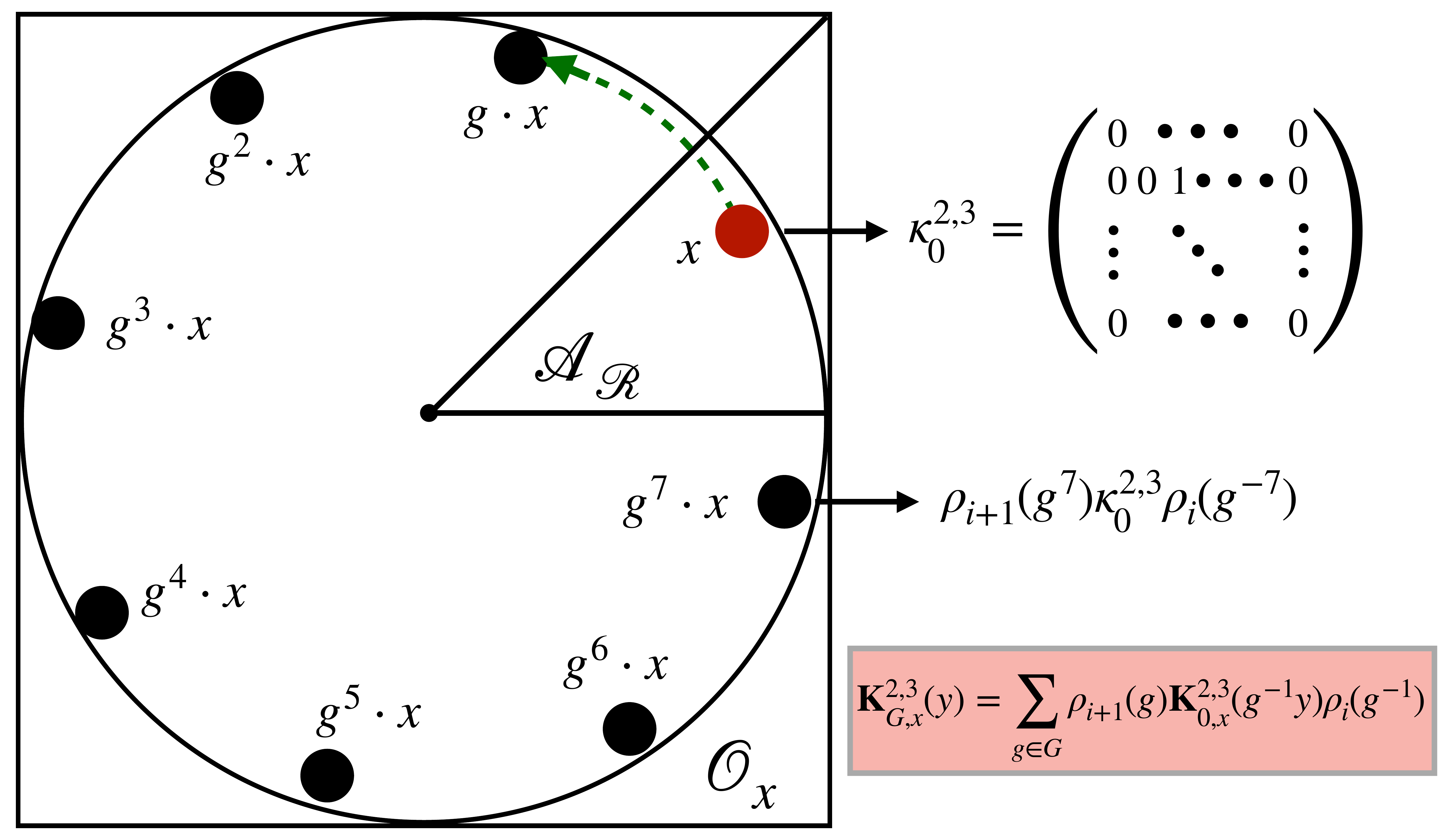}
      \vspace{-5pt}
\caption{Constructing $\tK_{G,x}^{2,3}\in \gB_x$ for $C_8$.}
  \vspace{-15pt}
\label{fig:e2_basis}
\end{wrapfigure}
We must first construct a basis of the equivariant kernels in our domain. 
Let $\gB = \{\gB_x,\, x \in \gA_{\gR}\}$ be a basis of equivariant kernels over the domain where each basis is a tensor of shape $\gB_{x} \subset \sR^{d \times d \times |G| \times |G|}$.
A single basis element $b \in \gB_{x}$ can be constructed by considering the canonical basis\footnote{Note that this canonical basis is of the same form (but different shape) as $\kappa_{n_i \to n_{i+1}}$ used for $\sF_{i}^{n_i}\to \sF_{i+1}^{n_{i+1}}$.} $\kappa_0 \subset \sR^{|G|\times |G|}$ at each location $x \in \gA_{\gR}$ and evaluating it under the action of the group. One can freely choose both a starting point $x \in \gA_{\gR}$, and an element of the canonical basis $\kappa_{0}^{p,q}$. Let $ \tK_{0,x}^{p,q} \in \sR^{d \times d \times |G| \times |G|}, \forall (p,q) \in [|G|] \times [|G|]$ be the tensor of $\kappa_0^{p,q}$ stacked across the grid---i.e. it is $0$ everywhere except at the index $(x,p,q)$ where it is $1$. Then to get the equivariant basis we symmetrize by acting on $\tK_{0,x}^{p,q}$ while enforcing the equivariance constraint.
\begin{equation}
    \forall y \in [d]^2 \quad b(y):= \tK_{G,x}^{p,q}(y) = \sum_{g \in G} \rho_{i+1}(g) \tK_{0,x}^{p,q}(g^{-1}y)\rho_{i}(g^{-1}).
\end{equation}
Repeating this procedure for all elements $\kappa_0^{p,q} \in \sR^{|G| \times |G|}$ in the canonical basis completes the construction of our basis $\gB_x = \{\tK_{G,x}^{p,q},\;p,q\in [|G|] \}$. We finally obtain a basis of the equivariant kernel as $\gB = \bigcup\limits_{x \in \gA_{\gR}}\gB_x$.
A $G$-steerable expanded kernel $\tK$ is then simply a linear combination of learned weights $\theta = [\theta_1, \dots, \theta_{|\gB|}]$---one for each basis element---$\tK(x) = \sum_{k=1}^{|\gB|} \theta_k b_{k}(x)$. In contrast, a standard convolution kernel has shape $\mathbb{R}^{d \times d \times c_{i} \times c_{i+1}}$ which means that the equivalent input/output channels for $G$-steerable convolutions are $c_{i} = |G| \times n_i$ and $c_{i+1} = |G| \times n_{i+1}$ respectively.
Fig. \ref{fig:e2_basis} illustrates the above process for a basis element for the $C_8$ group. Equipped with this basis, which has an identity element at the origin (~\S\ref{app:origin}). We can now apply Thm.~\ref{thm:main_approx_network} to get:

\begin{mdframed}[style=MyFrame2]
\begin{restatable}{corollary}{theorem_GCNN}
\label{cor:e2CNN}
Let $h \in \gH$ be a random $G$-steerable CNN with regular representation of depth $2l$, i.e., $h(\mathbf{x})=\mathbf{K}^h_{2l-1}*\sigma \left(\dots \sigma(\mathbf{K}^h_{0}*\rvx)\right)$ where $\mathbf{K}^h_{2i} \in \sR^{d^{2}\times |G|^2\times n_{i}\times \tilde{n}_{i}}$, $\mathbf{K}^h_{2i+1}\in \sR^{d^{2}\times |G|^2 \times \tilde n_{i} \times n_{i+1}}$ are equivariant kernels whose decomposition in $\gB$ have coefficients drawn from $\gU([-1,1])$.
If $\tilde{n_{i}}=C_3n_{i} \log \left(\frac{n_{i}n_{i+1}d^{2}|G|^3l}{\min(\epsilon, \delta)}\right)$, then with probability at least $1-\delta$ we have that for all $f \in \gF$ (whose kernels $\mathbf{K}_{i}^{f}$ have parameters less than 1, and with $\vertiii{f_i}\leq 1$) there exists a collection of pruning masks $\mathbf{S}_{2l-1}, \dots, \mathbf{S}_0$ such that, by defining $\mathbf{\tilde{K}}^h_{i}$ the kernel associated with $\mathbf{S}_i \odot \rmW_i^h$,
\begin{equation}
    \max_{\mathbf{x}\in \sR^{d^2\times n_0},\,\|\mathbf{x}\|\leq 1} \|\mathbf{\tilde{K}}^h_{2l-1} *\sigma \left(\dots \sigma(\mathbf{\tilde{K}}^h_{0} *\rvx)\right) - f(\rvx)\| \leq \epsilon
\end{equation}
\end{restatable}
\vspace{-4mm}
\end{mdframed}

\cut{
\xhdr{SLT for $2$-layered $G$-steerable networks}
Let us concern ourselves with the particular setting where: (i) $G \leq \text{O}(2)$ is finite (ii) the steerable field at layer $k$ is of the form $f_{k}: \{d\}^{2} \rightarrow \sR^{n_{k} \times |G|}$ (iii) $G$ acts on $\sR^{n_{k} \times |G|}$ by ways of $n_{k}$ direct sum regular representations.
Furthermore, let $\kappa_{b,x}[W] =\sum_{i,j}w_{i,j}\kappa_{b,x}^{i,j}(x)$ such that $W_{i,j}=w_{i,j}$. Then, each $G$-equivariant layer between two steerable feature fields can be written as a convolution using our previously computed expanded kernel applied to an input $\bold{K}_{\text{equivar}}(x) = \sum_{x \in \gA_{\gR}}^{}{}\kappa_{b,x}[W_x](x)$.

Recall that for $n_{in}=n_{out}=1$, the kernel at origin was required to be a circulant matrix $W_0$. In the general case where $n_{in} \neq 0$ and $n_{out}\neq 0$, we can decompose $W_x$ as $\sum_{i,j}\kappa^{i,j} \otimes W_{x}^{i,j} \text{ where } W_{x}^{i,j} \in \sR^{|G| \times |G|}$. Clearly,  $W_{x \neq 0}$ can be every matrix from $\sR^{|G| \times |G|}$ and  $W_0$ must still be a circulant matrix in the general case. Define $P_i, i \in [|G|]$ be the $i$-th circulant matrix as a basis for the circulant matrices then we can decompose $\bold{K}_{\text{equivar}}$ completely in this basis as follows:

% \begin{align}
%     \mathbf{K}_{\text{equivar}}(x) &= \sum_{x \in \gA_{\gR} \fgebackslash 0}^{}{}{\kappa}_{b,x}\left[\sum_{i=1, j=1}^{n_{in}, n_{out}}\kappa_0^{i,j} \otimes \left(\sum_{k,l=1}^{|G|, |G|} w_{x,i,j,k,l}\kappa_0^{k,l}\right)\right](x) +\kappa_{b,0}\left[\sum_{i=1, j=1}^{n_{in}, n_{out}}\kappa_0^{i,j} \otimes \left(\sum_{k=1}^{|G|} w_{0,i,j,k}P_{k}\right)\right](x)\\
%     & = \sum_{x \in \gA_{\gR}}^{}{}\kappa_{b,x}\left[\sum_{i=1, j=1}^{n_{in}, n_{out}}\kappa_0^{i,j} \otimes W_{x,i,j}\right](x) \nonumber
% \end{align}

\begin{align}
    \mathbf{K}_{\text{equivar}} &= \sum_{x \in \gA_{\gR} \fgebackslash 0}{\kappa}_{b,x}\left [ \sum_{i=1, j=1}^{n_{in}, n_{out}}\kappa_0^{i,j} \otimes W_x^{i,j}\right ] +\kappa_{b,0}\left[\sum_{i=1, j=1}^{n_{in}, n_{out}}\kappa_0^{i,j} \otimes \left(\sum_{k=1}^{|G|} w_{0,i,j,k}P_{k}\right)\right]\\
    & = \sum_{x \in \gA_{\gR}}^{}{}\kappa_{b,x}\left[\sum_{i=1, j=1}^{n_{in}, n_{out}}\kappa_0^{i,j} \otimes W_x^{i,j}\right] \nonumber,
\end{align}

where we factorized $W_x^{i,j}=\sum_{k,l=1}^{|G|, |G|} w_{x,i,j,k,l}\kappa_0^{k,l}$ and $W_{x,i,j}=\sum_{k}^{|G|} w_{0,i,j,k}P_{k}$ for for $x\neq 0$ and $x=0$ respectively. Decomposing $\tK_{\text{equivar}}$ it becomes clear that to preserve equivariance we must only prune parameters $w_{x,i,j,k,l}$ and the $w_{0,i,j,k}$. We are now ready to state Lemma \ref{lemma_two}, which asserts, that with high probability we can prune a random overparametrized  $G$-equivariant network of depth $2$ to approximate a single target $G$-equivariant layer with at most $\epsilon$ error.

% This leads to following Lemma which states that for any random overparametrized one-hidden layer equivariant network with pointwise ReLU, with high probability, for any equivariant linear mapping, one can find an equivariant pruning mask (ie a pruning mask on the parameters $w_{x,i,j,k}$) such that the pruned overparametrized network approximates the equivariant target linear map by at most $\epsilon$ error.

\begin{mdframed}[style=MyFrame2]
\begin{lemma}
\label{lemma_two}
Let F be a random equivariant network of depth $l=2$: $F(\mathbf{x})=\mathbf{K^{2}}\sigma(\mathbf{K^{1}x})$ where $\mathbf{K_{\text{equivar}}^{1}}:(\{d\}^{2}\rightarrow \sR^{n_{1}\times d})\rightarrow (\{d\}^{2} \rightarrow \sR^{\tilde{n}\times d})$ and $\mathbf{K_{\text{equivar}}^{2}}:(\{d\}^{2}\rightarrow \sR^{\tilde{n}\times d})\rightarrow (\{d\}^{2} \rightarrow \sR^{n_2\times d})$ are two general equivariant random layers with weights $w_{x,i,j,k,l}$ and $w_{0,i,j,k}$ taken from $\gU([-1,1])$. Let's suppose that $\tilde{n}=Cn_{1} \log\left(\frac{n_{1}n_{2}}{\min(\epsilon, \delta)}\right)$ then with probability at least 1-$\delta$, $\forall \mathbf{K}: (\{d\}^{2}\rightarrow \sR^{n_{1}\times d})\rightarrow (\{d\}^{2} \rightarrow \sR^{n_2\times d})$ equivariant mapping, where $\mathbf{K}$ is the convolution with expanded kernel $\bold{K}_{\text{equivar}}(x)=\sum_{x \in \gA_{\gR}}^{}{}\kappa_{b,x}\left[\sum_{i=1, j=1}^{n_{in}, n_{out}}\kappa_0^{i,j} \otimes \tilde{W}^{i,j}_{x}\right](x)$ such that $\| \tilde{W}_{x}^{i,j} \| \leq 1$, one can prune F in an equivariant way such that:

\begin{equation}
    \max_{\mathbf{x}\in [-1,1]^{\{d\}\times \{d\}\times (n_1\times d)}} \|(\mathbf{S_2} \odot\mathbf{K^{2}})\sigma(\mathbf{(\mathbf{S_1} \odot K^{1})x}) - \mathbf{K}\mathbf{x}\| \leq \epsilon
\end{equation}
where $\tilde{W}$ are the targets weights, $\sigma$ is the ReLU, and $\tS_1$ and $\tS_2$ are the pruning masks.

\end{lemma}
\end{mdframed}

The next figure illustrates the setting of the proof.

\dfe{put figure here}
\begin{proof}[Proof sketch] 
\jb{This proof sketch is not readable. We need to rewrite it to be less wordy and with significantly more intuition}

The first step in the proof is to bypass the ReLU which we achieve by pruning the first layer $\tK^1$ in a specific manner such that we can simply apply Lemma \ref{lemma_one}. Specifically, we prune all $W^{i,j}_{x}$ for $x \neq 0$ so that $\mathbf{K}^{1}$ is a convolution with a kernel of size $1 \times 1$. Also, for $W^{i,j}_{0}$ we prune all $w_{0,i,j,k}$ except for $k=0$, so that $W^{i,j}_{0}=w_{0,i,j,0}P_0=w_{0,i,j,0}\mathbb{I}_{|G|}$ where $\mathbb{I}_{|G|}$ is the identity. Finally, we obtain the "diamond shape" by pruning $w^{1}_{0,i,j,0}$ for all $(i,j)$ except the ones where $j \in f(i)$ where $f$ is defined by $f(i) = [(i-1)\times C\log(\frac{n_{1}n_{2}}{min(\epsilon, \delta)})+1, i\times C\log(\frac{n_{1}n_{2}}{min(\epsilon, \delta)})] \bigcap \mathbb{N}$.

Almost, we can consider that the ReLU is bypassed. Indeed we will only need to tackle the two different cases of a positive or negative input coordinate. $\mathbf{K^{2}}\sigma(\mathbf{K^{1}x})$ becomes $\mathbf{K^{2}}(\mathbf{K^{1}_+x_+}+\mathbf{K^{1}_-x_-})$. By associativity, we only want that $\mathbf{K^{2}}\mathbf{K^{1}_+} \simeq \mathbf{K}$ and $\mathbf{K^{2}}\mathbf{K^{1}_-} \simeq \mathbf{K}$ But we can show that $\mathbf{K^{2}}\mathbf{K^{1}_+}$ is the convolution with kernel $\bold{K}_{equivar}(\cdot) = \sum_{x \in \gA_{\gR}}^{}{}\kappa_{b,x}\left[\sum_{i=1, k=1}^{n_{1}, n_{2}}\kappa^{i,k} \otimes \left( \sum_{j\in f(i)}W^{2}_{x,j,k}W^{1}_{0,i,j,+}\right)\right](\cdot)$.
We basically want that $\sum_{j\in f(i)}W^{2}_{x,j,k}W^{1}_{0,i,j,+}$ approximates $W^{target}_{x,i,k}$. We achieve that by, $\forall x \in \gA_{\gR}, \forall (i,k) \in [n_{in}] \times [n_{out}]$ solving subset sum problems on each coefficient of the $|G|\times |G|$ matrix $W^{target}_{x,i,k}$, by putting a pruning mask on
    $W^{2}_{x,j,k}$. We can then prune $\bold{K^{2}}$ again to have that $\mathbf{K^{2}}\mathbf{K^{1}_-} \simeq \mathbf{K}$.
\end{proof}
Now that we have this approximation lemma, we can apply this lemma $l$ times on a $2l$-layers overparametrized network to approximates a target $l$-layers network.

Denote by $\mathcal{F}$ the set of target ReLU $G$-steerable CNN such that : (i) $f:\left(\{d\}^{2} \rightarrow \sR^{n_{0} \times |G|}\right) \rightarrow \left(\{d\}^{2} \rightarrow \sR^{n_{l}\times |G|}\right)$, (ii) $f$ has depth $l$, (iii) $\forall x \in \gA_{\gR}, \forall k \in [l], \forall (i,j) \in [n_{k-1}] \times [n_{k}]$ the weight matrix of layer $k$ (denoted $W^{k}_{x,i,j}$) has spectral norm at most 1. That is 

$\mathcal{F}=\{f:f(\mathbf{f_{in}})=\mathbf{K^{l}}\sigma(\mathbf{K^{l-1}}...\sigma(\mathbf{K^{1}f_{in}})), \forall k$  $\mathbf{K^{k}}: (\{d\}^{2} \rightarrow \sR^{n_{k-1}\times |G|})\rightarrow (\{d\}^{2} \rightarrow \sR^{n_{k}\times |G|})$ where $\mathbf{K^{k}}$ is the convolution with kernels of layer $k$: $\bold{K}_{equivar}^{k}(.)=\sum_{x \in \gA_{\gR}}^{}{}\kappa_{b,x}\left[\sum_{i=1, j=1}^{n_{k-1}, n_{k}}\kappa_0^{i,j} \otimes W^{k}_{x,i,j}\right](\cdot)$ where $W^{k}_{x,i,j} \in \sR^{|G| \times |G|}$, $\parallel W^{k}_{x,i,j} \parallel  \leq 1\}$

\begin{mdframed}[style=MyFrame2]
\begin{theorem}
\label{thm:g_prune}
Let $\gF$ be as defined above. Consider a randomly initialized 2l-layered G-steerable equivariant neural network:
$g(\mathbf{x})= \mathbf{K^{2l}}\sigma(\mathbf{K^{2l-1}} \dots \sigma(\mathbf{K^{1}x}))$ where $\mathbf{K^{2k}}$ is the convolution with kernels $\bold{K}_{equivar}^{2k}(\cdot)=\sum_{x \in \gA_{\gR}}^{}{}\kappa_{b,x}\left[\sum_{i=1, j=1}^{\tilde{n}_{k-1}, n_{k}}\kappa_0^{i,j} \otimes W^{2k}_{x,i,j}\right](\cdot)$ and $\mathbf{K^{2k-1}}$ is the convolution with kernels $\bold{K}_{equivar}^{2k-1}(\cdot)=\sum_{x \in \gA_{\gR}}^{}{}\kappa_{b,x}\left[\sum_{i=1, j=1}^{n_{k-1}, \tilde{n}_{k-1}}\kappa_0^{i,j} \otimes W^{2k-1}_{x,i,j}\right](\cdot)$, where every weight of the $W_{x,i,j}^{k}$ for $x\neq 0$ as well as the $w^{k}_{o,i,j,l}$ are taken in $\gU([-1,1])$. Let's take $\tilde{n_{k}} = Cn_{k}\log(\frac{d^{2}n_{k}n_{k+1}l}{min(\epsilon, \delta)})$.
Then with probability at least 1-$\delta$, $\forall f \in \gF$ such that $\|f\| \leq 1$ (to be defined), one can prune the matrices $W_{x,i,j}^{k}$ such that:

\begin{equation}
    \max_{\mathbf{x}\in [-1,1]^{\{d\}\times \{d\}\times n_0*|G|}} \|(\mathbf{S_{2l}} \odot \mathbf{K^{2l}})\sigma \left( \left(\mathbf{S_{2l-1}} \odot \mathbf{K^{2l-1}}\right) \dots \sigma \left(\left(\mathbf{S_{1}}\odot \mathbf{K^{1}}\right)\mathbf{x}\right)\right) - f(\mathbf{x})\| \leq \epsilon
\end{equation}

\cut{the resulting equivariant network approximates f by at most an error $\epsilon$ $\forall \mathbf{x}$, $\| \mathbf{x} \| \leq 1$}
\end{theorem}
\end{mdframed}

Taking $G = \{e\}$ ($|G|$ becomes 1), one recovers the result from \citet{da2022proving}, but with the extension that the entries do not need to be positive.
% \ggi{Make the diff with the CNN paper and what are the difficulties of extending the results to Equivariant nets}
%\ggi{Can we really extend it to $E(n)$?}
}

In Appendix \S\ref{app:steerable}, we compute $\max(|\gB_{i\to i+1}|, \vertiii{\gB_{i \to i+1}})$ that leads to the corollary above.

\subsection{SLT for 
Permutation Equivariant Nets}
\vspace{-5pt}
The symmetric group $\gS_n$ consists of all permutations that can be enacted on a set of cardinality $n$. The action of $\gS_n$ on a tensor $\tX \in \sR^{n^{k} \times m}$ is defined by permuting all but last index: $(g \cdot \tX)_{i_1, \dots, i_k, j} = (\tX_{g^{-1}(i_1), \dots, g^{-1}(i_k)}, j), \forall g \in \gS_n$. Any general linear permutation equivariant map $\rmW_i: \sR^{n^{k_i}} \rightarrow \sR^{n^{k_{i+1}}}$, must satisfy the following fixed point equation: $\rmP^{\otimes(k_i+k_{i+1})} \Vect(\rmW_i)=\rmW_i$,
where $\rmP^{\otimes(k_i+k_{i+1})}$ is the $(k_i+k_{i+1})$ Kroenecker power of a permutation matrix $\mathbf{P}$ \citep{maron2018invariant}.
General permutation equivariant networks are the concatenation of linear equivariant layers followed by pointwise non-linearities, which aligns with the setting needed to apply Theorem~\ref{thm:main_approx_network}.

\xhdr{Classification of all Linear Permutation Equivariant Maps}
In \citet{maron2018invariant}, the authors solve the above fixed point equation by first defining the equivalence relation $\gQ$ on $[n]^{k_i+k_{i+1}}$ as: 
\begin{equation}
   \gQ:= \forall a,b \in [n]^{k_i+k_{i+1}}, \, a\sim b \Leftrightarrow (\forall i,j \in [k_i+k_{i+1}], a_{i}=a_{j} \Leftrightarrow b_{i}=b_{j}).
\end{equation}
Now for all $\mu \in [n]^{k_i+k_{i+1}}/\gQ$ define the matrix
 $B^{\mu} \in \sR^{n^{k_i} \times n^{k_{i+1}}}$ such that each entry $ B^{\mu}_{a,b}=\mathbbm{1}_{(a,b)\in \mu}$ \footnote{$\mathbbm{1}_{(a,b)\in \mu}=1 \text{ if } (a,b) \in \mu \text{ and } 0 \text{ otherwise}$, for $a\in [n]^{k_i}$ and $b\in [n]^{k_{i+1}}$}. Then a basis for equivariant maps is $\gB_{i \to i+1}=\{ B^{\mu}, \mu \in [n]^{k_i+k_{i+1}}/\gQ\}$. The cardinality of this basis $|\gB_{i \to i+1}|=\Tilde{b}(k_i+k_{i+1})$ is known as the $(k_i+k_{i+1})$-th Bell number and can be understood as the number of ways to partition $[n]^{k_i+k_{i+1}}$. When $k_i=k_{i+1}$, the identity element is not in the basis, therefore we replace $B^{(1, \dots, 1)}$ by $\sum_{a \in [n]^k/\gQ}B^{(a,a)}=\sI$, which is still a basis. We are now in a position to apply Theorem \ref{thm:main_approx_network} to permutation equivariant networks.

%We will use $\|\cdot \|_{\infty}$ which is a classic metric to compute $\vertiii{\gB_{i \to i+1}}$. $\sigma$ is trivially one-Lipschitz.

\begin{mdframed}[style=MyFrame2]
\begin{restatable}{corollary}{theorem_mpgnn}
\label{cor:perm}
Let $h \in \gH$ be a random permutation equivariant network of depth $2l$, i.e., $h(\rvx)=\!\rmW^h_{2l-1}\sigma \left(\dots \sigma(\rmW^h_{0}\rvx)\right)$ where $\!\rmW^h_{2i} \in \sR^{n^{k_i}\times n_i\times n^{k_{i}}\times \tilde{n}_{i}}$, $\!\rmW^h_{2i+1}\in \sR^{n^{k_i}\times \tilde n_{i} \times n^{k_{i+1}} \times n_{i+1}}$ are equivariant layers whose decomposition in $\gB$ have coefficients drawn from $\gU([-1,1])$.
If $\tilde{n_{i}}=C_2n_{i} \log \left(\frac{n_{i}n_{i+1}\max(\Tilde{b}(k_i+k_{i+1}), n^{k_i}+1)l}{\min(\epsilon, \delta)}\right)$, then with probability at least $1-\delta$ we have that for all $f \in \gF$ (\cut{a target permutation equivariant network whose layers $f_i$ have weights less than 1 when decomposed in $\gB_{i\to i+1}$ and }with $\vertiii{f_i}\leq 1$ and parameters in the basis less than 1) there exists a collection of pruning masks on the decomposition in the equivariant basis of the layers $\mathbf{S}_{2l-1}, \dots, \mathbf{S}_0$ s.t.,
\begin{equation}
    \max_{\mathbf{x}\in \sR^{n^{k_0}\times n_0},\,\|\mathbf{x}\|\leq 1} \|(\mathbf{S}_{2l-1} \odot\rmW^h_{2l-1}) \sigma \dots \sigma(\left(\mathbf{S}_0 \odot \rmW^h_{0} \right)(\rvx)) - f(\rvx)\| \leq \epsilon
\end{equation}
\end{restatable}
\vspace{-4mm}
\end{mdframed}
We discuss in Appendix \S\ref{app:permutation_proof} the computation of $\vertiii{\gB_{i \to i+1}}$, and provide the detailed proof.

\xhdr{Message Passing GNNs}
MPGNNs are networks that act on graphs with $n$-nodes by defining a feature vector for each node which is updated based on ``messages" received from its neighbors which are then combined. Given a node $v$ in a graph and its hidden representation $x^v_i$, the message passing update for a layer $i$ is governed by the following equation: $x^v_{i}=f_{i}^{\text{up}}(x^v_{i-1}, \sum_{u \in \mathcal{N}(v)}f^{\text{agg}}_{i}(x^v_{i-1}, x^u_{i-1}))$. In its most general form the aggregation function $f^{\text{agg}}_i$ and update function $f^{\text{up}}_i$ are taken to be MLPs. 
In this case it is easy to see that Theorem \ref{thm:main_approx_network} can be applied separately to both $f^{\text{agg}}_i, f^{\text{up}}_i$ independently as MLPs are captured under $G = \{e \}$. Permutation in/equivariance is trivially maintained in the pruned network as the aggregate function operates on a local neighborhood of $v$ and pruning does not impact this as pruning does not impose any ordering over the nodes or the adjacency matrix in the graph.

\section{Experiments}
\vspace{-5pt}
\label{sec:experiments}
\looseness=-1
We substantiate our equivariant framework to finding winning SLTs by approximating target $G$-steerable networks, MPGNNs, and $k$-order GNNs on standard image classification, node and graph classification tasks
respectively. For steerable networks we consider networks for $G \in \{C_4, C_8, D_4 \}$ which are finite subgroups of $\text{O}(2)$. To show the generality of our framework, we experiment with two different equivariant basis for $\text{E}(2)$; the first one uses spherical harmonics and is taken from \citet{weiler2019general} (\textsc{default}), while the second is the one we introduce in~\S\ref{sec:e_two_steerable} (\textsc{ours}). MPGNNs and $k$-order GNNs naturally operate on $\gS_n$ where permutation invariance is with respect to the node labels of a given graph. 
For $\text{E}(2)$-steerable, we experiment with Rotation and FlipRotation-MNIST datasets which contain data augmentations from $G\leq \text{SO}(2)$ and $G \leq \text{O}(2)$ respectively \citep{weiler2019general}. To evaluate MPGNNs and $k$-order GNNs we consider standard node classification benchmarks in citation networks in Cora and CiteSeer \citep{sen2008collective} and real-world graph classification datasets in Proteins and NCI1 \citep{yanardag2015deep}.

% \xhdr{Methodology}
\looseness=-1
We find equivariant strong lottery tickets by utilizing our overparametrization strategy described in~\S\ref{sec:general_slt} by solving \textsc{Subset-Sum} problems using Gurobi \citep{gurobi2018gurobi}. The definition of the \textsc{Subset-Sum} problems as mixed-integer optimization problems can be found in eq.~\ref{app:eq:subset_sum} of \S\ref{app:experimental_detail}. In Table~\ref{tab:main_pruning_results} we report our main results for an overparametrization constant $C=5$ (see Thm. \ref{thm:main_approx_network}) towards approximating a single target network using $5$ random seeds to construct our overparametrized network. Specifically, we report the ratio of the number of parameters in the overparametrized and final pruned network divided by the original target network. We also report test accuracies for both, the maximum absolute weight error over all \textsc{Subset-Sum} problems, and the maximum relative output error between pruned and target networks. All model architectures and described in~\S\ref{app:experimental_detail}.

% \xhdr{Results}
\looseness=-1
For all equivariant architectures and datasets considered, we find that we
are able to approximate the corresponding trained target networks sufficiently well. Specifically, we achieve sufficiently low maximum relative output error across test samples such that the test accuracy of the resulting pruned network matches essentially that of the target one for all random seeds of the pruning experiments. 
% For $\text{E}(2)$ networks both equivariant basis considered work well with regards to finding winning tickets for all of the groups. In particular, we find that the maximum relative output error is $\approx 1.0e^{-3}$ for the default basis which indicates a high degree of output similarity between the pruned and target networks. This is also reflected in the accuracy of the final pruned network, which is essentially the same as the one achieved by the target one.
Finally, we conduct an ablation study on the effect of overparametrization constant factor $C$ to the approximation accuracy with respect to the tolerance $\epsilon$. We perform this study for the $\text{E}(2)$ equivariant architectures for different subgroups. In Fig. \ref{fig:ablation_fig} we plot this as a function of $C \in \{1, 2, 5, 10\}$ for the groups $C_4, C_8, D_4$ using the basis construction from \citet{weiler2019general}. As observed, increasing our overparametrization factor leads, up to $C=5$, to a lower maximum relative output error while the pruned accuracy marginally increases.
\begin{figure}[!ht]
    \vspace{-10pt}
    \centering
    \includegraphics[width=1.0\linewidth]{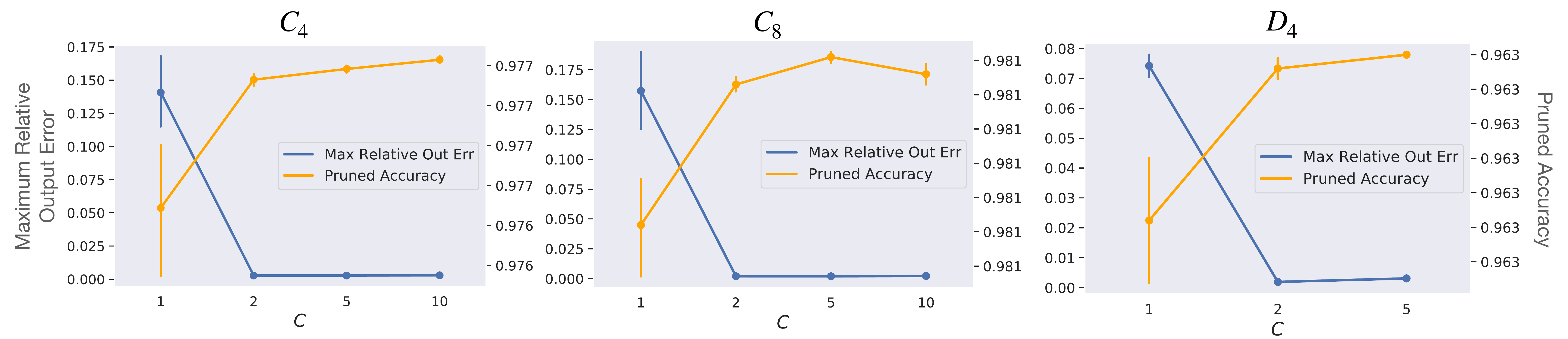}
    \vspace{-20pt}
    \caption{\small Ablation study of max. relative output error and pruned accuracy w.r.t. to $C$ for $C_4, C_8, D_4$.}
    \vspace{-15pt}
    \label{fig:ablation_fig}
\end{figure}
\begin{table}[ht]
\begin{scriptsize}
    \centering
    \begin{tabular}{llcccccc}
        \toprule
        \multicolumn{2}{c}{Task} & \multicolumn{2}{c}{Ratio $\nicefrac{p}{p_{\text{target}}}$} & \multicolumn{2}{c}{Accuracy ($\%$)} & \multicolumn{2}{c}{Errors} \\
        \cmidrule(lr){1-2} \cmidrule(lr){3-4} \cmidrule(l){5-6}  \cmidrule(l){7-8}
        Arch. & Dataset & overparam. & pruned & target & pruned$^\dagger$ & param. & output \\
        \midrule
        \multirow{2}{*}{MPGNN}
        {}    & Cora & $2.0e^4$ & $99.1$ & $80.2$ & $80.2$ & $6.2e^{-3}\pm3.3e^{-3}$ & $3.7e^{-4}\pm0.5 e^{-4}$ \\
        {}    & CiteSeer$^{*}$  & $5.5e^4$  & $103.8$ & $63.4$ & $63.4$ & $1.1e^{-1} \pm 1.7e^{-1}$ & $1.8e^{-3} \pm 1.7e^{-3}$ \\
        \midrule \multirow{2}{*}{$k$-order GNN$^{*}$}
        {} & Proteins & $3.6e^2$ & $51.8$ & $81.08$ & $81.08$ &  $1.2e^{-4}\pm4.2e^{-5}$ & $8.8e^{-4}\pm2.9e^{-4}$ \\
        {} & NCI1 & $5.3e^2$ & $85.9$ & $74.21$ & $74.21$ &  $8.1e^{-2}\pm5.2e^{-2}$ & $2.6e^{-2}\pm9.9e^{-3}$ \\
        \midrule
        E2-C4-\textsc{Default}    & \multirow{4}{*}{\shortstack[l]{Rot-\\MNIST}}
        {} & $3.0e^2$ & $41.0$ & $97.7$ & $ 97.7$ & $1.4e^{-3}\pm0.0e^{-3}$ & $2.7e^{-3}\pm0.4e^{-3}$ \\
        E2-C8-\textsc{Default}    & 
        {} & $3.1e^2$ & $41.2$ & $98.1$ & $98.1$ & $1.2e^{-3}\pm0.2e^{-3}$ & $1.9e^{-3}\pm0.2e^{-3}$ \\
        E2-C4-\textsc{Ours}$^{*}$    & 
        {} & $2.9e^2$ & $122.3$ & $96.2$ & $96.2$ & $6.2e^{-1}\pm5.8e^{-1}$ & $2.4e^{-2}\pm0.3e^{-2}$ \\
        E2-C8-\textsc{Ours}$^{*}$    & 
        {} & $3.0e^2$ & $120.9$ & $96.8$ & $96.8$ & $2.1e^{-2}\pm0.6e^{-2}$ & $3.3e^{-1}\pm3.7e^{-1}$ \\
        \midrule
        E2-D4-\textsc{Default}    & \multirow{2}{*}{\shortstack[l]{FlipRot-\\MNIST}}
        % {} & $3.1e^2$ & $40.7 $ & $96.3$ & $96.3$ & $9.6 e^{-4}$ & $3.1e^{-3}$ \\
        {} & $3.1e^2$ & $77.0 $ & $96.3$ & $96.3$ & $4.1e^{-2}\pm4.2e^{-2}$ & $1.4e^{-2}\pm1.1e^{-2}$ \\
        % {} & $3.1e^2$ & $123.1$ & $94.1$ & $94.1$ & $1.1e^{-1}$ & $4.4e^{-2}$ \\
        E2-D4-\textsc{Ours}    & 
        {} & $3.1e^2$ & $115.4$ & $94.1$ & $94.1$ & $2.4e^{-1}\pm0.9e^{-1}$ & $4.2e^{-2}\pm1.2e^{-2}$ \\
       % \midrule
        %Permutation$^{*}$ & Proteins & $3.6e^2$ & $51.8$ & $81.08$ & $81.08$ &  $1.2e^{-4}\pm4.2e^{-5}$ & $8.8e^{-4}\pm2.9e^{-4}$ \\
        \bottomrule
    \end{tabular}
    \caption{
    \small
    \looseness=-1
    Pruning random overparameterized $G$-equivariant networks to approximate $G$-equivariant targets. We report a) $\nicefrac{p}{p_{\text{target}}}$ the parameter ratio of the number of parameters $p$ of the overparametrized or the final pruned networks over $p_{\text{target}}$, b) the test accuracy of the target and the pruned networks, c) the maximum absolute weight error over subset sum problems, d) and the relative output errors of the pruned network in contrast to the target over samples in the test set. $^\dagger$STDs are below $1e^{-4}$. $^{*}$Maximum time of MIP solver for \textsc{Subset-Sum} problems was thresholded to 600ms.
    }
    \vspace{-25pt}
    \label{tab:main_pruning_results}
\end{scriptsize}
\end{table}
\section{Discussion}
\vspace{-5pt}
\label{sec:discussion}
This paper introduces a unifying framework to prove the strong lottery ticket hypothesis for general equivariant networks. We prove the existence with high probability of winning tickets for randomly (logarithmically) overparameterized networks with double the depth. We also theoretically demonstrate such an overparametrization scheme is optimal as a function of the tolerance. While our presented theory is built using overparametrized networks of depth $2L$ it may be possible to extend Theorem~\ref{thm:main_approx_network} to the setting where overparamatrized networks have depth $L+1$ as in \citet{burkholz2022most} by adapting the proof techniques. We leave this extension as future work.
Our framework enjoys broad applicability to MLPs, CNNs, $\text{E}(2)$-steerable networks, general permutation equivariant networks, and MPGNNs all of which become insightful corollaries of our main theoretical result. One limitation of our developed theory is the assumption of using a point-wise ReLU as the non-linearity. As a result, a natural direction for future work is to consider extensions of the \textsc{Subset-Sum} problem beyond linear functions to more general non-linearities. In addition, our overparametrization strategy employed the ``diamond shape" technique; however other schemes might also yield an optimal upper bound. Characterizing these schemes is an exciting direction for future work. 

\section{Ethics Statement}
\vspace{-5pt}
\label{app:ethics_statement}
The main contributions of this work are primarily theoretical in nature as we seek to provide a general framework to study equivariant lottery tickets. Consequently, any potential societal impact would necessarily be speculative in nature and deeply tied to a particular application domain. For example, one could consider the environmental cost savings from creating an overparametrized $G$-equivariant network that does not need any GPU hours to train, but instead CPU resources to solve \textsc{Subset-Sum} problems. Beyond these goals any application of our theory to actual practice is likely to inherit the complex broader impacts native to the problem domain and we encourage practitioners to exercise due caution in their efforts.

\section{Reproducibility Statement}
\vspace{-5pt}
We provide a complete proofs for all our theoretical results in the Appendix. In particular, proofs for Lemma \ref{Approximation_of_a_layer} can be found in Appendix \ref{app:proof_approximation_of_a_layer} and Theorem \ref{thm:main_approx_network} is a direct application of this result $l$-times and whose proof is located in \ref{app:proof_thm_one}. The proof for Theorem \ref{th:lowerbound} is located in Appendix \ref{app:lowerbound}. Furthermore, instantiations of framework for $\text{E}(2)$-steerable CNNs, permuation equivariant networks, MLPs, and vanilla CNNs resulting in corollaries \ref{cor:e2CNN}, \ref{cor:perm}, \ref{cor:SLTdense}, and \ref{theorem_CNN} respectively. The proofs for all the corollaries are located in Appendices \ref{app:pensia} (MLP), \ref{app:proof_CNN} (CNN), \ref{app:steerable_proof} ($\text{E}(2)$-CNN), and \ref{app:permutation_proof} (permutation equivariant networks). We provide full details on our experimental setup, including hyperparamters choices, architectures, and the exact \textsc{Subset-Sum} problem being solved for pruning in Appendix \ref{app:experimental_detail}. Finally, code to reproduce our experimental results can be found in submission's supplementary material.
\section{Acknowledgements}
\vspace{-5pt}
The authors would like to thank Louis Pascal Xhonneux, Mandana Samiei, Mehrnaz Mofakhami, and Tara Akhound-Sadegh for insightful feedback on early drafts of this work. In addition, the authors thank Riashat Islam, Manuel Del Verme, Mandana Samiei, and Andjela Mladenovic for their generous sharing of computational resources. AJB is supported by the IVADO Ph.D. Fellowship.

\clearpage
\bibliography{bibliography}
\bibliographystyle{abbrvnat}

\appendix
\onecolumn
\section{Additional material on the subset sum problem}
\label{app:subset_sum}
We recall here some results on subset sum originally from \citet{lueker1998exponentially} and modified by \citet{pensia2020optimal} to better fit the proof.

\begin{lemma}[\textsc{subset-sum} lemma]
\label{SS}
Let $U \simeq \gU([0,1])$ (or $\gU([-1,0])$ and $V \simeq \gU([-1,1])$ be two independent random variables. Let $P$ be the distribution of $UV$. Let $\delta_0$ be the dirac-delta function. Define a distribution $D=\frac{1}{2}\delta_0+ \frac{1}{2}P$. Let $X_1, \dots X_n$ be i.i.d. from the distribution D
where $n \geq C\log(\frac{2}{\epsilon})$ (for some universal constant $C$). Then, with probability at least $1-\epsilon$, we have

\begin{equation}
    \forall z \in [-1, 1], \exists S \subset [n] \quad \text{such that} \quad |z -\sum_{i\in S}X_i|\leq \epsilon
\end{equation}

\end{lemma}

This Lemma, is in fact a consequence of the corollary 3.3 from \citet{lueker1998exponentially} which states that as soon as a distribution contains a uniform distribution, one can achieve any target with exponentially small precision by \textsc{subset-sum}.

\xhdr{Extension to more general distributions}This allows us to extend the result Theorem \ref{thm:main_approx_network} to a more general setting, where the distribution of the random coefficients is not $\gU([-1,1])$ but contains a uniform distribution.

Let's say that a distribution $Z$ contains a uniform distribution $\gU([a,b])$ if there exist a distribution $Z_1$ and a constant $\zeta \in [0,1[$ such that:

\begin{align*}
    Z \coloneqq \zeta Z_1 +(1-\zeta)\gU([a,b])
\end{align*}

We want to extend the results of theorem \ref{thm:main_approx_network} to distributions containing $\gU([-a,a])$ for some $a > 0$

We follow therefore the same path as in \citet{pensia2020optimal} to prove Lemma \ref{SS} but with more general distributions. \citet{pensia2020optimal} already made a remark for this next extension that we state and prove here.

\begin{lemma}
Let $a >0$. Let $X$ and $Y$ be two independent random variables such that X contains $\gU([0,a])$ (or $\gU([-a,0])$) and Y contains $\gU([-a,a])$.
Then the PDF of the random variable $XY$ is such that:
\begin{align*}
    \exists A_a > 0, \quad f_{XY}(z)\geq A_a\log\left(\frac{a^2}{|z|}\right) \quad \text{for}|z|< a^2
\end{align*}
\end{lemma}

\begin{proof}
By the change of variable $\frac{X}{a}$ and $\frac{Y}{a}$, one can apply lemma 4 from \citet{pensia2020optimal} to get that if $\tilde{X} \sim \gU([0,a])$ (or $\gU([-a,a])$) and $\tilde{Y} \sim \gU([-a,a]]$ the PDF of $\tilde{X}\tilde{Y}$ is:

\begin{align*}
    \frac{1}{2a^2}log\left(\frac{a^2}{|z|}\right) \text{ if } |z| \leq a^2 \text{ and } 0 \text{ otherwise}
\end{align*}

Now we know by hypothesis that $\exists \alpha_X, \alpha_Y > 0$ such that $f_X \geq \alpha_Xf_{\gU([0,a])}$ and $f_Y \geq \alpha_Yf_{\gU([-a,a])}$.

Therefore, $f_{XY} \geq \alpha_X\alpha_Y f_{\tilde{X}\tilde{Y}}$ and finally, $f_{XY} \geq \frac{\alpha_X\alpha_Y}{2a^2}log(\left(\frac{a^2}{|z|}\right)$ if $|z|\leq a^2$
\end{proof}

\begin{lemma}
\label{contain_uniform}
Let $X$ and $Y$ be two independent random variables such that X contains $\gU([0,a])$ (or $\gU([-a,0])$) and Y contains $\gU([-a,a])$. Let $P$ be the distribution of $XY$.
Then there exists a distribution $Q$ and a scalar $B_a > 0$ such that:
\begin{align*}
    P= B_a\gU([-\frac{a^2}{2}, \frac{a^2}{2}]) + (1-B_a)Q
\end{align*}
\end{lemma}

\begin{proof}
This is a direct consequence of the lower bound on the PDF of $XY$ that was shown in the previous Lemma.
\end{proof}

Using Lemma \ref{contain_uniform} and Corollary 3.3 from \citet{lueker1998exponentially} leads immediately to the following result:

\begin{lemma}
\label{generalization_distribution}
Let $a > 0$, $X$ be a random variable containing $\gU([0,a])$ (or $\gU([-a,0])$) and $Y$ containing $\gU([-a,a])$. Let $X_1, \dots, X_n$ be $n$ iid random variables following the distribution $\frac{1}{2}\delta_0 + \frac{1}{2}P$ where $P$ is the distribution of $XY$. Then, if $n\geq C_a\log\left(\frac{2}{\epsilon}\right)$ (for some constant depending on $a$), with probability at least $1-\epsilon$, we have:
\begin{align*}
    \forall z \in [-1,1], \exists S \subset \{1, \dots, n\}\quad |z-\sum_{i\in S}X_i|\leq \epsilon
\end{align*}
\end{lemma}

\begin{proof}
This follows immediately from Corollary 3.3 in \citet{lueker1998exponentially} (by applying Markov's inequality) and Lemma \ref{contain_uniform}.
\end{proof}

\xhdr{Discussion}
This allows us to generalize Theorem \ref{thm:main_approx_network} to settings where the random overparametrized network has weights taken from a distribution which contains $\gU([-a,a])$. This includes almost all the usual settings, namely Gaussian, uniform, ... Indeed, the only thing to change is to no longer use Lemma \ref{SS} but Lemma \ref{generalization_distribution} at the same place in the proof and by assuming that the distribution of the parameters of the overparametrized network contains $\gU([-a,a])$ for some $a > 0$.

\section{Proof of the General SLT on equivariant networks using pointwise ReLU}

\subsection{Approximation of an equivariant target layer}
\label{app:proof_approximation_of_a_layer}
We now prove Lemma \ref{Approximation_of_a_layer} that is used to approximate a single layer in a $G$-equivariant target model.
\begin{mdframed}[style=MyFrame2]
\approxonelayer*
\end{mdframed}

\begin{proof}

Let us first recall that the main hypothesis needed for this lemma is to have an identity element in the basis $\sI \in \gB_{i\to i}$. We note that this is a very mild assumption since the identity is trivially equivariant between $\sF_i$ and $\sF_i$ and one can always choose to incorporate it in the basis. Consequently, we will designate the first element in our equivariant basis to be the identity $b_{i \to i,1}=\sI$. 

\xhdr{Remark}We choose $C_1 = 3C$ to ensure that ($C$ is the universal constant introduced in lemma \ref{SS}):
\begin{align*}
    C_1\log\left(\frac{n_{i}n_{i+1}\max\left(|\gB_{i\to i+1}|, \vertiii{\gB_{i\to i+1}}\right)}{\min\left(\epsilon, \delta\right)}\right)\geq C\log\left(\frac{4n_{i}n_{i+1}\max\left(|\gB_{i\to i+1}|, \vertiii{\gB_{i\to i+1}}\right)}{\min\left(\epsilon, \delta\right)}\right),
\end{align*}
which is true for the entire domain of variables we are interested in ($n_i, n_{i+1} \geq 1, \delta, \epsilon \leq \frac{1}{2} \text{ and } |\gB_{i \to i+1}|\geq 1$). It is easy to see that $3\log(x) \geq \log(4x)$ on $[2, +\infty[$ as $x^3 \geq 4x$ in this domain.

To begin, we first introduce a function, $\chi$, to identify blocks in our feature space $\sF_i^{n_i}$. In particular, we leverage the  ``diamond shape" structure (see Fig \ref{fig:general_equivar_pruning}) and define $\chi: [\tilde{n}_i] \to [n_{i}]$, such that it divides the intermediate layer of our overparametrized approximation into groups of  $C_1\log(\frac{n_{i}n_{i+1}\max(|\gB_{i\to i+1}|, \vertiii{\gB_{i\to i+1}})}{\min(\epsilon, \delta)})$ blocks which are linked with the same block in the first (i-th) layer. In other words, $\chi$ associates a block in $\sF^{\tilde{n}_i}$ in a surjective manner to a block in $\sF^{n_i}$. In a last piece of notation we will use $x_{\omega}$ to mean the $\omega$-th block of the feature space for $x$. For example, if $x \in \mathbb{R}^{n_i \times D_i}$ which is contained in the feature space $\sF^{n_i}_i$ of the $i$-th layer then $\omega \in [n_i]$ and $x_{\omega} \in \sF_i$ denotes the $\omega$-th vector of dimension $D_i$ in $x$. Finally, because $\omega$ is a dummy index, quite often we will replace it with appropriate layer index---e.g. $p,q,r$. With this in hand we can write the function $\chi(q)$ as follows:
\begin{align*}
    \chi(q)= \left\lfloor \frac{q-1}{C_1\log(\frac{n_{i}n_{i+1}\max(|\gB_{i\to i+1}|, \vertiii{\gB_{i\to i+1}})}{\min(\epsilon, \delta)})}\right\rfloor +1
\end{align*}

Before pruning, one has

\begin{align*}
    \rmW_{2i}^h=\sum_{p=1}^{n_i}\sum_{q=1}^{\tilde{n}_i}\sum_{k=1}^{|\gB_{i \to i}|}(\kappa_{n_i\to \tilde{n}_i}^{p,q}\otimes \lambda^{(i)}_{p \to q,k}b_{i \to i,k}).
\end{align*}

We begin pruning by annihilating all first layer coefficients not associated with the identity basis element ($k \neq 1$) $\lambda^{(i)}_{p\to q,k}$ and for $q \notin \chi^{-1}(p)$. This yields the following decomposition post-pruning,

\begin{align*}
    \rmW_{2i}^h=\sum_{p=1}^{n_{i}} \sum_{q \in \chi^{-1}(p)}\left(\kappa_{n_{i}\to \tilde{n}_i}^{p,q} \otimes \lambda^{(i)}_{p\to q,1}\sI\right).
\end{align*}

Note that we can write $p = \chi(q)$ leading to the following:

\begin{align*}
    \left(\rmW_{2i}^hx\right)_q=\lambda_{\chi(q) \to q,1}^{(i)}x_{\chi(q)}.
\end{align*}

After the $\sigma$, ---i.e. the pointwise-ReLU, one then has:

\begin{align*}
    \sigma(\rmW_{2i}^hx)_{q}=\sigma(\lambda^{(i)}_{\chi(q)\to q,1}x_{\chi(q)})= \lambda^{(i)+}_{\chi(q)\to q,1}x_{\chi(q)}^+ + \lambda^{(i)-}_{\chi(q)\to q,1}x_{\chi(q)}^-
\end{align*} where we used the fact that the ReLU is pointwise and the identity on scalars $\sigma(wx) = w^+x^+ + w^-x^-$. Expanding the second layer in its equivariant basis and using the above equation we get:
\begin{align}
    \left[\rmW^h_{2i+1}\sigma(\rmW^h_{2i}x)\right]_r&=\sum_{q=1}^{\tilde{n}_i}\left(\sum_{k=1}^{|\gB_{i\to i+1}|}\mu^{(i)}_{q\to r,k}b_{i\to i+1,k}\right)\sigma(\rmW_{2i}^hx)_{q}\\
    \cut{&=\sum_{q=1}^{\tilde{n_i}}\left(\sum_{k=1}^{|\gB_{i\to i+1}|}\mu^{(i)}_{q\to r,k}b_{i\to i+1,k}\right)(\lambda^{(i)+}_{\chi(q)\to q,1}x_{\chi(q)}^{+}+\lambda^{(i)-}_{\chi(q) \to q,1}x_{\chi(q)}^{-})\\}
    &=\sum_{p=1}^{n_i}\sum_{q \in \chi^{-1}(p)}\left(\sum_{k=1}^{|\gB_{i\to i+1}|}\mu^{(i)}_{q\to r,k}b_{i\to i+1,k}\right)(\lambda^{(i)+}_{p\to q,1}x_{p}^{+}+\lambda^{(i)-}_{p \to q,1}x_{p}^{-})\\
    &=\sum_{p=1}^{n_{i}}\sum_{k=1}^{|\gB_{i\to i+1}|}\sum_{q \in \chi^{-1}(p)}\left(\mu^{(i)}_{q\to r,k}\lambda^{(i)+}_{p\to q,1}\right)b_{i\to i+1,k}x_{p}^+ + \\
    &\sum_{p=1}^{n_{i}}\sum_{k=1}^{|\gB_{i\to i+1}|}\sum_{q \in \chi^{-1}(p)}\left(\mu^{(i)}_{q\to r,k}\lambda^{(i)-}_{p\to q,1}\right)b_{i\to i+1,k}x_{p}^-.
\end{align}
Our goal is to approximate the target model whose $r$-th block can be written as:

\begin{align*}
    f_{i}(x)_{r}=\sum_{p=1}^{n_{i}}\sum_{k=1}^{|\gB_{i\to i+1}|} \underbrace{\alpha^{(i)}_{p \to r,k}b_{i\to i+1,k}x_p^+}_{\text{term 1}}+ \sum_{p=1}^{n_{i}}\sum_{k=1}^{|\gB_{i\to i+1}|}\underbrace{\alpha^{(i)}_{p \to r,k}b_{i\to i+1,k}x_p^-}_{\text{term 2}}.
\end{align*}

To do so, we only have to approximate
$\alpha^{(i)}_{p \to r,k}$ in term $1$, for all $p,r,k$, using a subset sum of $\sum_{q \in \chi^{-1}(p)}\left(\mu^{(i)}_{q\to r,k}\lambda^{(i)+}_{p\to q,1}\right)$ and $\alpha^{(i)}_{p \to r,k}$ in term $2$, by a subset sum of $\sum_{q \in \chi^{-1}(p)}\left(\mu^{(i)}_{q\to r,k}\lambda^{(i)-}_{p\to q,1}\right)$. This can be achieved by judiciously choosing pruning masks that selectively include $\mu_{q \to r,k}^{(i)}$ which is a by-product of solving independent $\textsc{Subset-Sum}$ problems.

The key insight powering our analysis is to notice that the variables $\mu_{p\to r,k}^{(i)}$ that appear in each approximation problems are different if $(p,r,k) \neq (p',r',k')$.  Moreover, the two different problems for fixed indices $(p,r,k)$ can be seen using different variables since following whether it is positive or negative, $\lambda_{p \to q,1}^{(i)}$  will necessarily be $0$ in the first or the second term equation. Therefore, either in the first or the second equation, $\mu_{q \to r,k}^{(i)}$ can be seen as being not a variable of the \textsc{subset-sum} problem. We are then at liberty to decide whether to prune the variable or not in the equation where it appears, because the pruning of the variable will not affect the result of the other \textsc{subset-sum} problem. Following this approach,  we can then find a mask on the variables implied in subsequent problems, solve the problems independently and finally take the concatenation of all the masks in the second layer which will simultaneously solve all the problems.

We now quantify this approach by showing that with high probability, the $2n_{i}n_{i+1}|\gB_{i\to i+1}|$ subset sum problems (with independent variables) written below can all be solved by applying a pruning mask on the second layer. The pruned mask applied on the second layer is denoted $\rmS^{2i+1}_{q\to r,k} \in \{0,1\}^{\tilde{n}_i \times n_{i+1} \times |\gB_{i \to i+1}|}$. The subset sum problems are written below:

\begin{align*}
    |err^{(i)}_{p \to r,k,+}| & \coloneqq \left |\sum_{q\in \chi^{-1}(p)}(\rmS^{2i+1}_{q\to r,k} \circ \mu^{(i)}_{q\to r,k})\lambda^{(i)+}_{p\to q,k} - \alpha^{(i)}_{p\to r,k}\right| \     \forall (p,r, k) \in [n_{i}]\times [n_{i+1}] \times [|\gB_{i\to i+1}|] \\
    & \leq \frac{\epsilon}{2n_{i}n_{i+1}\max(|\gB_{i\to i+1}|,\vertiii{\gB_{i\to i+1}})}
\end{align*}
and
\begin{align*}
     |err^{(i)}_{p\to r,k,-}| & \coloneqq \left|\sum_{q\in \chi^{-1}(p)}(\rmS^{2i+1}_{q\to r,k}\circ \mu^{(i)}_{q\to r,k})\lambda^{(i)-}_{p\to q,k} - \alpha^{(i)}_{p\to r,k}\right| \     \forall (p,r, k) \in [n_{i}]\times [n_{i+1}] \times [|\gB_{i\to i+1}|] \\
     & \leq \frac{\epsilon}{2n_{i}n_{i+1}\max(|\gB_{i\to i+1}|,\vertiii{\gB_{i\to i+1}})}
 \end{align*}

We will now use the \textsc{subset-sum} Lemma \ref{SS} which explains the overparametrization that one needs to solve the \textsc{subset-sum} problems.
Since $\mu^{(i)}_{q\to r,k}$ and $\lambda^{(i)}_{p \to q,k}$ are i.i.d following $\gU([-1,1])$, $\lambda_{p\to q,1}^{(i), +}$ follows $\frac{1}{2}\delta_0 + \frac{1}{2}U$ with the notations of lemma \ref{SS}. We deduce that the $\mu^{(i)}_{q\to r,k}\lambda^{(i),+}_{p \to q,1}$ are i.i.d. following the distribution $D=\frac{1}{2}\delta_0+\frac{1}{2}P$. This is the same for $\mu^{(i)}_{q\to r,k}\lambda^{(i),-}_{p \to q,1}$ which are i.i.d. following the distribution $D=\frac{1}{2}\delta_0+\frac{1}{2}P$. Here one should note that $\forall p \in [n_i], |\chi^{-1}(p)|=C_1\log(\frac{n_{i}n_{i+1}\max\left(|\gB_{i\to i+1}|, \vertiii{\gB_{i\to i+1}}\right)}{\min\left(\epsilon, \delta\right)})$. Therefore, by using\footnote{At this point one may prove the same lemma but with more general distributions on the coefficients $\lambda_{p \to q,k}^{(i)}$ and $\mu_{q \to r,k}^{(i)}$ by assuming that they only contain $\gU([-a,a])$ for some $a >0$ and by using Lemma \ref{generalization_distribution} instead of Lemma \ref{SS}} lemma \ref{SS} , $\forall (p,r,k) \in [n_{i}]\times [n_{i+1}] \times [|\gB_{i\to i+1}|]$, the two subset sum problems can be achieved by pruning the coefficients $\mu^{(i)}_{q\to r,k}$ with probability at least $1-\frac{\delta}{2n_{i}n_{i+1}\max(|\gB_{i\to i+1}|,\vertiii{\gB_{i\to i+1}})}$. Call this the event $E^{(i)}_{p\to r,k}$. By taking the intersection of the events, we get that
$E^{(i)}= \bigcap_{(p,r) \in [n_{i}]\times [n_{i+1}], k \in [|\gB_{i\to i+1}|]}E^{(i)}_{p \to r,k}$ holds with probability at least,
\begin{align*}
    p(E^{(i)}) &= 1-n_{i}n_{i+1}|\gB_{i\to i+1}|\frac{\delta}{2n_{i}n_{i+1}\max(|\gB_{i\to i+1}|,\vertiii{\gB_{i\to i+1}})} \\
    & \geq 1-\delta.
\end{align*}

In other words, with probability at least $1-\delta$, all the  \textsc{Subset-Sum} problems are solved. Finally, it remains to check that the approximation holds with this pruning mask. Let $\Omega$ be defined as: 
\begin{align*}
    \Omega=\max_{\|x\|\leq 1} \|(\rmS_{2i+1} \odot\rmW^h_{2i+1})\sigma((\rmS_{2i} \odot \rmW^h_{2i})\rvx) - f_{i}(\rvx)\|.
\end{align*}

By applying the masks we get:
\begin{align*}
    \Omega&=\max_{r \in [n_{i+1}]}\max_{\|x\|\leq 1} \left\|(\rmS_{2i+1} \odot\rmW^h_{2i+1})\sigma((\rmS^h_{2i} \odot \rmW^h_{2i})\rvx)_r - f_i(\rvx)_r\right\| \\
    &=\max_{r \in [n_{i+1}]}\max_{\|x\|\leq 1}\|\sum_{p=1}^{n_{i}}\sum_{k=1}^{|\gB_{i\to i+1}|}\left((\alpha^{(i)}_{p\to r,k}+err^{(i)}_{p\to r,k,+})b_{i\to i+1,k}(x_p^+) + (\alpha^{(i)}_{p\to r,k} +err^{(i)}_{p\to r,k,-})b_{i\to i+1,k}(x_p^-)\right)\\
    &- \sum_{p=1}^{n_{i}}\sum_{k=1}^{|\gB_{i\to i+1}|}\left(\alpha^{(i)}_{p\to r,k}b_{i\to i+1,k}(x_p^+)+\alpha^{(i)}_{p \to r,k}b_{i\to i+1,k}(x_p^-)\right)\| \\
    & \leq \max_{r \in [n_{i+1}]} \sum_{p=1}^{n_{i}}\max_{\|x\|\leq 1} \left\|\sum_{k=1}^{|\gB_{i \to i+1}|}err^{(i)}_{p\to r,k,+}b_{i\to i+1,k}(x_p^+) +\sum_{k=1}^{|\gB_{i \to i+1}|}err^{(i)}_{p\to r,k,-}b_{i\to i+1,k}(x_p^-)\right\| \\ 
    & \leq \max_{r \in [n_{i+1}]} \sum_{p=1}^{n_{i}}\left(\max_{\|x\|\leq 1} \left\|\sum_{k=1}^{|\gB_{i \to i+1}|}err^{(i)}_{p\to r,k,+}b_{i\to i+1,k}(x_p^+)\right\| +\max_{\|x\|\leq 1}\left\|\sum_{k=1}^{|\gB_{i \to i+1}|}err^{(i)}_{p\to r,k,-}b_{i\to i+1,k}(x_p^-)\right\|\right) \\
    & \leq \max_{r \in [n_{i+1}]}\sum_{p=1}^{n_{i}}\left(\vertiii{\sum_{k=1}^{|\gB_{i \to i+1}|}err^{(i)}_{p\to r,k,+}b_{i\to i+1,k}(x_p^+)} +\vertiii{\sum_{k=1}^{|\gB_{i \to i+1}|}err^{(i)}_{p\to r,k,-}b_{i\to i+1,k}(x_p^-)}\right) \\
    & \leq \sum_{p=1}^{n_{i}}\frac{2\epsilon}{2n_{i}n_{i+1}\max(|\gB_{i\to i+1}|,\vertiii{\gB_{i\to i+1}})} \times \vertiii{\gB_{i\to i+1}} \\&\leq \epsilon
\end{align*}
\end{proof}

\xhdr{Note} In the statement of the Lemma we used a specific choice of norm ($l_p$) but our proof strategy will work with every norm as soon as $\sigma$, the ReLU non-linearity, is 1-Lipschitz (which may not be the case for some esoteric norms). As a result, there is no need to restrict oneself to the $l_p$-norm, though for ease of exposition and not to confuse the reader we made this choice above. Moreover thanks to the flexibility of the \textsc{Subset-Sum} theorem, the proof can also be extended to a milder hypothesis which is on the distribution of coefficients. Specifically, it is sufficient to have that the distribution contains a uniform distribution centered at $0$ (see Lemma \ref{generalization_distribution}). The immediate consequence of this is that it is possible to accommodate other weight initialization schemes that are commonly used in practice, but again for ease of readibility we chose to use Uniform distribution.

\subsection{Approximation of an equivariant target network}
\label{app:proof_thm_one}
We now prove in this appendix Theorem \ref{thm:main_approx_network} which approximates a full target model.
We first recall the two main assumptions (very mild) that are needed for the Theorem statement:
\begin{itemize}
    \item $\sI \in \gB_{i\to i}$
    \item $\sigma$ the pointwise ReLU is used as an equivariant nonlinearity and is $1$-Lipschitz. %The $1$-Lipschitzeness is very mild since it is true for every $l_p$ norm and hence for a wide variety of the one used in practise.
\end{itemize}
\begin{mdframed}[style=MyFrame2]
\approxnetwork*
\end{mdframed}

\begin{proof}
 We first note that we use a different constant $C_2 = 2C_1$ in the theorem as compared to lemma \ref{Approximation_of_a_layer} which helps ensure that,
\begin{align*}
    C_2n_{i}\log\left(\frac{n_{i}n_{i+1}\max\left(|\gB_{i\to i+1}|, \vertiii{\gB_{i\to i+1}}\right)l}{\min\left(\epsilon, \delta\right)}\right)\geq C_1n_{i}\log\left(\frac{2n_{i}n_{i+1}\max\left(|\gB_{i\to i+1}|, \vertiii{\gB_{i\to i+1}}\right)l}{\min\left(\epsilon, \delta\right)}\right),
\end{align*}

which is true in the domain of the following variables ($n_i, n_{i+1},l\geq 1, \delta\leq, \epsilon \leq \frac{1}{2},  |\gB_{i \to i+1}[\geq 1$) as $2\log(x)\geq \log(2x)$ on $[2, +\infty]$.

We first apply lemma \ref{Approximation_of_a_layer} $l$-times for each layer of the target network with $\epsilon$ becoming $\frac{\epsilon}{2l}$and $\delta$ becoming $\frac{\delta}{l}$. With an overparametrization factor $\tilde{n}_i\geq C_1n_{i}\log(\frac{2n_{i}n_{i+1}\max\left(|\gB_{i\to i+1}|, \vertiii{\gB_{i\to i+1}}\right)l}{\min\left(\epsilon, \delta\right)})$ we get that for each layer $i$,
\begin{align}
    \label{lerror}
    \max_{\|\mathbf{x}\|\leq 1} \|(\rmS_{2i+1} \odot\rmW^h_{2i+1})\sigma((\rmS_{2i} \odot \rmW^h_{2i})\rvx) - f_i(\mathbf{x})\| \leq \frac{\epsilon}{2l}
\end{align}
holds with probability at least $1-\frac{\delta}{l}$.
By taking a union bound, we get that this holds for every layer with probability at least $1-\delta$.
Now, let $x'_{i}$ be the input to the $(2i)$-th layer of the pruned overparametrized network $h$. Furthermore, let $x_{i}$ be the input to the $i$-th layer of the target network $f$. Then we have,
\begin{itemize}
    \item $x'_0=x_0=x$
    \item $x'_{i+1}=\sigma\left((\rmS_{2i+1} \odot\rmW^h_{2i+1})\sigma\left((\rmS_{2i}\odot \rmW^h_{2i})\rvx'_i\right)\right)$ \text{ for $i\leq l-2$}
    \item $x'_{l}=(\rmS_{2l-1} \odot\rmW^h_{2l-1})\sigma\left((\rmS_{2l-2}\odot \rmW^h_{2l-2})\rvx'_{l-1}\right)$
\end{itemize}
Equation \ref{lerror} implies that,
\begin{align}
    \|(\rmS_{2i+1} \odot\rmW^h_{2i+1})\sigma((\rmS_{2i} \odot \rmW^h_{2i})\rvx'_i) - f_i(\rvx'_i)\| \leq \|\rvx'_i\|\frac{\epsilon}{2l}
\end{align}

Passing through the point-wise ReLU which is $1$-Lipschitz for all the norms that we work with we get:
\begin{align}
    \|\rvx'_{i+1}\| \leq \|\rvx_i'\| \left(1 + \frac{\epsilon}{2l}\right)
\end{align}
By leveraging a recursive argument, and using the fact that $\|\rvx'_0\|=\|\rvx\|\leq 1$ we then get that for all $ i \in \{0, \dots, l-1\}, \quad \|x'_i\| \leq (1+\frac{\epsilon}{2l})^{i}$. Then, forall $ i \leq l-2$:
\begin{align*}
    \|\rvx'_{i+1}-\rvx_{i+1}\| &=\|\sigma\left((\rmS_{2i+1} \odot\rmW^h_{2i+1})\sigma((\rmS_{2i} \odot \rmW^h_{2i})\rvx'_{i})\right) - \sigma\left(f_i(\mathbf{x_{i}})\right)\|\\
    &\leq \|\sigma\left((\rmS_{2i+1} \odot\rmW^h_{2i+1})\sigma((\rmS_{2i} \odot \rmW^h_{2i})\rvx'_{i})\right) - \sigma\left(f_i(\rvx_i')\right)\|+ \|\sigma\left(f_i(\rvx_i')\right)-\sigma\left(f_i(\mathbf{x_{i}})\right)\|\\
    &\leq \|\rvx'_i\|\frac{\epsilon}{2l} + \vertiii{f_i}\|\rvx'_{i}-\rvx_i\| \\
    &\leq \left(1+\frac{\epsilon}{2l}\right)^{i}\frac{\epsilon}{2l} + \|\rvx'_{i}-\rvx_i\|,
\end{align*}
where we used the fact that $\sigma$ is one Lipschitz. We then get that, 

\begin{align*}
    \|x'_l-x_l\| &=\|(\rmS_{2l-1} \odot\rmW^h_{2l-1})\sigma\left((\rmS_{2l}\odot \rmW^h_{2l})\rvx'_{l-1}\right))-f_{l-1}(\rvx_{l-1})\|\\
    &\leq \|(\rmS_{2l-1} \odot\rmW^h_{2l-1})\sigma((\rmS_{2l-2} \odot \rmW^h_{2l-2})\rvx'_{l-1}) - f_{l-1}(\rvx_{l-1}')\|+ \|f_{l-1}(\rvx_{l-1}')-f_{l-1}(\rvx_{l-1})\| \\
    &\leq \|\rvx'_{l-1}\|\frac{\epsilon}{2l} + \vertiii{f_{l-1}}\|\rvx'_{l-1}-\rvx_{l-1}\|\\
    &\leq (1+\frac{\epsilon}{2l})^{l-1}\frac{\epsilon}{2l}+ \|\rvx'_{l-1}-\rvx_{l-1}\|\\
    &\leq \sum_{i=0}^{l-1}\left(1+\frac{\epsilon}{2l}\right)^{i}\frac{\epsilon}{2l}\\
    &\leq (1+\frac{\epsilon}{2l})^l-1\\
    & \leq e^{\frac{\epsilon}{2}}-1\\
    &\leq \epsilon \quad \text{because $\epsilon \leq \frac{1}{2}$}
\end{align*}

\end{proof}

\subsection{Lower bound}
\label{app:lowerbound}
We now prove Theorem \ref{th:lowerbound}.
\begin{mdframed}[style=MyFrame2]
\lowerbound*
\end{mdframed}
Let us first recall the main assumptions of our setting:

\begin{itemize}[noitemsep,topsep=0pt,parsep=0pt,partopsep=0pt,label={\large\textbullet},leftmargin=*]
    \item For all $f_i \in \Span(\kappa_{n_i \to n_{i+1}}\otimes \gB_{i \to i+1})$, where $f_i=\sum_{p \in [n_i]}\sum_{q \in [n_{i+1}]}\sum_{k}\alpha^{(i)}_{p \to q,k}b_{i\to i+1,k}$, we have:
    $\vertiii{f_i} \leq 1 \implies \|\alpha^{(i)}\|_{\infty}\leq 1$. This assumption is extremely mild as it can always trivially be satisfied by rescaling the basis elements.
    \item $\exists M_1 \in \sR^+$ such that, uniformly over $i$, for every possible building block in our equivariant feature spaces $\sF_i$ and $\sF_{i+1}$ we have, $|\gB_{i \to i}|\leq M_1|\gB_{i\to i+1}|$. This assumption is used to mainly guard against a non-trivial scenario where the first layer would be able to carry ``a lot of superfluous parameters". This assumption finds its solitary use in achieving the lower bound on the overparametrization factor in the theorem \ref{thm:main_approx_network}, i.e. to prove $\tilde{n}_i \geq \Omega(n_i \log\left(\frac{1}{\epsilon}\right))$ and is not used for the lower bound on $\Theta$. This assumption is very mild because in most of the usual cases (MLPs, CNNs, $\text{E}(2)$-steerable CNNs) the possible building blocks of each layer $\sF_i$ are finite (respectively $\sR$, $\sR^{d^2}$ and $\sR^{d^2}$ or $\sR^{d^2\times |G|})$. Being finite  automatically implies the existence of such a constant $M_1$ as we can simply take the maximum over the possible values of $\frac{|\gB_{i \to i+1}|}{|\gB_{i \to i}|}$.
    \item Finally we assume the existence of a constant $M_2 \in \sR^+$ such that $n_i \leq M_2n_{i+1}$. This mild assumption---like the previous one---is used to ensure that $\tilde{n}_i \geq \Omega(n_i \log\left(\frac{1}{\epsilon}\right))$.
\end{itemize}

Our proof relies on a counting based argument that compares the number of pruning masks to the cardinal of a $2\epsilon$-separated net $\gP$ in the set of target networks with respect to the operator norm. A similar argument was used by~\citet{pensia2020optimal} in the context of dense nets. 
We recall here the definition of a $2\epsilon$-separated net:
\begin{definition}
Let $F$ be a normed vector space. A $2\epsilon$-separated net $\gP$, in $F$ is a subset $\gP \subset F$ such that: 
\begin{align*}
    \forall x_1, x_2 \in \gP, \quad x_1 \neq x_2 \implies \|x_1 - x_2\| \geq 2\epsilon
\end{align*}
\end{definition}

In Lemma~\ref{Approximation_of_a_layer}, we considered only a set of target network $\gF_i \subset \Span(\kappa_{n_i \to n_{i+1}}\otimes \gB_{i \to i+1})$ where each function has $\vertiii{f_i} \leq 1$, and $\|\alpha^{(i)}\|_{\infty} \leq 1$. Mixing this with the first assumption written above, it is therefore the set of maps $f_i$ such that $\|\alpha^{(i)}\|_{\infty}\leq 1$. Now consider the isomorphism $\gI_i:\Span(\kappa_{n_i \to n_{i+1}}\otimes \gB_{i \to i+1}) \to \sR^{n_in_{i+1}|\gB_{i \to i+1}|}$, which identify a function with its coefficients in the equivariant basis.
\begin{align}
\gI_i \coloneqq f_i \mapsto \alpha^{(i)}
% \begin{matrix}
%     \Span(\kappa_{n_i \to n_{i+1}}\otimes \gB_{i \to i+1}) \to \sR^{n_in_{i+1}|\gB_{i \to i+1}|}\\
%     \gI \coloneqq f_i \mapsto \alpha^{(i)}
% \end{matrix}
\end{align}
This isomorphism shows that $\gF_i$ can be seen as the norm ball of $\sR^{n_in_{i+1}|\gB_{i \to i+1}|}$ with respect to the norm induced on $\sR^{n_in_{i+1}|\gB_{i \to i+1}|}$ by the isomorphism. Moreover $\gP$ is a $2\epsilon$-separated net on $\gF_i$ if and only if its image is a $2\epsilon$-separated net on $\sR^{n_in_{i+1}|\gB_{i \to i+1}|}$ with respect to the induced norm.

\xhdr{Lower bound on $|\gP|$}
We just need to use Lemma 4.2.8 and extend Proposition 4.2.12 from~\citet{vershynin2018high} to non-Euclidean balls. The general idea is as follows: Denote by $\gB(x, R)$ the ball centered at $x$ of radius $R$ in $\sR^{n_in_{i+1}|\gB_{i \to i+1}}$ and by $\mu$ the Lebesgue measure. Let us construct a $2\epsilon$-separated net as follow: we take for first point the origin $0$ of the vector space. At the step $n$, to construct the $n+1$-th point of $\gP$ we proceed as follow: if $\gB(0,1) \subset \bigcup_{x \in \gP}\gB(x, 2\epsilon)$ we stop the processus and don't take any $n+1$-th point. Else, we take a point in $\gB(0,1) \diagdown \bigcup_{x \in \gP}\gB(x, 2\epsilon)$. We know that this point is at a distance of at least $2\epsilon$ of the other points of $\gP$. Moreover it is in the unit ball. At the end of the process (which must end since the unit ball is compact), we finally get that for the $2\epsilon$-separated net $\gP$,  $\gB(0,1) \subset \bigcup_{x \in \gP}\gB(x, 2\epsilon)$. Therefore, $\mu(\gB(0,1)) \leq \mu(\bigcup_{x \in \gP}\gB(x, 2\epsilon)) \leq |\gP|\mu(\gB(0, 2\epsilon))$. Finally, $|\gP|\geq \frac{\mu(\gB(0,1))}{\mu(\gB(0, 2\epsilon))}=\left(\frac{1}{2\epsilon}\right)^{n_in_{i+1}|\gB_{i \to i+1}|}$.
In the last step, we use the fact that the Lebesgue measure of a ball of radius $R$ in a vector space of dimension $n$ is $R^nV_n$ where $V_n$ is the Lebesgue measure of the unit ball.
% it is straightforward to show that there exists a $2\epsilon$-separated net which has at least $\frac{\mu(\{v \in \sR^{n_in_{i+1}|\gB_{i \to i+1}|}, \|v\|\leq 1\})}{\mu(\{v \in \sR^{n_in_{i+1}|\gB_{i \to i+1}|}, \|v\|\leq 2\epsilon\})}$ elements. Here $\mu$ denotes the Lebesgue measure. Indeed, one can choose a $2\epsilon$-separated net such that the union of the balls of radius $\epsilon$ centered in the point of $\gP$ cover the unit ball.\footnote{For instance, the maximal $2\epsilon$-separated net of the unit ball is also an $2\epsilon$-net of the unit ball, see~\citep{vershynin2010introduction}.}
This allows to  choose $|\gP|\geq \left(\frac{1}{2\epsilon}\right)^{n_in_{i+1}|\gB_{i \to i+1}|}$. 

\xhdr{Lower bound induced on $\Theta$}
As the network that we seek to prune has $\Theta$ parameters, the number of binary pruning masks that can be constructed is $2^{\Theta}$. Moreover, due to the triangular inequality, each pruned network can approximate at most one element of $\gP$. Indeed, if $f_i^{1}\neq f_i^{2} \in \gP$ are approximated with the same pruning mask,
\begin{align}
    \|f_i^{2}-f_i^{1}\| \leq \|f_i^{2}-(S\odot \hat h_i)\|+ \|(S \odot \hat h_i)-f_i^{1}\|\leq 2\epsilon,
\end{align}
which contradicts the fact that $\gP$ is a $2\epsilon$-separated net.
This directly implies that the number of pruning masks must be bigger than the cardinal of $\gP$.

\begin{align}
    2^{\Theta}\geq \left(\frac{1}{2\epsilon}\right)^{n_in_{i+1}|\gB_{i \to i+1}|}
\end{align}

and by taking the $\log$,

\begin{align}
\label{eq:lower_params}
    \Theta \geq \frac{n_in_{i+1}|\gB_{i \to i+1}|}{\log(2)}\log\left(\frac{1}{2\epsilon}\right)
\end{align}

which shows that $\Theta$ must be at least $\Omega(n_in_{i+1}|\gB_{i \to i+1}|\log\left(\frac{1}{\epsilon}\right))$

\xhdr{Lower bound on $\tilde{n}_i$}
We now seek to provide a lower bound on $\tilde{n}_i$ such that Theorem \ref{thm:main_approx_network} holds. Since, our main claim requires that we approximate every target network with probability at least $1-\delta >0$, the set of parameters (drawn from any distribution) that can achieve this is non zero. %We deduce the existence of a network with width $\tilde{n}_i$ blocks such that it can approximate every target network by pruning. We will then use the first part of the theorem on this network to lower bound $\tilde{n}_i$
What remains is to count the number of parameters contained within the overparametrized $G$-equivariant network in $\gH_i$ (see Lemma  \ref{Approximation_of_a_layer}) as a function of the overparametrization factor $\tilde{n}_i$. This allows us to lower bound $\tilde{n}_i$ via the lower bound on the number of parameters established above. Now any overparametrized $G$-equivariant network we construct has the following number of parameters:
\begin{itemize}
    \item Number of parameters of the first layer: $n_i\tilde{n}_i|\gB_{i \to i}|$
    \item Number of parameters of the second layer: $\tilde{n}_in_{i+1}|\gB_{i \to i+1}|$
\end{itemize}

Therefore the overparametrized network $h_i$ has $\Theta=\tilde{n}_i(n_i|\gB_{i \to i}| +n_{i+1}|\gB_{i \to i+1}|)$ parameters. Using the second and third assumptions, we get that:
\begin{align}
    \Theta \leq \tilde{n}_i(M_1M_2n_{i+1}|\gB_{i\to i+1}|+n_{i+1}|\gB_{i \to i+1}|)\leq \tilde{n}_i(M_1M_2+1)n_{i+1}|\gB_{i \to i+1}|.
\end{align}

Moreover by~\eqref{eq:lower_params}, we know that:

\begin{align*}
    \Theta \geq \Omega\left(n_in_{i+1}|\gB_{i\to i+1}|\log\left(\frac{1}{\epsilon}\right)\right).
\end{align*}

It therefore implies that:
\begin{align}
    \tilde{n}_i \geq \Omega\left(n_i\log\left(\frac{1}{\epsilon}\right)\right)
\end{align}

\xhdr{Discussion} Using the result in Theorem \ref{th:lowerbound} we can now understand that Theorem \ref{thm:main_approx_network} informs us that our proposed overparametrization strategy is optimal with respect to the tolerance $\epsilon$ and almost optimal with respect to $n_in_{i+1}|\gB_{i \to i+1}|$. In Theorem \ref{thm:main_approx_network} we observe an additional factor of $\log(n_in_{i+1}\max(|\gB_{i \to i+1}|, \vertiii{\gB_{i\to i+1})})$ which appears in $\tilde{n}_i$. We can reconcile this term which appears in the proof due to both our choice of with which metric do we want to approximate the target network and to the probabilistic setting of the SLTH. \cut{the number of \textsc{Subset-Sum} problems ----i.e. the complexity of the approximation.} \cut{Indeed, a rescaling or a better choice of norms might eliminate the first term $n_in_{i+1}\vertiii{\gB_{i \to i+1}}$. }

Indeed, first we note that we chose to approximate each target layer by $\epsilon$ with respect to the operator norm associated with the norms on the input and output space. But such a choice is arbitrary, and if we had chosen another metric, \cut{such that being $\epsilon$ close of $\alpha_{p\to q,k}^{(i)}$ for every linear region on $x^+, x^-$} such as approximating each weight of the target network in the diamond shape structure by $\epsilon$, then the term $n_in_{i+1}\vertiii{\gB_{i \to i+1}}$ might have been eliminated. 

The term $n_in_{i+1}|\gB_{i \to i+1}|$ arises from the fact that the parameters of the overparametrized network are drawn from a random process. Specifically, a bigger overparametrization is needed because of the scenario when not all the \textsc{Subset-Sum} problems have solutions, which has a probability of occurring that grows with $n_in_{i+1}|\gB_{i\to i+1}|$---i.e. the complexity of the approximation. We could replace the probabilistic setting by instead taking an overparametrized network deterministically initilized by a smart initialization such that with probability one all possible subnetworks can be obtained by pruning the overparametrized one.  In this case, the overparametrization on the width would no longer have the term $n_in_{i+1}|\gB_{i \to i+1}|$ in the $\log$. Such an initialization can be taken for example by decomposing the overparametrized network in the different blocks of the diamond shape and taking the weights in each block to be $\pm 1, \pm \frac{1}{2}, \pm \frac{1}{4}, \pm (\frac{1}{2})^{\frac{\log(\epsilon) }{ \log(2)}} = \frac{1}{\epsilon}$ (the weights that are not part of a diamond shape can be initialized freely). Each weight of the target network can then be approximated by pruning the diamond shape with a mask which is the binary writing of the target weight. This is possible for every weight of the target network and for all target network at once with probability one (with a different mask for each target network).
We note here the similarity of this construction with the one used in \citet{sreenivasan2022finding}, albeit under a different setting than the one we considered here. 

In conclusion, we give some hints to annihilate the term $n_in_{i+1}\max(|\gB_{i \to i+1}, \vertiii{\gB_{i \to i+1}})$: first choosing another metric for approximating a layer and secondly going to a non-probabilistic setting where the overparametrized network is smartly initialized.
\cut{This means that apart from these two side-effects, our Lemma \ref{Approximation_of_a_layer} is optimal with respect to $n_in_{i+1}|\gB_{i\to i+1}|$ and the specified tolerance $\epsilon$.}

\section{Proof of STL on MLP using Theorem \ref{thm:main_approx_network}}
\label{app:pensia}
This corollary recovers the main result of~\citet{pensia2020optimal}.

In this case, $G=\{e\}$ and the representation is trivial. The building block of a layer is $\sF_i=\sR$ and each layer is composed of a stack of $n_i$, i.e. $\sR^{n_i}=\sF_i^{n_i}$. The norm that we will use on $\sR$ is of course the absolute value $|\cdot|$. Therefore, as explained above, the norm that we consider on $\sF_i^{n_i}=\sR^{n_i}$ is $\|\cdot\|_{\infty}$. The pointwise ReLU is trivially equivariant, since the $G$ is trivial. It is moreover $1$-Lipschitz.

All maps are equivariant, since the group $G$ is trivial. An equivariant basis of the maps $\sF_i \to \sF_{i+1}$ and of the maps $\sF_i \to \sF_i$ is therefore a basis of the maps $\sR \to \sR$. It is of dimension $1$ and of course taken to be the identity. We therefore obtain that the identity is in $\gB_{i \to i}$. 
One  has $|\gB_{i \to i}|=1$ and $\vertiii{\gB_{i \to i+1}}=\max_{|\alpha|\leq 1}\vertiii{\alpha \sI}=1$.

All the conditions are therefore validated and we are free to apply Theorem \ref{thm:main_approx_network} in this setting, with $\max(|\gB_{i \to i}|, \vertiii{\gB_{i \to i+1}})=1$ which leads to the following corollary:

\begin{mdframed}[style=MyFrame2]
\begin{restatable}{corollary}{SLTdense}
\label{cor:SLTdense}
Let $h \in \gH$ be a random MLP of depth $2l$, i.e., $h(\mathbf{x})=\mathbf{W}^h_{2l-1}\sigma\left( \dots \sigma(\mathbf{W}^h_{0}x)\right)$ where $\mathbf{W}^h_{2i} \in \sR^{n_{i}\times \tilde{n}_{i}}$, $\mathbf{W}^h_{2i+1}\in \sR^{\tilde n_{i} \times n_{i+1}}$ are dense linear maps with weights drawn from $\gU([-1,1])$
If $\tilde{n_{i}}=C_2n_{i} \log \left(\frac{n_{i}n_{i+1}l}{\min(\epsilon, \delta)}\right)$, then with probability at least $1-\delta$ we have that for all $f \in \gF$ a target MLP with layers $\mathbf{W}_{i}^{f}\in [-1,1]^{n_i \times n_{i+1}}$ and $\vertiii{f_i}\leq 1$ there exists a collection of pruning masks $\mathbf{S}_{2l-1}, \dots, \mathbf{S}_0$ such that,
\begin{equation}
    \max_{\mathbf{x}\in \sR^{n_0},\,\|\mathbf{x}\|\leq 1} \|(\mathbf{S}_{2l-1} \odot\mathbf{W}^h_{2l-1})\sigma \left(\dots \sigma\left(\left(\mathbf{S}_0 \odot \mathbf{W}^h_{0} \right)x\right)\right) - f(x)\| \leq \epsilon
\end{equation}
%where $\sigma$ is the ReLU nonlinearity.
\end{restatable}
\end{mdframed}

\section{Proof of SLT on CNN using Theorem \ref{thm:main_approx_network}}
\label{app:proof_CNN}
We now prove Theorem \ref{thm:main_approx_network} application to the case regular translation equivariant CNNs. 
We highlight here that this is a strict generalization of the result obtained by \citet{da2022proving} as we do not assume strictly positive inputs (recently extended in parallel in \citet{burkholz2022convolutional}).

\begin{mdframed}[style=MyFrame2]
\begin{restatable}{corollary}{theorem_CNN}
\label{theorem_CNN}
Let $h \in \gH$ be a random CNN of depth $2l$, i.e., $h(\mathbf{x})=\mathbf{K}^h_{2l-1}*\sigma\left( \dots \sigma(\mathbf{K}^h_{0}*x)\right)$ where $\mathbf{K}^h_{2i} \in \sR^{d^{2}\times  n_{i}\times \tilde{n}_{i}}$, $\mathbf{K}^h_{2i+1}\in \sR^{d^{2}\times \tilde n_{i} \times n_{i+1}}$ are convolutional kernels with weights in $\gU([-1,1])$
If $\tilde{n_{i}}=C_2n_{i} \log \left(\frac{d^{2}n_{i}n_{i+1}l}{\min(\epsilon, \delta)}\right)$, then with probability at least $1-\delta$ we have that for all $f \in \gF$ a target CNN with kernels $\mathbf{K}_{i}^{f}\in [-1,1]^{d^{2}\times n_i \times n_{i+1}}$ and $\vertiii{f_i}\leq 1$ there exists a collection of pruning masks $\mathbf{S}_{2l-1}, \dots, \mathbf{S}_0$ such that,
\begin{equation}
    \max_{\mathbf{x}\in \sR^{d^2 \times n_0},\,\|\mathbf{x}\|\leq 1} \|(\mathbf{S}_{2l-1} \odot\mathbf{K}^h_{2l-1}) *\sigma \left(\dots \sigma\left(\left(\mathbf{S}_0 \odot \mathbf{K}^h_{0} \right)*x\right)\right) - f(x)\| \leq \epsilon
\end{equation}
%where $\sigma$ is the ReLU nonlinearity.
\end{restatable}
\vspace{-4mm}
\end{mdframed}

We now prove Corollary \ref{theorem_CNN}. In our case, the building blocks of every layer are $\sF_i=\sR^{d^2}$  where $d^{2}$ is the size of an image. Therefore, $\gB_{i \to i+1}$ is the basis of translation equivariant maps: $\sR^{d^{2}} \to \sR^{d^{2}}$. When working with CNNs, the basis that is used in practice is the convolution with a kernel $\rmK^{p,q} \in \sR^{d^{2}}$ where $\rmK^{p,q}$ has only a $1$ at the index $(p,q)$ and is filled everywhere else with zeros on the grid $d\times d$, where $(p,q) \in [d]^2$.
It is therefore easy to see that :
\begin{align*}
    |\gB_{i \to i+1}|=d^{2}.
\end{align*}
Let us choose $\|\cdot\|_{\infty}$ as a norm on $\sR^{d^{2}}$.
Applying the proposition 1 from \citet{da2022proving}, we get:
\begin{align*}
    \forall \rmK\in \sR^{d\times d}, \quad \forall \rmX \in \sR^{d\times d}, \quad \|\rmK*\rmX\|_{\infty}\leq \|\rmK\|_1\|\rmX\|_{\infty}.
\end{align*}

By using this basis we then get that,

\begin{align*}
    \vertiii{\gB_{i \to i+1}}%&=\max_{\rmK \in [-1,1]^{d^2}}\vertiii{K*\cdot}\\
    &=\max_{\rmK \in [-1,1]^{d^2}}\max_{\rmX \in [-1,1]^{d^2}}\|\rmK*\rmX\|_{\infty}\\
    &\leq \max_{\rmK \in [-1,1]^{d^2}}\max_{\rmX \in [-1,1]^{d^2}}\|\rmK\|_1\|\rmX\|_{\infty}\\
    &\leq d^2.
\end{align*}
We then get that:
\begin{align*}
    \max(|\gB_{i \to i+1}|, \vertiii{\gB_{i \to i+1}})=d^{2}.
\end{align*}

It is trivial to notice that the pointwise-ReLU used is equivariant and $1$-Lipschitz. Moreover, the identity is clearly in $\gB_{i \to i}$ by taking the kernel with only a $1$ at the origin. Therefore all the conditions are met and we can apply theorem \ref{thm:main_approx_network} which states that the overparametrization needed is:
\begin{align*}
    \tilde{n}_i&=C_2n_{i}\log\left(\frac{n_{i}n_{i+1}\max\left(|\gB_{i\to i+1}|, \vertiii{\gB_{i\to i+1}}\right)l}{\min\left(\epsilon, \delta\right)}\right)\\
    &=C_2n_{i}\log\left(\frac{d^{2}n_{i}n_{i+1}l}{\min\left(\epsilon, \delta\right)}\right).
\end{align*}

\section{Additional material on $\text{E}(2)$-Steerable Networks}
\label{app:steerable}
% Let's take $G$ a subgroup of $O(2)$. We study here equivariant networks for the induced representation $\left[\Ind_{G}^{(\mathbb{R}^{2}, +) \rtimes G}\rho\right]$ as explained in \cite{weiler2019general}, where $\rho$ is a representation of $G$ on $\mathbb{R}^{c}$.

%\input{old_background}

\cut{
\xhdr{$\text{E}(2)$-equivariant nets}
The Euclidean group $\text{E}(2)$ is the group of isometries of the plane $\mathbb{R}^2$ and is defined as the semi-direct product between the translation and orthogonal groups of two dimensions $ (\mathbb{R}^2, + ) \rtimes \text{O}(2)$. The most general method to build equivariant networks for $\text{E}(2)$ is in the framework of 
steerable $G$-CNN's which represent the input $x$, as a signal over a homogenous space (a manifold where $G$ acts transitively) and design convolutional filters that are \textit{steerable} with respect to the action of the group \cite{cohen2016steerable,cohen2018intertwiners,weiler2018learning}. Concretely, steerable feature fields associate a $c$-dimensional feature vector to each point in a base space $f: \mathbb{R}^2 \to \mathbb{R}^c$ which identifies the appropriate transformation law under the action of a group.
For example, a gray scale image $x$ can be thought of as a scalar field $s: \mathbb{R}^2 \to \mathbb{R}$ over the 2D-plane where at each pixel co-ordinate we associate an intensity value. Now an action of $\text{E}(2)$ by some $(tg) \in  (\mathbb{R}^2, + ) \rtimes \text{O}(2)$ on the scalar field $s$ results in moving a pixel to its new position: $s(x) \mapsto s((tg)^{-1}x) = s(g^{-1}(x-t))$. 

% The theory of steerable $G$-CNNs is applicable to not just scalar fields but also vector and tensor fields. 
To understand how these fields transform $f: \mathbb{R}^2 \to \mathbb{R}^c$, we must specify a group representation $\rho: G \to GL(\mathbb{R}^c)$, which  itself is a group and satisfies $\rho(g_1g_2) = \rho(g_1)\rho(g_2)$ with the group operation being ordinary matrix multiplication, e.g.  2D-rotations  are $2 \times 2$ rotation matrices. Thus, feature fields in steerable $\text{E}(2)$-CNNs transform according to their \emph{induced representation} $\left[\Ind^{(\mathbb{R}^{2}\rtimes G)}_{G}\rho\right]$,
\begin{equation}
    f(x) \rightarrow \left(\left [\Ind^{(\mathbb{R}^{2}\rtimes G)}_{G}\rho \right ](tg) \cdot f \right)(x) := \rho(g) \cdot f(g^{-1}(x-t)).
\end{equation}
Equipped with this feature spaces in steerable $G$-CNNs are represented as stacks of feature fields $f = \bigoplus_i f_i$ with each field $f_i$ being associated its own representation type $\rho_i$ which results in a block diagonal direct sum representation over the entire feature space. Clearly, an input RGB image--a scalar field---transforms to the \emph{trivial representation} $\rho(g) = 1\,,\; \forall g \in G$, but intermediate layers may transform according to other representation types such as regular or irreducible. As proven in \citet{cohen2019general}, any equivariant linear map between steerable feature spaces transforming under $\rho_{in}$ and $\rho_{out}$ must be a group convolution with $G$-steerable kernels $\pi: \mathbb{R}^{c_{out}} \times \mathbb{R}^{c_{in}}$ which satisfy the following equivariance constraint: $\pi(gx) = \rho_{out}(g)\pi(x)\rho_{in}(g^{-1}) \ \forall g \in G, x \in \mathbb{R}^2$. %In section \ref{}, we will seek to approximate a group convolution using two overparametrized group convolutions. 

}

\subsection{General equivariant layers in the case of feature fields defined on $\sR^{2}$}

In full generality, the theory of $\text{E}(2)$-steerable CNN has been developed in the setting of continuous and infinite steerable fields defined on $\sR^2$. The input and output of a layer are then respectively functions in $(\sR^2 \to \sR^{c_{\text{in}}})$ and in $(\sR^2 \to \sR^{c_{\text{out}}})$. The reader will immediately note that it does not correspond to the practical case of $\text{E}(2)$-steerable CNN since these type of inputs are not infinite dimensional.
The condition for a layer to be equivariant between these two feature fields is to be written as a continuous convolution with kernels satisfying the condition (called equivariant kernels):
\begin{align}
\label{equiv_constraint}
    \pi(g\cdot x) = \rho_{\text{out}}(g)\pi(x)\rho_{\text{in}}^{-1}(g), \quad \forall g \in G, x \in \mathcal{X}
\end{align}

There are different methods to compute the possible kernels that satisfy this condition, that will lead to different basis. For example, in \citet{weiler2019general}, the authors use the polar coordinates to solve this condition. They have a free parameter which is the frequency and by varying this parameter they can compute a basis of the equivariant kernels.

Our method to construct a basis of the equivariant kernels is different: we quotient the plane $\sR^2$ by the equivalence relation induced by the orbits under the group $G$. For each point in the continuous quotient space $\gA_{\gR}$, we compute a basis of the equivariant kernels by putting an element of the canonical basis at this point and summing over the group $G$ the action of an element of $G$ on this element.
More precisely, we impose having some matrix $\tK_{0,x}^{p,q}$ at the point $x \neq 0$ and to obtain the full equivariant kernel, we just apply the following formula:

\begin{equation}
    \forall y \in \sR^2 \quad b(y):= \tK_{G,x}^{p,q}(y) = \sum_{g \in G} \rho_{i+1}(g) \tK_{0,x}^{p,q}(g^{-1}y)\rho_{i}(g^{-1}).
\end{equation}

This formula is well defined because in the case of subgroups of $\text{O}(2)$, $\forall x, y \in \sR^2 \diagdown \{0\},$ the set $\{g \in G, g \cdot x =y \}$ is finite meaning that the above sum is finite for every $y\in \sR^2$. It remains to check that the kernel
$\tK_{G,x}^{p,q}$ respects the above condition on equivariant kernels.  Indeed, one has that:

\begin{align*}
\forall y \in \sR^{2}, \forall h \in G, \quad \tK_{G,x}^{p,q}(h\cdot y)&=\sum_{g \in G}\rho_{i+1}(g)\tK_{0,x}^{p,q}(g^{-1}\cdot(h \cdot y))\rho_i(g^{-1}) \\
&=\sum_{g \in G}\rho_{i+1}(h)\rho_{i+1}(h^{-1}g)\tK_{0,x}^{p,q}((h^{-1}g)^{-1} \cdot y))\rho_i(g^{-1}h)\rho_i(h)^{-1}\\
&= \rho_{i+1}(h) \left(\sum_{g \in G}\rho_{i+1}(h^{-1}g)\tK_{0,x}^{p,q}((h^{-1}g)^{-1} \cdot y))\rho_i(g^{-1}h)\right)\rho_i(h)^{-1} \\
&=\rho_{i+1}(h)\tK_{G,x}^{p,q}(y)\rho_i(h)^{-1}
\end{align*}

where we used that $g \mapsto h^{-1}g$ from $G$ to $G$ is a bijection.
One should note that for some groups $G$ and some $x\neq 0$ it may be possible that $\exists g \in G, \quad g\cdot x =x$. The set of all elements that keep the point unchanged is known as the stabilizer subgroup. For example, for $G$ a dihedral group and a point $x$ on the symmetry axis, one has that $x$ remains untouched by the symmetry with respect to this axis. This is however not a problem, as the set of such $g$ is finite, and therefore the above formula is still valid, even at the point $x$. One will note however that $\tK_{G,x}^{p,q}(x) \neq \tK_{0,x}^{p,q}(x)$. This means that we may lose the fact that the set of equivariant kernels $\{\tK_{G,x}^{p,q}, (p,q) \in [c_{\text{in}}]\times [c_{\text{out}}]\}$ is composed of independent vectors and therefore forms a basis. We will still have that it spans the space of equivariant kernels but not that it will form a basis.

\cut{More precisely, we look at the equivariant maps from ($f_{in}, \rho_{\text{in}})$ to ($f_{\text{out}}, \rho_{\text{out}})$ where $f_{\text{in}}:\mathbb{R}^{2} \rightarrow \mathbb{R}^{c_{\text{in}}}$ and $f_{\text{out}}:\mathbb{R}^{2} \rightarrow \mathbb{R}^{c_{\text{out}}}$.
It was shown in \citet{weiler2019general} that the equivariant maps are written as convolution with a kernel $k:\mathbb{R}^{2} \rightarrow \mathbb{R}^{c_{\text{out}}  \times c_{\text{in}}}$ where 
\begin{equation}
k(g \cdot x)=\rho_{out}(g)k(x)\rho_{in}(g^{-1}) \,,\quad \forall x \in \sR^2\,,\quad \forall g \in G \,.
\label{kernel}
\end{equation}
To look at the possible equivariant maps, one has just to look at the possible $G$-steerable Kernels.
One can first decompose \eqref{kernel} as independent conditions on $\{\mathcal{O}(x),$ where $x \in \mathbb{R}^{2}/\mathcal{R}\}$ where $\mathcal{R}$ is the equivalence relation which has $\mathcal{O}(x)$ as equivalence classes.

$\forall x \in \mathbb{R}^{2}/\mathcal{R}$ one finds a basis $\mathcal{B}_{x}$ of the allowed $k(x)$ that can solve \eqref{kernel}. Once one has chosen a decomposition of $k(x)$ in this basis, the values of $k(g \cdot x)$ are all determined.

To conclude, the space of equivariant maps is spanned by convolutions with kernels 
\begin{equation}
    k_x \quad \text{s.t.} \quad k_x(y) = 0\,, \; \forall y \notin \gO_x\,,
    \quad k_x(x) \in \gB_x \quad \text{and} \quad k_x(g \cdot x)=\rho_{out}(g)k_x(x)\rho_{in}(g^{-1}) \,.
\end{equation}
% $\exists x \in \mathbb{R}^{2}/\mathcal{R}$ $k(y)=0$ if $y \notin \mathcal{O}_{x},

 \begin{remark}
 One recovers the fact that for usual convolutions, all possible kernels are allowed because: $g \cdot x=x \ \forall x \in \sR^2,\, g\in G$ and $\rho_{out}, \rho_{in}$ are trivial. Especially, if one prunes any coefficient of the kernel, it will remain a convolution.
However for $G\neq \{e\}$ and non trivial representations this will no longer be true.
 \end{remark}
 
One can ask several questions:
-can one prune equivariant networks with $G\neq \{e\}$ to have a lottery ticket hypothesis? How to prune it?
-can one get rid of the ReLU?

For the trivial representations one can easily show that it will be fine, even with ReLU! For others, one decompose in irreducible representations.

\ggi{
Mentionner les kernels valides dans le cas continus. (p-e meme mettre une propostion qui les decrits.
Parler de la difference entre le cas continus et la discretisation
}

\dfe{Expliquer ici précisément notre construction de base avec le problème de l'origine}
}
\subsection{Construction of the Kernel at the Origin}
\label{app:origin}

We would like to apply our basis construction formula to every point in the plane, including the origin but the problem is that at the origin: $\forall g \in G, \quad g \cdot 0 = 0$. Therefore the above sum is not well defined because it is infinite for infinite groups. We can only apply this formula in the case of $G$ finite. The usual way to solve the problem at the origin if one deals with infinite groups is to solve all the linear problems $\pi(g \cdot x)=\rho_{\text{out}}(g)\pi(x)\rho_{\text{in}}(g)^{-1}$. However in the setting of Corollary 3, we deal with finite subgroups of $\text{O}(2)$. Therefore we can apply the above formula:
\begin{align*}
    \forall y \in \sR^2, \quad b(y):= \tK_{G,0}^{p,q}(y) = \sum_{g \in G} \rho_{i+1}(g) \tK_{0,0}^{p,q}(g^{-1}y)\rho_{i}(g^{-1}).
\end{align*}
In our case, when dealing with the regular representation, if one takes $G=C_n$ the cyclic group of n rotations, one will check that the $\rho_i(g)$ are permutation matrices associated with the permutation of $G: h \mapsto g\cdot h$. One can then check that summing over $G$ leads to a circulant matrix.

We have thus computed the set of equivariant kernels at the origin by using the above formula. One may have wanted to solve all the linear problems set by the equivariant constraints. Here they can be reformulated by the fact that the kernel at the origin must commute with all the matrices associated with the permutations of $G: h \mapsto g\cdot h$. Solving this leads to the set of circulant matrices.

\subsection{Discretization of $\mathbb{R}^2$}
\label{app:discretization}
We now highlight the practical challenges of building equivariant networks and their associated pruning when we discretize continuous signals on $\mathbb{R}^2$ to a pixelized grid.

\xhdr{From $\mathbb{R}^{2}$ to $[-\frac{d}{2},\frac{d}{2}]^{2}$}
The first problem we want to address is that we do not usually work on the plane $\mathbb{R}^{2}$ but on spatially delimited images on $[-\frac{d}{2},\frac{d}{2}]^{2}$. This is problematic since when $G$ acts on a square images, it can become a non-square image after a rotation. For example, $C_8$ doe not always send $[-\frac{d}{2},\frac{d}{2}]^{2}$ on $[-\frac{d}{2},\frac{d}{2}]^{2}$ (take the rotation by 45° for instance). In the same way restricting the equivariant kernel to a finite space $[-\frac{d}{2},\frac{d}{2}]^{2}$ as it is done in usual CNNs would lead to problems since for some $x \in [-\frac{d}{2},\frac{d}{2}]^{2}$, and some $g \in G$, $g \cdot x \notin [-\frac{d}{2},\frac{d}{2}]^{2}$. We overcome this issue by restricting the kernels to not being defined on $\sR^2$ but on a disk centered at the origin whose diameter equals the size of the image (see Figure \ref{fig:e2_basis}). To implement this, we multiply with a mask which exponentially decays to zero for points with radius larger than the radius of the disk.
This is permitted because the equivariant constraint set constrains the interior of the orbit, and it is trivial that for sub-groups of $\text{O}(2)$, the disc is stable under the action of the group. Therefore, the kernels that we obtain are still equivariant because they can check the equivariant constraint.

\xhdr{From $[-\frac{d}{2},\frac{d}{2}]^{2}$ to $\{-\frac{d}{2},-\frac{d-2}{2},...,\frac{d-2}{2},\frac{d}{2}\}^{2}$}
The second problem that we must address is the discretization process. Indeed, we do not work with continuous feature fields $f: [-\frac{d}{2},\frac{d}{2}]^{2} \rightarrow \mathbb{R}^{c}$ but with pixellized images, i.e. discretized inputs $f: \{-\frac{d}{2},-\frac{d-2}{2},...,\frac{d}{2}\}^{2}\rightarrow \mathbb{R}^{c}$.
This a problem, because the equivariant constraint \eqref{equiv_constraint} puts constraints between $k(g \cdot x)$ and $k(x)$. Equation \ref{equiv_constraint} cannot be used anymore because $g \cdot x$ is not always on the grid. For instance, if $x=(1,1)$ and $g$ is the rotation by 45°, then $g\cdot x = (0, \sqrt{2}) \notin \{-\frac{d}{2},-\frac{d-2}{2},...,\frac{d}{2}\}^{2}$. Moreover, note that it is not sufficient to discretize the equivariant kernels: one must choose only a finite subset of them. Indeed, the dimension of the equivariant map must be finite in the discretized setting as opposed to the continuous setting where it is infinite. In practice, the network is not exactly equivariant, but almost equivariant due to a discretization error. However, this is not an issue in the setting of Theorem \ref{thm:main_approx_network}.
Indeed, once we have chosen a basis of the ``almost-equivariant" kernels, we can prove the SLTH for the class of such networks, which is exactly the result that we want in practice. 

\citet{weiler2019general} choose a finite subset of the equivariant kernels, the authors upper-bound the frequency of the polar coordinate solution by an anti-aliasing condition. They then discretize the continuous kernels on the grid.
For our basis construction, we choose a finite subset of the equivariant kernels by restricting $\gA_{\gR}$ to only $\gA_{\gR} \bigcap \{-\frac{d}{2},-\frac{d-2}{2},...,\frac{d}{2}\}^{2}$. 
There are many different ways to discretize our kernels $\tK_{G,x}^{p,q}$ defined on $\sR^2$. One way would be to send $g\cdot x$ to the nearest pixel if it is not on the grid. In order to decrease the discretization error, we first upsample the grid by a factor $3$ before we start applying actions of the group $G$ to the base space. For the subgroups of $\text{O}(2)$ we consider, rotations of the base space are performed using bilinear interpolation. Finally, we downsample to the original size in order to obtain the discretized version of $\tK_{G,x}^{p,q}$.

\subsection{Proof of SLT on $\text{E}(2)$-steerable CNNs using theorem \ref{thm:main_approx_network}}
\label{app:steerable_proof}

We now prove Corollary \ref{cor:e2CNN}. We work with trivial or regular representations of $G \leq \text{O}(2)$ on top of feature fields. It is straightforward that the pointwise ReLU is equivariant. Moreover, to easily compute $\vertiii{\gB_{i \to i+1}}$ we work with $\|\cdot\|_{\infty}$ which implies that the ReLU is then $1$-Lipschitz.
Finally, the identity can trivially be written as the convolution with an equivariant kernel having the identity at the origin (the identity is trivially a circulant matrix). Therefore we have that $\sI \in \gB_{i \to i}$

If we use a trivial representation on top of the feature field at layer $i$, then the building block of this layer is $\sF_i=\sR^{d^2}$. If we instead use a regular representation, then the building block of this layer is $\sF_i=\sR^{d^2 \times |G|}$. 
From the construction of the equivariant basis, we deduce that $|\gB_{i \to i+1}|\leq d^{2}|G|^2$ for each layer. Indeed, we must first choose a pixel on the set of representatives $\gA_{\gR} \subset \{-\frac{d}{2},-\frac{d-2}{2},...,\frac{d-2}{2},\frac{d}{2}\}^{2} \simeq[d] \times [d]$ grid, and then choose a subset of the canonical basis at this point. But such canonical basis has $|G| \times |G|$ elements for regular to regular, $1 \times |G|$ element for trivial to regular (or regular to trivial), and finally only $1 \times 1$ for trivial to trivial. This is even less than that at some points such as the origin because of the additional constraints. Finally, one has less that $d^{2} \times |G|^2$ choices in all cases which indicates that,
\begin{align*}
    |\gB_{i \to i+1}| \leq  d^{2}|G|^2.
\end{align*}

In fact, since we can only choose $x \in \gA_{\gR}$ to obtain a set of independent elements, the true dependency will be $|\gB_{i \to i+1}| \simeq |\gA_{\gR}| \cdot |G|^2 \simeq \frac{d^2}{|G|} \cdot|G|^2 =d^2|G|$. However because of the discretization procedure, it is easier to upper bound by $d^2|G|^2$ since the cardinal of the discretized version of $\gA_{\gR}$ is not easily computable. Moreover, the reader will note that the cardinal of the basis has no real significance by itself because the basis was computed with an arbitrary discretization procedure, and therefore another procedure may have lead to another cardinal. Due to the artifacts during the discretization procedure the basis we construct $\gB_{i \to i+1}$ and only approximate a subset of all equivariant maps.
We now compute $\vertiii{\gB_{i \to i+1}}$ when employing the $\|\cdot \|_{\infty}$ on each feature space. Applying the triangular inequality we get:
\begin{align*}
    \vertiii{\gB_{i \to i+1}} \leq |\gB_{i \to i+1}|\max_{b_{i \to i+1,k}\in \gB_{i \to i+1}}\vertiii{b_{i \to i+1,k}}.
\end{align*}
It remains then to upper-bound $\vertiii{b_{i \to i+1,k}}$ for every element in the basis.
For all $ x\in \gA_{\gR}$ and for all $p,q \in [|G|]$ denote $b_{i \to i+1, p, q,x}$ the convolution with the equivariant kernel $\rmK_{G,x}^ {p,q}$. We have using a result from \citet{da2022proving} that $\vertiii{b_{i \to i+1,p,q,x}}\leq \|\tK_{G,x}^{p,q}\|_1$. Then, in a non-discretized kernel setting, while noticing that the orbit of $x$ has $|G|$ elements, one has $\|\rmK_{G,x}^{i,j}\|_1\leq |G|$. Then, $\vertiii{b_{i \to i+1,p,q,x}} \leq |G|$. For the identity this remains true as by using of circulant matrices it is trivial that $\|\rmK_{G,0}^{i,j}\|_1=|G|$. 
\cut{In a discretized setting, the only thing to check is that we still have $\|\rmK_{G,x}^{i,j}\|_1=|G|$. Indeed, this is true because when a $1$ does not appear on an exact pixel, the discretization procedure by upsampling the grid sends it on a mixture of the pixels around it but the sum of the absolute values in the pixels around remains $1$.}

\cut{
\subsubsection{Case of trivial representations with ReLU}
$G$ is still a subgroup of $O(2)$, not necessarily finite.
One study the case where $\rho_{in}$ and 
$\rho_{out}$ are respectively trivial on $\mathbb{R}^{c_{in}}$ and $\mathbb{R}^{c_{out}}$. This is for example the case of RGB or grey images where the images have a group of symetry G, for example by rotations or symetry. A typical example are visual recognition of digits (MNIST \CITE) or objects (CIFAR10, ImageNet, \CITE). This is however not the case of flows of vectors \CITE. 
One consider exactly the same set-up as in \cite{da2021proving} where we want to approximate a single equivariant convolution $K \in \mathbb{R}^{d,d,c,1}$ between $\mathbb{R}^{c}$ and $\mathbb{R}^{1}$ by a one hidden-layer equivariant CNN, where $U \in \mathbb{R}^{d,d,c,n}$ and $V \in \mathbb{R}^{1,1,n,1}$.

We have the right to say that because the fact that $V \in \mathbb{R}^{1,1,n,1}$ makes $V$ automatically equivariant. Indeed, in this case $k(x(r, \phi)) = 0$ if $r \neq 0$ which automatically implies that \eqref{kernel} is automatically verified because $G$ makes $r$ invariant!
However $U$ must satisfy the equivariant properties of a convolution for $G$ that we specified above, and especially that $k(x)$ is constant in each $\mathcal{O}(x)$.

It is very nice because now let's prune from U all the negative coefficients. One must check that $U^{+}$ is still equivariant. It is true because k(x) is is still constant in each $\mathcal{O}(x)$.
Now one can ignore the ReLU if one consider only positive entries that we will do now!!!

$V*\sigma(U^{+}*X)=V*(U^{+}*X)=(V*U^{+})*X$

Finally, one must prune $U^{+}$ so that $(V*U^{+}) \simeq K$.
One can do this by subset sum trivially. Moreover, for each $x$, since $k_{U}(x)$ and $k_{K}(x)$ are constant in each $\mathcal{O}(x)$, one can prune $U$ in the same way inside each $\mathcal{O}(x)$. Finally, it's the same conclusion as in \citet{da2021proving}.

\ggi{Question: What did da Cuhna do differently}

\subsubsection{Case of a regular representation of a cyclic group with ReLU}
One consider a finite subgroup $G$ of $O(2)$ where $\#G =d$.
Let's do it first for $G=\mathcal{C}_{d}$ the cyclic group of $d$ rotations. It will be exactly the same for the diedral group or any finite compact group.
The goal is to show that for any linear equivariant mapping from $f_{in}: \mathbb{R}^{2} \rightarrow \mathbb{R}^{d}$ to $f_{out}:\mathbb{R}^{2} \rightarrow \mathbb{R}^{d}$ which transform under the induced representation $Ind_{G}^{(\mathbb{R}^{2}, +) \rtimes G}\rho$ where $\rho$ is the regular representation, for any random overparametrized $G$-steerable network with regular representations of $G$, one can prune it where it remains equivariant and it can approximate the linear application.

More precisely the random overparametrized network is a one hidden layer network:

entry:$f_{in}: \mathbb{R}^{2} \rightarrow \mathbb{R}^{d}$ which transforms under the regular representation

intermediate layer: $f_{int}:\mathbb{R}^{2} \rightarrow \mathbb{R}^{nd}$ which transforms under n direct sum of the regular representation

output layer :  $f_{out}:\mathbb{R}^{2} \rightarrow \mathbb{R}^{d}$ which transforms under the regular representation.

$U : f_{in} \rightarrow f_{int}$ is a $d\times d$ equivariant convolution with kernels satisfying \eqref{kernel}.

$V: f_{int} \rightarrow f_{out}$ is a $1\times 1$ equivariant convolution with $k(0)$ which must satisfy \eqref{kernel} too.

One must look at the possible form of $k(x)$ inside $\mathcal{O}(x)$ allowed by \eqref{kernel}. For $U$ and $V$, one will just concatenate $n$ different such kernels because the direct sum of representations makes the constraints independant.

Let us fix $k(x)$. Then for every $g \in G, k(gx)$ consists of a permutation of $k(x)$ on rows and columns on the same time. One checks then that every $k(x) \in \mathbb{R}^{c_{out}*c_{in}}$ is possible and fix $k(y) \forall y \in \mathcal{O}(x)$.
x is 0 (r=0=) it is different because one must check that k(gx)=k(x) satisfies the circular operations equalities. It therefore implies that k(x) has only d free parameters and is $\sum a_{i} P_{i}$ where P are the circulant matrices.

The first thing is to bypass the ReLU. We then restrict ourselves to positive entries and prune all negative coefficienst of U. One checks that U remains equivariant because it implies that when a coefficient is negative it is replaces by 0 in all its occurence in $\mathcal{O}(x)$ and then the circulant condition remains.

Now $V*\sigma(U^{+}*X)=V*(U^{+}*X)=(V*U^{+})*X$

The strategy is now: show that $\forall x \in \mathbb{R}^{2}/\mathcal{R}$, one can prune U so that it remains equivariant and $k_{V*U^{+}}(x)$ approximates $k_{K}(x)$

The fact that we have still an equivariant network implies that \eqref{kernel} holds for K and V*U and then k(gx) is completely defined by this equation and the approximation remains true $\forall x \in \mathbb{R}^{2}$

$k_{V*U^{+}}(x)=\sum_{i=1}^{n} k_{V, i}(x)k_{U^{+},i}(x)$

Now, one prune $V$ in such a way that for each i $k_{V,i}$ has only its antidiagonal left. One checks that V is still equivariant.
Now one can prune the n last rows of each submatrix of U(x) and in the $\mathcal{O}(x)$ to remain equivariant too, such that the first row of V*U approximates the first one of K etc... In the end one has the good thing for $k_{K}(x)$ and by equvariance the good thing in the whole $\mathcal{O}(x)$.

For approximating a whole layer of regular representation, it would be the same trick as in the paper for CNN, ie pruning V in a diagonal way in some sense to make the subset sum problems independant.

Check that hypothesis of the general dense equivariant case are verified?!
}

\cut{\subsection{Proof for the steerable CNN at the manner of dacunha to see how it goes exactly}

%\subsubsection{Merging of the case of regular representation of a cyclic group and of the case of non-positive entries}

%Unformal: The idea is to mix the fact that we must have a sum of hyperplanes on specific coordinates and that U and V must satisfy equivariance constraints.

One consider a finite subgroup $G$ of $\gO(2)$ where $\#G =d$.
Let's do it first for $G=\mathcal{C}_{d}$ the cyclic group of d rotations of the plane. It will be exactly the same for the diedral group $\gD_{d}$ or any finite compact group of $\mathcal{O}(2)$ . The reason we work with finite sub-groups of $\gO(2)$ is because we are dealing with finite dimensional features, and under the regular representation, the dimension of a feature space is the cardinal of the group.

We moreover consider the case of equivariant networks for which all the features fields of the input layer, output layer and intermediate layers transforms under the regular representation $\rho_{reg}$ if it is $\mathbb{R}^{d}$ and $\overset{n}{\underset{1}{\oplus}}\rho_{reg}$ if it is $\mathbb{R}^{nd}=\overset{n}{\underset{1}{\oplus}}\mathbb{R}^{d}$

Let's look at the possible equivariant layers between ($f_{in}: \mathbb{R}^{2} \rightarrow \mathbb{R}^{n_{1}d}, \overset{n_{1}}{\underset{1}{\oplus}}\rho_{reg})$ and ($f_{out}: \mathbb{R}^{2} \rightarrow \mathbb{R}^{n_{2}d}, \overset{n_{2}}{\underset{1}{\oplus}}\rho_{reg})$

Denote $W$ such a linear function.

From section \ref{} one can decompose $W$ as a 2D-convolution with kernels $k(x)$ where $x \in \mathbb{R}^{2}$ satisfying the equivariant constraint $\eqref{kernel}$:

If $f_{in}: \mathbb{R}^{2}\rightarrow \mathbb{R}^{n_{1}d}$ is the input feature map,
$W(f_{in})_{i,j}= \sum_{k,l}k(k,l)f_{in}(i-k,j-l$, where $k(.,.)$ must satisfy $\eqref{kernel}$.

In the particular case of $G=\gC_{n}$, let for $x \neq 0$ fix $k(x)=\sum_{i=1,j=1}^{n_{1}, n_{2}}W_{i,j}B_{i,j}$ where $W_{i,j} \in \mathbb{R}^{d \times d}$, and $B_{i,j} \in \mathbb{R}^{n_{1}d*n_{2}d}$ transfers the i-th bloc of d components of a vector $x \in \mathbb{R}^{n_{1}d}$ on the j-th block of the null vector in $\mathbb{R}^{n_{2}d}$. In other words, $B_{i,j}$ sends
$\begin{pmatrix}
x_{1} \in \mathbb{R}^{d} \\
x_{2} \in \mathbb{R}^{d} \\
... \\
x_{n_{1}} \in \mathbb{R}^{d}
\end{pmatrix}$
on the vector
$\begin{pmatrix}
0 \in \mathbb{R}^{d} \\
... \\
x_{i} \in \mathbb{R}^{d}  \\
... \\
0 \in \mathbb{R}^{d} 
\end{pmatrix}$
\ggi{To be checked but $W_{i,j} B_{i,j}$ could be $E_{i,j} \otimes W_{i,j}$}
where $x_{i}$ is at the j-th line in the second vector.
Then, $W_{i,j}B_{i,j}$ is a notation to denote the linear application in $\mathbb{R}^{n_{1}d*n_{2}d}$ which sends 
$\begin{pmatrix}
x_{1} \in \mathbb{R}^{d} \\
x_{2} \in \mathbb{R}^{d} \\
... \\
x_{n_{1}} \in \mathbb{R}^{d}
\end{pmatrix}
on
\begin{pmatrix}
0 \in \mathbb{R}^{d} \\
... \\
W_{i,j}x_{i} \in \mathbb{R}^{d}  \\
... \\
0 \in \mathbb{R}^{d} 
\end{pmatrix}$

Then one can check that \eqref{kernel} is equivalent here as $\forall g \in G$, $k(gx)=\sum_{i=1,j=1}^{n_{1}, n_{2}}(P_{\sigma_{g}}W_{i,j}P_{\sigma_{g}}^{-1})B_{i,j}$. One checks then that $\forall x \in \gA$ with $x\neq 0$ every $k(x)$ $\in \mathbb{R}^{n_{1}d*n_{2}d}$ is possible and fixes $k(y)$ $\forall y \in \mathcal{O}(x)$. If $x=0$ it is different because one has $\forall g \in G, gx=x$ and one must check that $k(gx)=k(x)$ still satisfies \eqref{kernel}. It implies that $\forall g \in G, \forall (i,j) \in [n_{1}]*[n_{2}], P_{\sigma_{g}}W_{i,j}P_{\sigma_{g}}^{-1}=W_{i,j}$. It is equivalent to the fact that $\forall (i,j) \in [n_{1}]*[n_{2}], W_{i,j} \in Span(P_{k}, k \in [d])$ It therefore implies that $k(x)$ has only $n_{1}n_{2}d$ free parameters and more precisely: $k(0)=\sum_{i=1,j=1}^{n_{1}, n_{2}}W_{i,j}B_{i,j}$ where $W_{i,j}=\sum_{k=1}^{d}\gamma^{i,j}_{k}P_{k}$ where $P_{k}$ is the k-th circulant matrix.

Finally, denote by $\mathcal{F}$ the set of target ReLU G-steerable CNN such that : (i) $f:\mathbb{R}^{n_{0}*d_{0}}\rightarrow \mathbb{R}^{n_{l}*d}$, (ii) $f$ has depth $l$, (iii) weight matrix of layer $k$ which is the combination coefficient of $B_{i,j}$ (denoted $W^{k}_{i,j}$) has spectral norm at most 1. That is 

$\mathcal{F}=\{f:f(\mathbf{f_{in}})=\mathbf{K^{l}}\sigma(\mathbf{K^{l-1}}...\sigma(\mathbf{K^{1}f_{in}})), \forall k$  $\mathbf{K^{k}}: (\mathbb{R}^{2} \rightarrow \mathbb{R}^{n_{k-1}d})\rightarrow (\mathbb{R}^{2} \rightarrow \mathbb{R}^{n_{k}d})$ where $\mathbf{K^{k}}$ is the convolution with kernels of layer $k$: $k^{k}(x)=\sum_{i=1,j=1}^{n_{1}, n_{2}}W^{k}_{i,j}B^{k}_{i,j}$ where $W^{k}_{i,j} \in \mathbb{R}^{d*d}, B^{k}_{i,j}: (\mathbb{R}^{2}\rightarrow \mathbb{R}^{n_{k}d})\rightarrow (\mathbb{R}^{2}\rightarrow \mathbb{R}^{n_{k+1}d})$, $\parallel W^{k}_{i,j} \parallel  \leq 1\}$

\begin{theorem}
Let $\gF$ be as defined above. Consider a randomly initialized 2l-layered permutation equivariant neural network:
$g(\mathbf{x})= \mathbf{K^{2l}}\sigma(\mathbf{K^{2l-1}}...\sigma(\mathbf{K^{1}x}))$ where $\mathbf{K^{2k}}$ is the convolution with kernels $k^{2k}(x)=\sum_{i=1,j=1}^{n_{k}, \tilde{n_{k}}}W^{2k}_{i,j}B^{2k}_{i,j}$ and $\mathbf{K^{2k-1}}$ is the convolution with kernels $k^{2k-1}(x)=\sum_{i=1,j=1}^{\tilde{n_{k}}, n_{k+1}}W^{2k-1}_{i,j}B^{2k-1}_{i,j}$, where every weight of the $W_{i,j}^{k}$ is taken in $\gU([-1,1])$. Let's take $\tilde{n_{k}} = Cn_{k}\log(\frac{dn_{k}l}{min(\epsilon, \delta)})$.

Then with probability at least 1-$\delta$, $\forall f \in \gF,$ one can prune the matrices $W_{i,j}^{k}$ so that the resulting equivariant network approximates f by at most an error $\epsilon$ $\forall \mathbf{x}$, $\parallel \mathbf{x} \parallel \leq 1$
\end{theorem}

We first prove an approximation lemma which states that one can approximate a linear equivariant layer by pruning an overparametrized one-hidden layer neural network with a log-overparametrization.
The goal is to show that for any linear equivariant mapping from $f_{in}: \mathbb{R}^{2} \rightarrow \mathbb{R}^{n_{1}d}$ to $f_{out}:\mathbb{R}^{2} \rightarrow \mathbb{R}^{n_{2}d}$ which transform under the induced representation $Ind_{G}^{(\mathbb{R}^{2}, +)\rtimes G}\rho$ where $\rho$ is the direct sum of $n_{1}$ or $n_{2}$ regular representation ($\rho =\overset{n_{i}}{\underset{1}{\oplus}}\rho_{reg}$), for any random overparametrized one-hidden layer G-CNN with regular representations of $G$, one can prune it such that it remains equivariant and it can approximate the target equivariant linear application.

More precisely the random overparametrized network is a one-hidden layer network defined as:

Input feature space $E_{in}$: $f_{in}: \mathbb{R}^{2} \rightarrow \mathbb{R}^{n_{1}d}$ which transforms under the sum of $n_{1}$ regular representation

Hidden layer feature space $E_{int}$: $f_{int}:\mathbb{R}^{2} \rightarrow \mathbb{R}^{nd}$ which transforms under $n$ direct sum of the regular representation

Output feature space $E_{out}$:  $f_{out}:\mathbb{R}^{2} \rightarrow \mathbb{R}^{n_{2}d}$ which transforms under the regular representation.

$K^{1}$ : $E_{in} \rightarrow E_{int}$ is a D*D equivariant convolution with kernels satisfying \eqref{kernel}.

$K^{2}$: $E_{int} \rightarrow E_{out}$ is a D*D equivariant convolution with k(0) which must satisfy \eqref{kernel} too.

One must look at the possible form of k(x) inside $\mathcal{O}(x)$ allowed by \eqref{kernel}. For U and V, one will just concatenate n different such kernels because the direct sum of representations makes the constraints independant.

\begin{lemma}
Let F be a random equivariant network of depth $l=2$: $F(\mathbf{x})=\mathbf{K^{2}}\sigma(\mathbf{K^{1}x})$ where $\mathbf{K^{1}}:\mathbb{R}^{2 \times n_{1}d}\rightarrow \mathbb{R}^{2\times \tilde{n}d}$ and $\mathbf{K^{1}}:\mathbb{R}^{2 \times \tilde{n}d}\rightarrow \mathbb{R}^{2 \times n_{2}d}$ are two general equivariant random layers with weights taken from $\gU([-1,1])$. Let's suppose that $\tilde{n}=Cn_{1} \log(\frac{n_{1}n_{2}}{min(\epsilon, \delta)})$

Then with probability at least 1-$\delta$, $\forall \mathbf{K}: \mathbb{R}^{2*n_{1}d}\rightarrow \mathbb{R}^{2*n_{2}d}$, where $\mathbf{K}$ is the convolution with kernels $k(x)=\sum_{i=1,j=1}^{n_{1}, n_{2}}W_{i,j}B_{i,j}$ such that $\parallel W_{i,j} \parallel \leq 1$, one can prune F such that $\forall x, \parallel x \parallel \leq 1$ it approximates f by at most $\epsilon$.
\end{lemma}

The first step in the proof is to prune $\mathbf{K}^{1}$ which will allow us to bypass the ReLU. More precisely, we prune all the $k^{1}(x)$ if x is non zero.
From now on, the first layer acts on the input $X: \mathbb{R}^{2} \rightarrow \mathbb{R}^{n_{1}d}$ as $\mathbf{K^{1}}(X)_{k,l}=k^{1}(0)(X_{k,l})=\sum_{i=1,j=1}^{n_{1}, \tilde{n}}W^{1}_{i,j}B^{1}_{i,j}(X_{k,l})$

Notation: for $X \in \mathbb{R}^{n*d}$, $[X]_{i} \in \mathbb{R}^{d}$ denotes the (i-1)*d+1 to i*d components of X.

Then, one has :

$[\mathbf{K^{1}}(X)_{k,l}]_{j}=\sum_{i=1}^{n_{1}}W^{1}_{i,j}[X_{k,l}]_{i}$

Using the equivariance constraint on $k^{1}(0)$, one can show that each $W^{1}_{i,j}$ can be written as $W^{1}_{i,j}=\sum_{m=0}^{d-1}a^{m}_{i,j}P_{m}$ where $P_{m}$ is the circulant permutation matrix associated with the permutation $\sigma^{m}$ where  $\sigma = (1\ 2\ 3\ ...\ d)$

Finally, $[\sigma(\mathbf{K^{1}}(X))_{k,l}]_{j}=\sigma(\sum_{i=1}^{n_{1}}(\sum_{m=0}^{d-1}a^{m}_{i,j}P_{m})[X_{k,l}]_{i})$

We will now prune $k^{1}(0)$ by (i) pruning $W^{1}_{i,j}$ if $j \notin [(i-1)*Clog(\frac{n_{1}d}{min(\epsilon, \delta)})+1, i*Clog(\frac{n_{1}d}{min(\epsilon, \delta)})]$ (diamond shape) (ii) we prune all the remaining $a^{m}_{i,j}$ for $m \neq 0$

We can write $[K^{2}\sigma(K^{1}X)_{i,j}]_{k}=[\sum_{l,m}k^{2}_{i-l,j-m}\sigma(K^{1}X)_{l,m}]_{k}=[\sum_{l,m}k^{2}_{i-l,j-m}\sigma(k^{1}(0,0)X_{l,m})]_{k}=[\sum_{l,m}\sum_{o=1}^{\tilde{n}}W^{2 (i-l,j-m)}_{o,k}[\sigma(k^{1}(0,0)X_{l,m})]_{o}=\sum_{l,m}\sum_{o=1}^{\tilde{n}}W^{2 (i-l,j-m)}_{o,k}\sigma(a^{0}_{f(o),o}[X_{l,m}]_{f(o)})$

We want this to approximate:

$[\mathbf{K}(X)_{i,j}]_{k}=[\sum_{l,m}k_{i-l,j-m}X_{l,m}]_{k}=\sum_{l,m}\sum_{p=1}^{n_{1}}W^{(i-l,j-m)}_{p,k}[X_{l,m}]_{p}$

It is then sufficient to show that by prunning the $W^{2,(i-l,j-m)}_{ok}$,

$\forall i,j,l,m,p,k$, $\forall X \in \mathbb{R}^{d}$, $W^{(i-l,j-m)}_{p,k}X \simeq \sum_{o such that f(o)=p}W^{2 (i-l,j-m)}_{ok}\sigma(a^{0}_{f(o),o}X)$

$(W^{(i-l,j-m)}_{p,k}X)_{\alpha} = \sum_{\beta}(W^{(i-l,j-m)}_{p,k})_{\alpha, \beta}X_{\beta}$ and 

$(\sum_{o such that f(o)=p}W^{2 (i-l,j-m)}_{ok}\sigma(a^{0}_{f(o),o}X))_{\alpha}=\sum_{o such that f(o)=p}\sum_{\beta}(W^{2(i-l,j-m)}_{ok})_{\alpha, \beta}\sigma(a^{0}_{f(o),o}X_{\beta})$

The good thing is now that, because we pruned $K^{1}$ in a clever way, we have that inside a ReLU, only one variable $X_{\beta}$ appears. We can basically apply the same process as in Pensia: summing $\sigma(ax)$ to approximate wx in one dimension.

More precisely it is sufficient to show that by pruning $\{(W^{2(i-l,j-m)}_{ok})_{\alpha,\beta},o\in f^{-1}(p)\}$, one can have 

$\sum_{o such that f(o)=p}(W^{2 (i-l,j-m)}_{ok})_{\alpha, \beta}\sigma(a^{0}_{f(o),o}X_{\beta}) \simeq (W^{(i-l,j-m)}_{p,k})_{\alpha, \beta}X_{\beta}$ 
One finally recovers the subset sum problem with two cases following the sign of $X_{\beta}$. It is solved as in Pensia.
}
\section{Additional Material on the Permutation Equivariant Networks}

\cut{
\subsection{Proof of STL for 
permutation equivariant nets in the manner of dacunha}

Proof of the lemma:

We decompose $\mathbf{H^{1}}$ and $\mathbf{H^{2}}$ in the basis $\{B^{\mu}, \mu \in [n]^{k}/\sim\}$ of all permutation equivariant linear functions.

$\mathbf{H^{1}}= \sum_{\lambda \in [n]^{2k}/\sim}W^{1}_{\lambda}B^{\lambda, d_{0}, d}$

$\mathbf{H^{2}}= \sum_{\mu \in [n]^{2k}/\sim}W^{2}_{\mu}B^{\mu, d, d_{1}}$

$\forall \mathbf{x} \in \mathbb{R}^{n^{k}*d_{0}} \mathbf{H^{1}}(\mathbf{x})= \sum_{\lambda \in [n]^{2k}/\sim}W^{1}_{\lambda}B^{\lambda, d_{0}, d}$

Now, $\forall x \in \mathbb{R}^{n^{k}*d_{0}}, \forall b \in [n]^{k} \mathbf{H^{1}}(\mathbf{x})_{b}=\sum_{a \in [n]^{k}}W^{1}_{(a,b)}\mathbf{x_{a}}$

$\forall x \in \mathbb{R}^{n^{k}*d_{0}}, \forall b \in [n]^{k} \mathbf{H^{1}}(\mathbf{x})_{b}=\sum_{a \in [n]^{k}}W^{1}_{(a,b)}\mathbf{x_{a}}$

$\forall x \in \mathbb{R}^{n^{k}*d}, \forall c \in [n]^{k} \mathbf{H^{2}}(\mathbf{x})_{c}=\sum_{b \in [n]^{k}}W^{2}_{(b,c)}\mathbf{x_{b}}$

Finally we use a pointwise ReLU, ie that acts on k-order tensors as: $\sigma: \mathbb{R}^{n^{k}*d}\rightarrow \mathbb{R}^{n^{k}*d}$, $\forall \mathbf{x} \in \mathbb{R}^{n^{k}*d}, \forall a \in [n]^{k}, \forall i \in [d], \sigma(\mathbf{x})_{a,i}= \sigma(\mathbf{x_{a,i}})$. One has therefore:

$\sigma(\mathbf{H^{1}x})_{b}=\sigma(\sum_{a \in [n]^{k}}W^{1}_{(a,b)}\mathbf{x_{a}})$

Finally, $\forall c \in [n]^{k} \mathbf{H^{2}}\sigma(\mathbf{H^{1}x})_{c} =\sum_{b \in [n]^{k}}W^{2}_{(b,c)}\sigma(\sum_{a \in [n]^{k}}W^{1}_{(a,b)}\mathbf{x_{a}})$

The first step of pruning is to delete all the $W^{1}_{(a,b)}$ expect for the one where b=a. In other words one keeps only the $W^{1}_{\lambda}$, $\lambda \in [n]^{k}/\sim$ where if $(a,b) \in \lambda$, one has a=b where $a \in [n]^{k}, b \in [n]^{k}$. One can now simplify the above sum:

$\forall c \in [n]^{k} \mathbf{H^{2}}\sigma(\mathbf{H^{1}x})_{c} =\sum_{b \in [n]^{k}}W^{2}_{(b,c)}\sigma(W^{1}_{(b,b)}\mathbf{x_{b}})$

We want this sum to approximate $\mathbf{H(x)_{c}}=\sum_{b\in [n]^{k}}W_{(b,c)}\mathbf{x_{b}}$ 

Informally, it can be easily done by putting the $W^{1}_{(a,a)}$ in a diamond shape and the pruning the $W^{2}_{(b,c)}$ !

We will use more precisely a lemma from \citet{pensia2020optimal} which will be the leverage point in the proof where we use the subset sum theorem and where the logarithmic overparametrization takes place.

\begin{lemma}
Consider a randomly initialized two layer neural network
$g(x) = N\sigma(M\mathbf{x})$ with $\mathbf{x}\in \mathbb{R^{d_{0}}}$ such that M has dimension $C*d_{0}*log(\frac{d_{0}d_{1}}{min(\delta, \epsilon)}$, $d_{0})$ and N has dimension $d_{1}$, $C*d_{0}*log(\frac{d_{0}d_{1}}{min(\delta, \epsilon)})$, where each weight is initialized independently from the distribution $\mathcal{U}([-1, 1])$.
Let $\hat{g}(x) = (S \circ N)\sigma((T \circ M)\mathbf{x})$ be the pruned network for a choice of pruning matrices S and T whose coefficients are in $\{0,1\}$. If $f_{W}(x) = Wx$ is the linear (single layered) network, where W has dimensions $d_{1},d_{0}$ , then with probability at least $1- \delta$,
\begin{equation}
   \underset{W:\parallel W \parallel \leq 1, W \in \mathbb{R}^{d_{1}*d_{0}}}{\sup} \underset{S,T}{\inf} \underset{\mathbf{x}, \parallel \mathbf{x} \parallel \leq 1}{\sup} \parallel f_{W}(\mathbf{x}) - \hat{g}(\mathbf{x}) \parallel \leq \epsilon 
\end{equation}

(informal but very important): With the same setting, one can restrict T to a diamond shape which does not depend on N and M but only on the dimension of M (diamond shape pruning)!!

\end{lemma}

In the following, we denote by $T^{(d_{0},d)} \in \{0,1\}^{d_{0}, d_{1}}$ defined by $\forall (i,j) \in [d]*[d_{0}], T^{(d_{0},d)}_{(i,j)}=1$ if and only if $j \in [(i-1)*Clog(\frac{d_{0}d_{1}n^{k}}{min(\epsilon, \delta)}), i*log(\frac{d_{0}d_{1}n^{k}}{min(\epsilon, \delta)})]$.

More precisely, we want to approximate each $W_{(b,c)}$ by pruning $W^{2}_{(b,c)}\sigma(W^{1}_{(b,b)}\mathbf{x}_{b})$ with a precision of $\frac{\epsilon}{\#([n]^{k})}$.

Since $W_{(b,b)}^{1}$ and $W_{(b,c)}^{2}$ have their weights drawn from $\gU([-1,1])$, $\parallel W_{(b,b)}^{1} \parallel \leq 1$, $\parallel W_{(b,c)}^{2} \parallel \leq 1$, and $d=C*d_{0}log(\frac{d_{0}d_{1}n^{k}}{min(\epsilon, \delta)})$, we can directly apply Lemma 2 to obtain that:

$\forall b \in [n]^{k}, \forall c \in [n]^{k}$, with probability at least 1-$\frac{\delta}{n^{k}}$, $\exists S_{(b,c)} \in \{0,1\}^{d_{1}*d}$, $\exists T_{(b,b)} \in \{0,1\}^{d_{1}*d} $ (one can take $T^{(d_{0}, d)}$ as defined previously), such that $(S_{(b,c)}\circ W^{2}_{(b,c)})\sigma((T_{(b,b)}\circ W^{1}_{(b,b)})\mathbf{x})$ approximates $W_{(b,c)}\mathbf{x}$ by at most $\frac{\epsilon}{n^{k}}$ $\forall \mathbf{x}$, $\parallel \mathbf{x} \parallel \leq 1$

By taking the intersection of $[n]^{k}$ events, one has that with a probability of at least 1-$\delta$ one can prune $\mathbf{H^{2}}$ and $\mathbf{H^{1}}$ in such a way that:

$\mathbf{H^{2}}\sigma(\mathbf{H^{1}x})_{c} =\sum_{b \in [n]^{k}}W^{2}_{(b,c)}\sigma(W^{1}_{(b,b)}\mathbf{x_{b}})$ approximates $\sum_{b\in [n]^{k}}W_{(b,c)}\mathbf{x_{b}}$ by at most $\epsilon$
}
\subsection{Proof of SLT on Permutation Equivariant Networks}
\label{app:permutation_proof}
The aim of this appendix is to prove Corollary \ref{cor:perm}.
The building blocks of the layers are here $\sF_i=\sR^{n^{k_i}}$. Taking direct sums of them we obtain $\sF_i^{n_i}=\sR^{n^{k_i}\times n_i}$.
Again as in appendix section \ref{app:steerable_proof}, the pointwise ReLU is equivariant and furthermore we facilitate the computation of $\vertiii{\gB_{i \to i+1}}$ by working with $\|\cdot\|_{\infty}$, which implies that the ReLU is $1$-Lipschitz.
As explained above, the norm that we must consider on $\sF_i^{n_i}=\sR^{n^{k_i}\times n_i}$ to apply theorem \ref{thm:main_approx_network} is the $\max$ of the norm across the blocks, i.e. still $\|\cdot\|_{\infty}$ on $\sR^{n^{k_i}\times n_i}$.
First, observe that $\vertiii{\gB_{i \to i+1}}=n^{k_i}+1$.

\begin{proof}
One can check that the worse case scenario happens when making $\sum_{b_k\in \gB_{i \to i+1}}b_k$ act on a tensor $\tX \in \sR^{n^{k_i}}$ full of 1.
Denote by $\tY_a$ for $a \in [n]^{k_i}$ the tensor in $\sR^{n^{k_i}}$ such that it has a $1$ at the index $a$ and $0$ everywhere else. The tensor full of $1$ is therefore $\sum_{a \in [n]^{k_i}}\tY_a$

\begin{align*}
    \vertiii{\gB_{i \to i+1}}&= \max_{\|\alpha\|_{\infty}\leq 1}\max_{\|\tX\|_{\infty}\leq 1}\left\|\left(\sum_k\alpha_kb_k\right)\tX\right\|_{\infty}\\
    &=\left\|\left(\sum_{b_k \in \gB_{i \to i+1}}b_k\right)\left(\sum_{a \in [n]^{k_i}}\tY_a\right)\right\|_{\infty}\\
    &\leq\left\|\left(\sI+\sum_{\mu \in [n]^{k_i+k_{i+1}}/\gQ}B^{\mu}\right)\left(\sum_{a \in [n]^{k_i}}\tY_a\right)\right\|_{\infty}\\
    &\leq \max_{b \in [n]^{k_{i+1}}} \left|\left(\left(\sI+\sum_{\mu \in [n]^{k_i+k_{i+1}}/\gQ}B^{\mu}\right)\left(\sum_{a \in [n]^{k_i}}\tY_a\right)\right)_b\right|.
\end{align*}

Now, $\forall b \in [n]^{k_{i+1}}$,

\begin{align*}
    \left(\left(\sI+\sum_{\mu \in [n]^{k_i+k_{i+1}}/\gQ}B^{\mu}\right)\left(\sum_{a \in [n]^{k_i}}\tY_a\right)\right)_b&=1+\left(\sum_{a \in [n]^{k_i}}\left(\sum_{\mu \in [n]^{k_i+k_{i+1}}/\gQ}B^{\mu}\right)\tY_a\right)_b\\
    &=1 + \left(\sum_{a \in [n]^{k_i}}B^{(a,b)}\tY_a\right)_b\\
    &=1 +\sum_{a \in [n]^{k_i}}1\\
    &=1+n^{k_i}.
\end{align*}
\end{proof}

Moreover $|\gB_{i \to i+1}| = \Tilde{b}(k_i + k_{i+1})$ by definition of the Bell numbers. In fact the interested reader will check that one has $|\gB_{i \to i+1}| \leq \Tilde{b}(k_i + k_{i+1})$ and that the equality happens as soon as $n \geq k_i+k_{i+1}$ (for example with $n=1$ one can not have an independent vector family of $\Tilde{b}(k_i +k_{i+1})$ vectors in $\gL(\sR, \sR)$ which is of dimension $1$. The argument expressed in \citet{maron2018invariant} needs $n \geq k_i+k_{i+1}$ to ensure that all the equivalence classes $\mu$ have at least one element. Finally, all the conditions to apply theorem \ref{thm:main_approx_network} are true and one only need to replace $\max(|\gB_{i \to i+1}|, \vertiii{\gB_{i\to i+1}})$ by $\max(\Tilde{b}(k_i + k_{i+1}), n^{k_i}+1)$.

\cut{Finally, it remains to check that for a layer $f_i^t = \sum_k \alpha_{i \to i+1,k}b_k$

\begin{align}
\label{consequence}
    \vertiii{f_i^t}\leq 1 \implies \|\alpha_{i \to i+1}\|_{\infty}\leq 1
\end{align}

Indeed, $\|\alpha_{i \to i+1}\|_{\infty}\leq 1$ is required in theorem \ref{thm:main_approx_network}. We didn't put the assumption in the corollary, becomes it is a consequence of the other assumption $\vertiii{f_i^t}\leq 1$

\begin{proof}[Proof of \eqref{consequence}]

Let's assume $\vertiii{f_i^t}\leq 1$.
Make $f_i^t$ act on a tensor $X \in \sR^{n^{k_i}}$ which is filled with zeros everywhere except at the position $a \in [n]^{k_i}$. One has $\|X\|=1$. It implies that $\forall b \in [n]^k_{i+1}, \quad |f_i^t(X)_b|\leq 1$. But on the other hand, $f_i^t(X)_b$ is the coefficient of $B^{(a,b)}$ in the decomposition of $f_i^t$ in the basis. By changing $a$ and $b$ it implies that: $\forall a \in [n]^{k_i}, \forall b \in [n]^{k_{i+1}}$, the coefficient of $B^{(a,b)}$ in the decomposition of $f_t^i$ in the basis is less than 1. (Note that for the coefficient of the identity, it is still true by looking at $f_i^t(X)_{(1, \dots,1)}$ where X has only a $1$ at position $(1, \dots, 1) \in n^{k_i}$

\end{proof}
}

\section{Experimental Details}
\label{app:experimental_detail}

\looseness=-1
The purpose of our experiments is to empirically validate the our theory found in the main text,
and as a result show that by solving appropriate $\textsc{Subset-Sum}$  problems one can prune an overparameterized random network to a target one.
In this section of the appendix, we describe the network architectures we use for our experiments, the overparameterization scheme we select in order to be compatible with our claims, and the linear program we solve for each target weight in order to find the sparsification mask which leads to the approximation of the target network by the overparameterized one.

For both MPGNN and $\text{E}(2)$-CNN experiments, we first train a single target network on the supervised tasks that we described in table \ref{tab:main_pruning_results}. The architecture we use for each
of the target networks is described in the tables \ref{tab:mpgnn-arch}, \ref{tab:e2cnn-arch}, and \ref{tab:sn-arch} below. Notice that we do not utilize bias in the parameterized layers, as well as we do not make use of learnable element-wise affine transformations in the batch normalization layers.
We train for $50$ epochs using AdamW as the optimizer with learning rate $0.015$ and default momentum parameters $\beta = (0.9, 0.999)$ and a cosine scheduler. The weight decay coefficient is set to $5\mathrm{e-}4$.
For the transductive learning tasks on Cora and CiteSeer with the MPGNN, we
define an epoch as $10$ parameter updates. For the image classification tasks on RotMNIST and FlipRotMNIST with the $\text{E}(2)$-CNN the batch size is set to $64$.
Finally, the target model is selected as the one which achieves the best validation accuracy throughout training.

Afterwards, we define the overparameterized network. In particular, for each parameterized layer (linear or equivariant) of the target network we
declare a module that we are going to approximate it with. The module consists of the composition of three layers; the first and the last being of the same type as the target layer, and the middle one is an element-wise ReLU activation function. We make source that the shapes of the input and output tensors match. We initialize these modules using iid drawn samples from $\gU([-a,a])$, where $a$ is determined as twice the maximum absolute parameter of the target network. As we explain in the appendix section \ref{app:subset_sum}, this is compatible with our theorem.

For each parameter in the target network we solve two $\textsc{Subset-Sum}$ problems, one to approximate the positive input tensors and one to approximate the negative input tensors. This distinction is needed if we want to use a ReLU in the overparameterized layers. The width is overparameterized by multiplying the input tensor size with a number that scales proportionally to a hyperparameter constant factor $C$, and logarithmically in the input and output size, the number of layers to be approximated, and in $1/\epsilon$, where $\epsilon$ is the desired network approximation error. For our experiments, we use $\epsilon = 1\mathrm{e-}2$. For further details, the reader is requested to examine the associated Python repository that we provide.

Finally, we solve each defined $\textsc{Subset-Sum}$ problem by treating it as a mixed-integer linear program, similar to \citet{pensia2020optimal}. Each one of the problems amounts to a different constraint optimization problem of the following form:
\begin{align}
    \label{app:eq:subset_sum}
    & \min_{z \in \mathbb{R}, \bm{m} \in \mathbb{Z}_2^n}  z \\
    \text{s.t.} \quad  & y  - \bm{m}^\top \bm{x} <= z \nonumber\\
    &\bm{m}^\top \bm{x} - y <= z &&\nonumber
\end{align}
In the optimization problem above, $\bm{x}$ is a vector
resulting from the multiplication of the two weight matrices which participate in the diamond-shaped approximation scheme for each target weight $y$, as explained in \ref{Approximation_of_a_layer}. Optimization variables $z$ and $\bm{m}$ amount to the absolute weight approximation error and part of the binary mask of the second layer in the overparameterized network, which is responsible for approximating the particular weight $y$.

\begin{table}[!ht]
    \centering
    \begin{tabular}{lcc}
        Layer & $\rho_\text{in}$ & $\rho_\text{out}$ \\
        \toprule
        GCNConv & data features & 16 \\
        ReLU \& Dropout(prob. = 0.5) &&\\
        GCNConv & 16 & 16 \\
        ReLU \& Dropout(prob. = 0.5) &&\\
        GCNConv & 16 & number of classes \\
        \bottomrule
    \end{tabular}
    \caption{Target network architecture for the MPGNN experiments.}
    \label{tab:mpgnn-arch}
\end{table}

\begin{table}[!ht]
    \centering
    \begin{tabular}{lcc}
        Layer & $\rho_\text{in}$ & $\rho_\text{out}$ \\
        \toprule
        R2Conv(kernel = (9, 9)) & number of channels $\times$ trivial repr. & 24 regular repr. \\
        InnerBatchNorm &&\\
        Pointwise Max Pool(kernel = (2, 2)) &&\\
        Pointwise ReLU &&\\
        R2Conv(kernel=(7, 7)) & 24 regular repr. & 48 regular repr. \\
        InnerBatchNorm &&\\
        Pointwise ReLU &&\\
        GroupPooling &&\\
        Max Pool(kernel = (2, 2)) \& Flatten  &&\\
        Linear & $48 \times |G|$ & $48$ \\
        BatchNorm \& ReLU &&\\
        Linear & $48$ & number of classes \\
        \bottomrule
    \end{tabular}
    \caption{Target network architecture for the E2CNN experiments.}
    \label{tab:e2cnn-arch}
\end{table}

\begin{table}[!ht]
    \centering
    \begin{tabular}{lcc}
        Layer & $\rho_\text{in}$ & $\rho_\text{out}$ \\
        \toprule
        Linear $2\mapsto2$ & data features $\times N^2$ & $16 \times N^2$ \\
        Linear $2\mapsto1$ & $16 \times N^2$ & $32 \times N$ \\
        Sum Pooling across $N$ \& Flatten & & \\
        Linear & $32$ & $64$ \\
        ReLU & & \\
        Linear & $64$ & number of classes \\
        \midrule
        Linear $2\mapsto2$ & data features $\times N^2$ & $24 \times N^2$ \\
        ReLU && \\ 
        Linear $2\mapsto2$ & $24 \times N^2$ & $48 \times N^2$ \\
        ReLU && \\
        Linear $2\mapsto1$ & $48 \times N^2$ & $96 \times N$ \\
        Sum Pooling across $N$ \& Flatten & & \\
        Linear & $96$ & $96$ \\
        ReLU & & \\
        Linear & $96$ & number of classes \\
        \bottomrule
    \end{tabular}
    \caption{Target network architecture for the $k$-order  $\mathcal{S}_n$-equivariant GNN experiments. Top architecture was used for the Proteins dataset, whereas the bottom one for the NCI1 dataset.}
    \label{tab:sn-arch}
\end{table}

\cut{
\section{Future work}

\begin{itemize}
    \item Particular role of the pointwise ReLU/ Subset sum on functions and not on scalars $\rightarrow$ lead to more general than just pointwise RELU
\end{itemize}
}
\cut{
\section{Further Related Work}
\label{app:further_related_work}
\subsection{Related Work}
\ggi{Potentially merge with the intro}
\dfe{make it more brief, only two paragraphs}

The recent progress in the field of deep learning is largely due to physical technological advances in the field of hardware. The use of GPU and hardware always more powerful have allowed to train models whose number of parameters can reach several billions on datasets more and more complex (ImageNet, REsnet, CITE).if the overparameterization is favorable to a good generalization of these models, it remains that their use on small devices remains very complex or even impossible because of their size in memory and the number of operations required. This is why some methods aim at eliminating the redundant parameters of a network by pruning since the 80s (\citep{lecun1989optimal}, \citep{mozer1988skeletonization}). The paper \citep{blalock2020state} provides an exhaustive overview of the current pruning methods and of their respective performances.

In line with this general effort to reduce model complexity, \citep{frankle2018lottery} observed experimentally that within any randomly initialized network, there exists a subnetwork (lottery ticket) which, when trained separately on the same dataset as the initial network, will have undiminished performance! This surprising result, then formalized under the term of lottery ticket hypothesis galvanized research, in particular in industry for the saving of time and energy, during the training or in the test phase that these ticks represent. However, the way presented by \citep{frankle2018lottery} to find these winning tickets consists in successive steps of training and pruning which are very costly in time and energy, making the search for these sparse sub-networks not very advantageous. This is why many works have focused on methods to efficiently find these sub-networks by mask learning (CITE, ...) or other (CITE ...). %DISCUSS WORKS ON LTH AS IT IS DONE IN THE PAPER OF CNN.

A stronger version of the LTH was proposed by \citep{ramanujan2020s} who remarked that sparse subnetworks can achieve good performance without modification of their weights, i.e., without any training. This result hypothesised as the SLTH was proved later by \citep{malach2020proving} who showed that one could approximate any target dense network by pruning a random overparametrized dense network, with polynomial bounds in the width and depth of the target one. This result was improved by \citep{pensia2020optimal} who reduced the polynomial bound to a logarithmic overparametrization. Finally, \citep{da2021proving} proved a SLTH for CNN: from any random overparametrized CNN with logarithmic overparametrization, one can with high probability prune it to approximate a target CNN.
}

\end{document}